\documentclass[10pt,twocolumn,letterpaper]{article}

\pdfoutput=1

\usepackage[pagenumbers]{cvpr} %

\usepackage[utf8]{inputenc}
\usepackage[T1]{fontenc}
\usepackage[english]{babel}
\usepackage{microtype}
\usepackage{xcolor}
\usepackage{graphicx}
\usepackage{amsmath}
\usepackage{amssymb}
\usepackage{booktabs}
\usepackage[inline]{enumitem}
\usepackage{multirow}
\usepackage{pifont}
\usepackage{amsthm}
\usepackage{mathtools}
\usepackage{stmaryrd}
\usepackage{calrsfs}
\usepackage[numbers,square,sort&compress]{natbib}
\usepackage{dsfont}
\usepackage[ruled]{algorithm2e}
\usepackage[pagebackref,breaklinks,colorlinks]{hyperref}
\usepackage{float}
\usepackage[accsupp]{axessibility}

\DeclareMathAlphabet{\pazocal}{OMS}{zplm}{m}{n}

\newcommand{\lrp}[1]{\left( #1 \right)}

\newcommand{\sums}[2]{\sum\limits_{#1}^{#2}}
\newcommand{\cl}[1]{\pazocal{#1}}
\newcommand{\T}{^{\top}}
\newcommand{\real}{\mathbb{R}}

\newcommand{\dd}{\mathrm{d}}

\newcommand{\E}{\mathds{E}}

\renewcommand{\vec}[1]{\boldsymbol{#1}}

\let\epsilon\varepsilon

\let\exp\undefined
\DeclareMathOperator{\exp}{exp}

\DeclareMathOperator*{\argmin}{arg\,min}
\DeclareMathOperator*{\argmax}{arg\,max}

\def\SPHERE{\mathbb{S}}
\def\SPHEREUNIFORM{u_{\SPHERE_d}}
\def\TARGETS{\vec z_1, \dots, \vec z_n \in \cl \SPHERE_{d}}
\def\ARGMIN{\argmin\limits_{\TARGETS}}
\def\ARGMAX{\argmax\limits_{\TARGETS}}

\def\LVMFT{\tilde{\cl{L}}_\mathrm{}}
\def\Lalign{{\cl{L}}_\mathrm{LSP}}
\def\Lunif{{\cl{L}}_\mathrm{Unif}}
\def\Lfinal{{\cl{L}}_\mathrm{\method}}

\def\supportSet{\cl{S}}
\def\querySet{\cl{Q}}
\def\dataSet{\cl{X}}
\def\baseSet{\dataSet_{\text{Base}}}
\def\novelSet{\dataSet_{\text{Novel}}}
\def\testSet{\cl{T}}
\def\classes{\cl{C}}
\def\baseClass{\classes_{\text{Base}}}
\def\novelClass{\classes_{\text{Novel}}}

\def\TRUE{\ding{51}}
\def\FALSE{--}

\def\method{noHub\xspace}

\def\methodLongWithMarks{U\underline{n}if\underline{o}rm \underline{H}yperspherical Str\underline{u}cture-preserving Em\underline{b}eddings\xspace}
\def\methodS{noHub-S\xspace}

\def\methodSLongWithMarks{\method with \underline{S}upport labels\xspace}

\def\BEST#1{\textbf{#1}}
\def\SECONDBEST#1{\underline{#1}}
\def\MTCWITHCONF#1#2{#1~({\tiny #2})}

\def\VENUE#1#2#3{{\tiny(\textsc{#1}'#2 #3)}}
\def\OURS{\tiny(\textsc{Ours})}

\definecolor{arrowGreen}{rgb}{0,0.6,0}
\def\HIGHERBETTER{\!{\tiny\color{arrowGreen}\( \uparrow \)}}
\def\LOWERBETTER{\!{\tiny\color{arrowGreen}\( \downarrow \)}}

\def\TableMethodN{None}
\def\TableMethodL{L2~\VENUE{ArXiv}{19}{\cite{wangSimpleShotRevisitingNearestNeighbor2019}}}
\def\TableMethodC{CL2~\VENUE{ArXiv}{19}{\cite{wangSimpleShotRevisitingNearestNeighbor2019}}}
\def\TableMethodZ{ZN~\VENUE{ICCV}{21}{\cite{feiZScoreNormalizationHubness2021}}}
\def\TableMethodR{ReRep~\VENUE{ICML}{21}{\cite{cuiParameterlessTransductiveFeature2021}}}
\def\TableMethodE{EASE~\VENUE{CVPR}{22}{\cite{zhuEASEUnsupervisedDiscriminant2022}}}
\def\TableMethodT{TCPR~\VENUE{NeurIPS}{22}{\cite{xuAlleviatingSampleSelection2022}}}
\def\TableMethodOurs{noHub~\OURS}
\def\TableMethodOursS{noHub-S~\OURS}

\newcommand{\MainTabInput}[1]{
    \bgroup
    \scriptsize
    \centering
    \input{tab/main-results/#1} \\
    \egroup
}

\makeatletter
\DeclareRobustCommand\onedot{\futurelet\@let@token\@onedot}
\def\@onedot{\ifx\@let@token.\else.\null\fi\xspace}

\def\eg{\emph{e.g}\onedot}

\def\wrt{w.r.t\onedot} 
\def\etal{\emph{et al}\onedot}
\makeatother

\newcommand{\customparagraph}[2][0.05cm minus 0.05cm]{\vspace{#1}\noindent\textbf{#2.}~}

\newtheorem{proposition}{Proposition}
\newtheorem{definition}{Definition}
\newtheorem{lemma}{Lemma}

\SetKwComment{Comment}{/* }{ */}
\SetAlCapSty{}

\makeatletter
\newenvironment{customalgorithm}[1][htpb]{\def\@algocf@post@ruled{\kern\interspacealgoruled\hrule  height\algoheightrule\kern3pt\relax}%
\def\@algocf@capt@ruled{under}%
\setlength\algotitleheightrule{0pt}%
\begin{algorithm}[#1]}
{\end{algorithm}}
\makeatother

\let\tableFontSize\footnotesize
\let\algFontSize\small

\def\theTitle{Hubs and Hyperspheres: Reducing Hubness and Improving Transductive Few-shot Learning with Hyperspherical Embeddings}

\def\asp{~~}
\def\tsp{\hspace{.9ex}}
\def\theAuthor{%
    Daniel J. Trosten%
        \thanks{Equal contributions.}\tsp%
        \thanks{UiT Machine Learning group ({\url{machine-learning.uit.no}}) and Visual Intelligence Centre (\url{visual-intelligence.no}).},%
    \asp Rwiddhi Chakraborty%
        \footnotemark[1]\tsp%
        \footnotemark[2],%
    \asp Sigurd Løkse%
        \footnotemark[2],%
    \asp Kristoffer Knutsen Wickstrøm%
        \footnotemark[2],\\
    \asp Robert Jenssen%
        \footnotemark[2]\tsp%
        \thanks{Norwegian Computing Center.}\tsp
        \thanks{Department of Computer Science, University of Copenhagen.}\tsp%
        \thanks{Pioneer Centre for AI (\url{aicentre.dk}).},%
    \asp Michael C. Kampffmeyer%
        \footnotemark[2]\tsp
        \footnotemark[3]\\
    Department of Physics and Technology,~UiT The Arctic University of Norway\\
    {\tt \small firstname[.middle initial].lastname@uit.no}
}

\def\githubLink{\url{https://github.com/uitml/noHub}}

\def\definitionHypersphericalUniform{
    \begin{definition}[Uniform PDF on the hypersphere.]
        The uniform probability density function (PDF) on the unit hypersphere \( \cl \SPHERE_d = \{ \vec x \in \real^d \mid ||\vec x|| = 1 \} \subset \real^d \) is
        \begin{align}
            \label{eq:sphereUniform}
            \SPHEREUNIFORM(\vec x) = A_d^{-1} \delta(||\vec x|| - 1)
        \end{align}
        where \( A_d = \frac{2 \pi^{d/2}}{\Gamma(d/2)} \) is the surface area of \( \SPHERE_d \), and \( \delta(\cdot) \) is the Dirac delta distribution.
    \end{definition}
}

\def\propositionZeroMean{
    \begin{proposition}
        \label{prop:zeroMean}
        Suppose \( \vec X \) has PDF \( \SPHEREUNIFORM(\vec x) \).
        Then
        \begin{align}
         \E(\vec X) = 0
        \end{align}
    \end{proposition}
}

\def\propositionZeroGrad{
    \begin{proposition}
        \label{prop:zeroGradient}
        Let \( \Pi_{\vec p} \) be the tangent plane of \( \SPHERE_d \) at an arbitrary point \( \vec p \in \SPHERE_d \).
        Then, for any direction \( \vec \theta^* \) in \( \Pi_{\vec p} \) the directional derivative of \( \SPHEREUNIFORM \) along \( \vec \theta^* \) is
        \begin{align}
            \nabla_{\vec \theta^*} \SPHEREUNIFORM = 0
        \end{align}
    \end{proposition}
}

\def\propositionConnectionLE{
    \begin{proposition}[]
        \label{prop:conLE}
        Let \( W_{ij} = \frac{1}{2}\kappa p_{ij} \), where \( \sums{i,j}{} p_{ij} = 1 \), and let \( \TARGETS \).
        Then we have
        \begin{align}
            \Lalign = \sum_{i,j} \| \vec z_i - \vec z_j \|^2 W_{ij} - \kappa.
        \end{align}
    \end{proposition}
}

\def\propositionMaxEntropy{
    \begin{proposition}[Minimizing \( \Lunif \) maximizes entropy]
        \label{prop:maxEntropy}
        Let \( H_2(\cdot) \) be the 2-order R\'enyi entropy,
        estimated with a kernel density estimator using a Gaussian kernel.
        Then
        \begin{align}
            \ARGMIN \Lunif = \ARGMAX H_2(\vec z_1, \dots, \vec z_n).
        \end{align}
    \end{proposition}
}

\def\definitionNormalizedCountingMeasure{
    \begin{definition}[Normalized counting measure]
        The normalized counting measure associated with a set \( B \) on \( A \) is
        \begin{align}
            \nu_B(A) = \frac{| B \cap A|}{|B|}
        \end{align}
    \end{definition}
}

\def\definitionSurfaceAreaMeasure{
    \begin{definition}[Normalized surface area measure on \( \SPHERE_d \)]
        The normalized surface area measure on the hyperspehere \( \SPHERE_d \subset \real^d \), of a subset \( S' \subset \SPHERE_d \) is
        \begin{align}
            \sigma_d (S') = \frac{\int_{S'}\dd S}{\int_{\SPHERE_d} \dd S} = A_d^{-1} \int_{S'} \dd S
        \end{align}
        where \( A_d \) is defined as in Eq.~\eqref{eq:sphereUniform}, and \(\int \dd S\) denotes the surface integral on \( \SPHERE_d \).
    \end{definition}
}

\def\definitionWeakStarConvergence{
    \begin{definition}[Weak\(^*\) convergence of measures~\cite{wangUnderstandingContrastiveRepresentation2020}]
        A sequence of Borel measures \( \{ \mu_n \}_{n=1}^{\infty} \) in \( \real^d \) converges weak\(^*\) to a Borel measure \(\mu\), if for all continuous functions \(f: \real^d \to \real \),
        \begin{align}
            \lim\limits_{n\to\infty} \int f(x) \dd \mu_n(x) = \int f(x) \dd \mu(x)
        \end{align}
    \end{definition}
}

\def\propositionLunifMinimizer{
    \begin{proposition}[Minimizer of \( \Lunif \)]
        \label{prop:minLunif}
        For each \( n > 0 \), the \( n \) point minimizer of \( \Lunif \) is
        \begin{align}
            \vec z_1^{\star}, \dots, \vec z_n^{\star} = \ARGMIN \Lunif.
        \end{align}
        Then \( \nu_{\{ \vec z_1^{\star}, \dots, \vec z_n^{\star} \}} \) converge weak\( ^* \) to \( \sigma_d \) as \( n \to \infty \).
    \end{proposition}
}

\begin{document}
    \title{\theTitle}
    \author{\theAuthor}
    \maketitle

    \begin{abstract}
        Distance-based classification is frequently used in transductive few-shot learning (FSL).
However, due to the high-dimensionality of image representations, FSL classifiers are prone to suffer from the hubness problem, where a few points (hubs) occur frequently in multiple nearest neighbour lists of other points. 
Hubness negatively impacts distance-based classification when hubs from one class appear often among the nearest neighbors of points from another class, degrading the classifier's performance. 
To address the hubness problem in FSL, we first prove that hubness can be eliminated by distributing representations uniformly on the hypersphere. 
We then propose two new approaches to embed representations on the hypersphere, which we prove optimize a tradeoff between uniformity and local similarity preservation -- reducing hubness while retaining class structure.
Our experiments show that the proposed methods reduce hubness, and significantly improves transductive FSL accuracy for a wide range of classifiers\footnote{Code available at \githubLink.}.

    \end{abstract}

    \section{Introduction}
        While supervised deep learning has made a significant impact in areas where large amounts of labeled data are available~\cite{HeResnet2016,dosovitskiy2021anViT}, few-shot learning (FSL) has emerged as a promising alternative when labeled data is limited~\cite{kim2019edge, wangSimpleShotRevisitingNearestNeighbor2019,zikoLaplacianRegularizedFewShot2020,boudiafTransductiveInformationMaximization2020,veilleuxRealisticEvaluationTransductive2021,qiTransductiveFewShotClassification2021,lazarouIterativeLabelCleaning2021,wangRoleGlobalLabels2021,zhuEASEUnsupervisedDiscriminant2022,taoPoweringFinetuningFewShot2022,huPushingLimitsSimple2022}.
FSL aims to design classifiers that can discriminate between novel classes based on a few labeled instances, significantly reducing the cost of the labeling procedure.

In transductive FSL, one assumes access to the entire query set during evaluation.
This allows transductive FSL classifiers to learn representations from a larger number of samples, resulting in better performing classifiers.
However, many of these methods base their predictions on distances to prototypes for the novel classes~\cite{zikoLaplacianRegularizedFewShot2020,boudiafTransductiveInformationMaximization2020,veilleuxRealisticEvaluationTransductive2021,qiTransductiveFewShotClassification2021,lazarouIterativeLabelCleaning2021,zhuEASEUnsupervisedDiscriminant2022}.
This makes these methods susceptible to the hubness problem~\cite{shigetoRidgeRegressionHubness2015,suzukiCenteringSimilarityMeasures2013,radovanovicHubsSpacePopular2010, haraLocalizedCenteringReducing2015}, where certain exemplar points (hubs) appear among the nearest neighbours of many other points.
If a support sample is a hub, many query samples will be assigned to it regardless of their true label, resulting in low accuracy.
If more training data is available, this effect can be reduced by increasing the number of labeled samples in the classification rule -- but this is impossible in FSL.

\begin{figure}
    \centering
    \includegraphics[width=0.85\columnwidth]{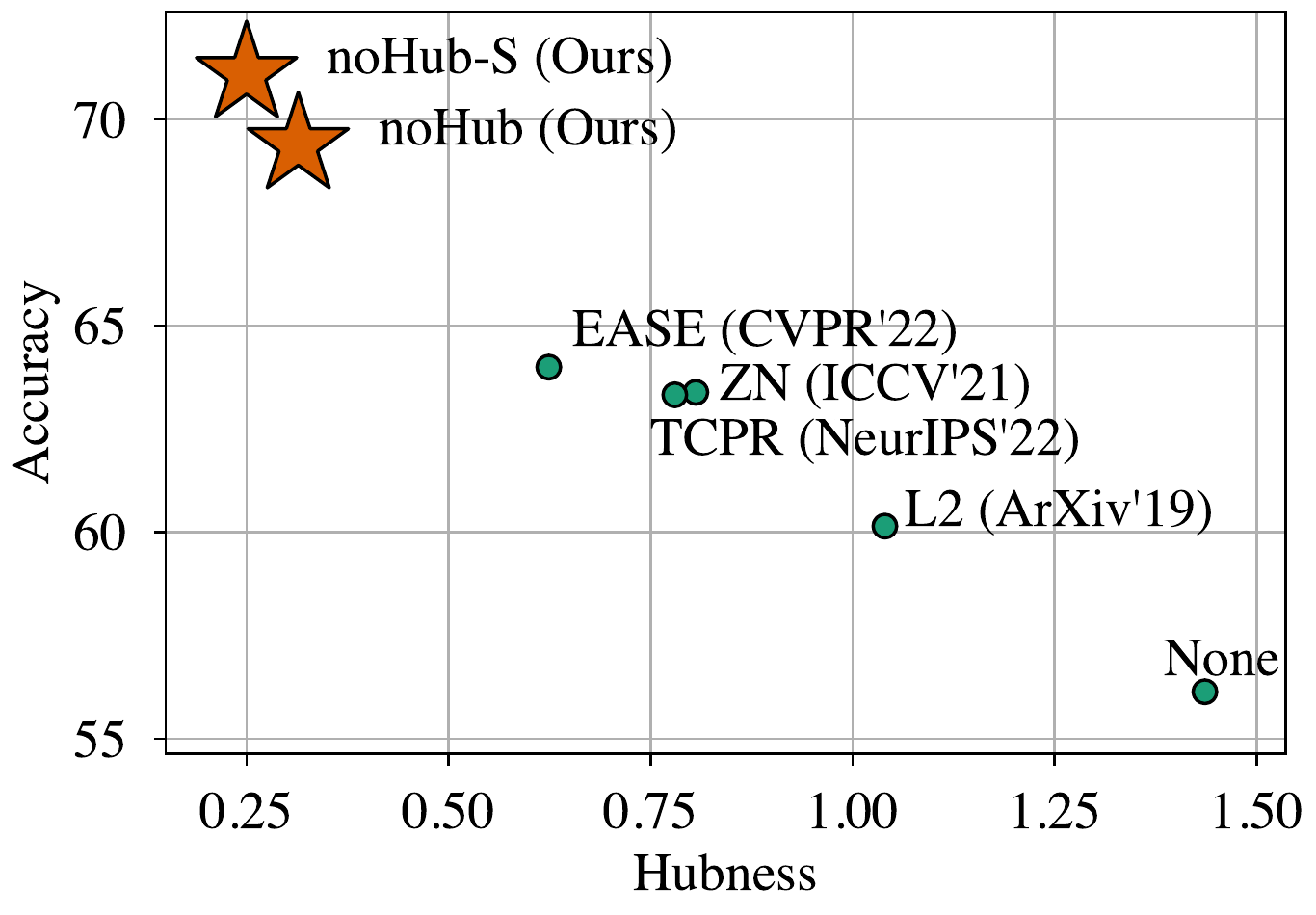}
    \vspace{-.3cm}
    \caption{Few-shot accuracy increases when hubness decreases.
    The figure shows the 1-shot accuracy when classifying different embeddings with SimpleShot~\cite{wangSimpleShotRevisitingNearestNeighbor2019} on mini-ImageNet~\cite{vinyalsMatchingNetworksOne2016}.}
    \label{fig:motivatingPlot}
    \vspace{-0.4cm}
\end{figure}

Several approaches have recently been proposed to embed samples in a space where the FSL classifier's performance is improved~\cite{wangSimpleShotRevisitingNearestNeighbor2019,cuiParameterlessTransductiveFeature2021,feiZScoreNormalizationHubness2021,lePOODLEImprovingFewshot2021,zhuEASEUnsupervisedDiscriminant2022,xuAlleviatingSampleSelection2022,chikontweCADCoAdaptingDiscriminative2022}.
However, only one of these directly addresses the hubness problem.
Fei \etal~\cite{feiZScoreNormalizationHubness2021} show that embedding representations on a hypersphere with zero mean reduces hubness.
They advocate the use of Z-score normalization (ZN) along the feature axis of each representation, and show empirically that ZN can reduce hubness in FSL.
However, ZN does not guarantee a data mean of zero, meaning that hubness can still occur after ZN.

In this paper we propose a principled approach to embed representations in FSL, which both reduces hubness and improves classification performance.
First, we prove that hubness can be eliminated by embedding representations uniformly on the hypersphere.
However, distributing representations uniformly on the hypersphere without any additional constraints will likely break the class structure which is present in the representation space -- hurting the performance of the downstream classifier.
Thus, in order to both reduce hubness and preserve the class structure in the representation space, we propose two new embedding methods for FSL.
Our methods, \methodLongWithMarks (\method) and \methodSLongWithMarks (\methodS), leverage a decomposition of the Kullback-Leibler divergence between representation and embedding similarities, to optimize a tradeoff between Local Similarity Preservation (LSP) and uniformity on the hypersphere.
The latter method, \methodS, also leverages label information from the support samples to further increase the class separability in the embedding space.

Figure~\ref{fig:motivatingPlot} illustrates the correspondence between hubness and accuracy in FSL.
Our methods have both the \emph{least hubness} and \emph{highest accuracy} among several recent embedding techniques for FSL.

Our contributions are summarized as follows.
\begin{itemize}
    \item We prove that the uniform distribution on the hypersphere has zero hubness and that embedding points uniformly on the hypersphere thus alleviates the hubness problem in distance-based classification for transductive FSL.
    \item We propose \method and \methodS to embed representations on the hypersphere, and prove that these methods optimize a tradeoff between LSP and uniformity. The resulting embeddings are therefore approximately uniform, while simultaneously preserving the class structure in the embedding space.
    \item Extensive experimental results demonstrate that \method and \methodS outperform current state-of-the-art embedding approaches, boosting the performance of a wide range of transductive FSL classifiers, for multiple datasets and feature extractors.
\end{itemize}

    \section{Related Work}
        \label{sec:hubnessFSL}
        \customparagraph{The hubness problem}
    The hubness problem refers to the emergence of \emph{hubs} in collections of points in high-dimensional vector spaces~\cite{radovanovicHubsSpacePopular2010}.
    Hubs are points that appear among the nearest neighbors of many other points, and are therefore likely to have a significant influence on \eg nearest neighbor-based classification.
    Radovanovic \etal~\cite{radovanovicHubsSpacePopular2010} showed that points closer to the expected data mean are more likely be among the nearest neighbors of other points, indicating that these points are more likely to be hubs.
    Hubness can also be seen as a result of large density gradients~\cite{haraFlatteningDensityGradient2016}, as points in high-density areas are more likely to be hubs.
    The hubness problem is thus an intrinsic property of data distributions in high-dimensional vector spaces, and not an artifact occurring in particular datasets.
    It is therefore important to take the hubness into account when designing classification systems in high-dimensional vector spaces.

\customparagraph{Hubness in FSL}
    Many recent methods in FSL rely on distance-based classification in high-dimensional representation spaces~\cite{wangSimpleShotRevisitingNearestNeighbor2019,zikoLaplacianRegularizedFewShot2020,boudiafTransductiveInformationMaximization2020,NhanSEN, zhang2021iept, ye2018, allen2019},
    making them vulnerable to the hubness problem.
    Fei \etal~\cite{feiZScoreNormalizationHubness2021} show that hyperspherical representations with zero mean reduce hubness.
    Motivated by this insight, they suggest that representations should have zero mean and unit standard deviation (ZN) \emph{along the feature dimension}.
    This effectively projects samples onto the hyperplane orthogonal to the vector with all elements \( = 1 \), and pushes them to the hypersphere with radius \( \sqrt{d} \), where \( d \) is the dimensionality of the representation space.
    Although ZN is empirically shown to reduce hubness, it does not guarantee that the data mean is zero.
    The normalized representations can therefore still suffer from hubness, potentially decreasing FSL performance.

\customparagraph{Embeddings in FSL}
    FSL classifiers often operate on embeddings of representations instead of the representations themselves, to improve the classifier's ability to generalize to novel classes~\cite{wangSimpleShotRevisitingNearestNeighbor2019,cuiParameterlessTransductiveFeature2021,zhuEASEUnsupervisedDiscriminant2022,xuAlleviatingSampleSelection2022}.
    Earlier works use the L2 normalization and Centered L2 normalization to embed representations on the hypersphere~\cite{wangSimpleShotRevisitingNearestNeighbor2019}.
    Among more recent embedding techniques, ReRep~\cite{cuiParameterlessTransductiveFeature2021} performs a two-step fusing operation on both the support and query features with an attention mechanism.
    EASE~\cite{zhuEASEUnsupervisedDiscriminant2022} combines both support and query samples into a single sample set, and jointly learns a similarity and dissimilarity matrix, encouraging similar features to be embedded closer, and dissimilar features to be embedded far away.
    TCPR~\cite{xuAlleviatingSampleSelection2022} computes the top-k neighbours of each test sample from the base data, computes the centroid, and removes the feature components in the direction of the centroid.
    Although these methods generally lead to a reduction in hubness and an increase in performance (see Figure~\ref{fig:motivatingPlot}), they are not explicitly designed to address the hubness problem resulting in suboptimal hubness reduction and performance.
    In contrast, our proposed \method and \methodS directly leverage our theoretic insights to target the root of the hubness problem.

\customparagraph{Hyperspherical uniformity}
    Benefits of uniform hyperspherical representations have previously been studied for contrastive self-supervised learning (SSL)~\cite{wangUnderstandingContrastiveRepresentation2020}.
    Our work differs from~\cite{wangUnderstandingContrastiveRepresentation2020} on several key points.
    First, we study a non-parametric embedding of support and query samples for FSL, which is a fundamentally different task from contrastive SSL.
    Second, the contrastive loss studied in~\cite{wangUnderstandingContrastiveRepresentation2020} is a combination of different cross-entropies, making it different from our KL-loss.
    Finally, we introduce a tradeoff-parameter between uniformity and LSP, and connect our theoretical results to hubness and Laplacian Eigenmaps.

    \section{Hyperspherical Uniform Eliminates Hubness}
        \label{sec:hypersphericalHubness}
        
We will now show that hubness can be eliminated completely by embedding representations \emph{uniformly} on the hypersphere\footnote{Our results assume hyperspheres with unit radius, but can easily be extended to hyperspheres with arbitrary radii.}.

\definitionHypersphericalUniform
We then have the following propositions\footnote{The proofs for all propositions are included in the supplementary.} for random vectors with this PDF.

\propositionZeroMean
\propositionZeroGrad

These two propositions show that the hyperspherical uniform has
\begin{enumerate*}[label=(\roman*)]
    \item zero mean;
    and \item zero density gradient along all directions tangent to the hypersphere's surface, at all points on the hypersphere.
\end{enumerate*}
The hyperspherical uniform thus provably eliminates hubness, both in the sense of having a zero data mean, and having zero density gradient everywhere.
We note that the latter property is un-attainable in Euclidean space, as it is impossible to define a uniform distribution over the whole space.
It is therefore necessary to embed points on a non-Euclidean sub-manifold in order to eliminate hubness.

\begin{figure}
    \centering
    \includegraphics[width=0.85\columnwidth]{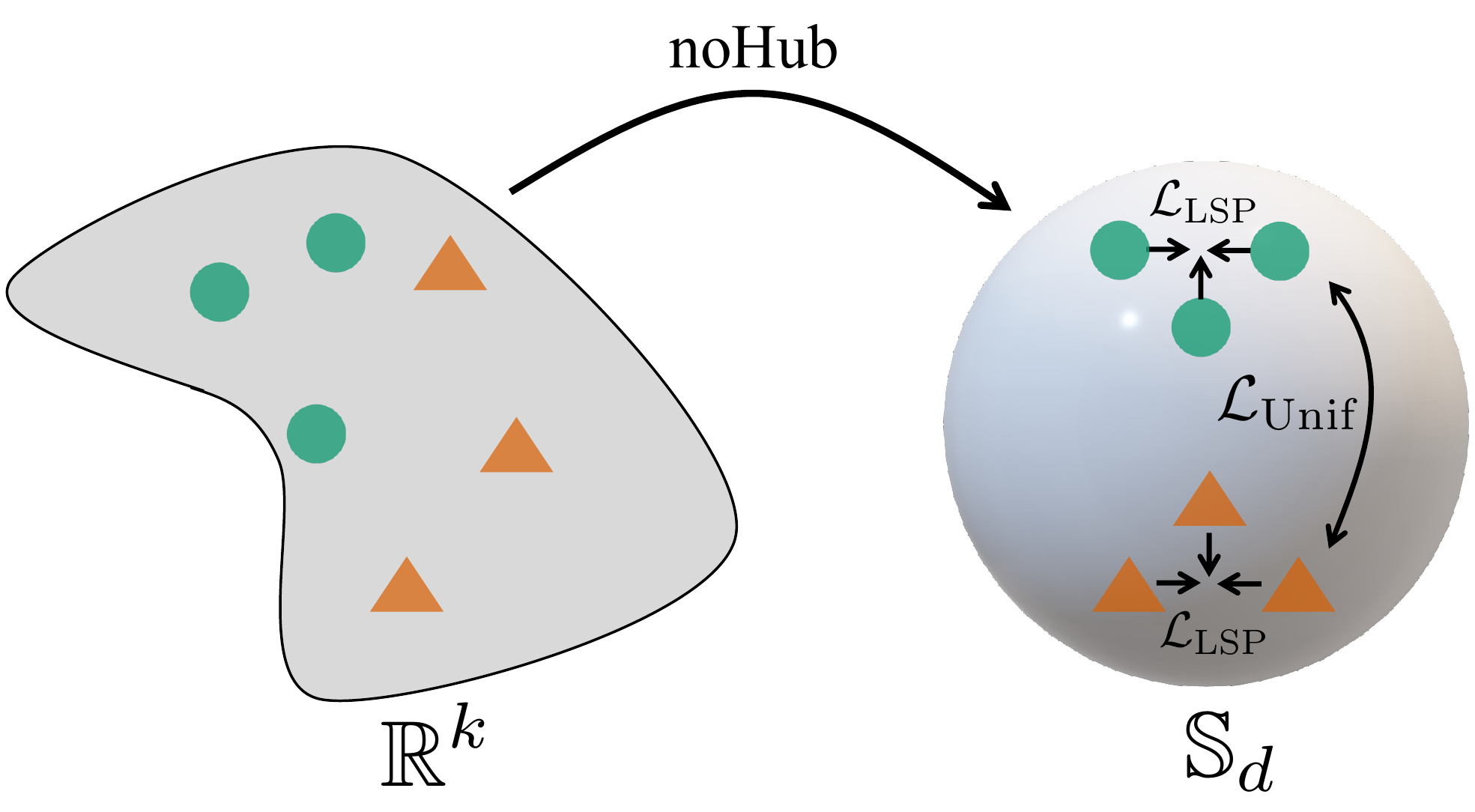}
    \caption{Illustration of the \method embedding. Given representations \( \in \real^k \), $\Lalign$ preserves local similarities. $\Lunif$ simultaneously encourages uniformity in the embedding space $\SPHERE_d$. This feature embedding framework helps reduce hubness while improving classification performance.}
    \label{fig:noHub}
\end{figure}
        
    \section{Method}
        \label{sec:method}
        In the preceding section, we proved that uniform embeddings on the hypersphere eliminate hubness. 
However, naïvely placing points uniformly on the hypersphere does not incorporate the inherent class structure in the data, leading to poor FSL performance.
Thus, there exists a tradeoff between uniformity on the hypersphere and the preservation of local similarities.
To address this tradeoff, we introduce two novel embedding approaches for FSL, namely \method and \methodS.
\method (Sec.~\ref{subsec:nohub}) incorporates a novel loss function for embeddings on the hypersphere, while \methodS (Sec.~\ref{subsec:nohubs}), guides \method with additional label information, which should act as a supervisory signal for a class-aware embedding that leads to improved classification performance.
Figure~\ref{fig:noHub} provides an overview of the proposed \method method.
We also note that, since our approach generates embeddings, they are compatible with most transductive FSL classifier.

\customparagraph{Few-shot Preliminaries}
    Assume we have a large labeled \emph{base} dataset \( \baseSet = \{ (\vec x_i, y_i) \mid y_i \in \baseClass;~ i = 1, \ldots, n_{\text{Base}} \}\), where \( \vec x_i \) and \( y_i \) denotes the raw features and labels, respectively.
    Let \( \baseClass \) denote the set of classes for the base dataset.
    In the few--shot scenario, we assume that we are given another labeled dataset \( \novelSet = \{ (\vec x_i, y_i) \mid y_i \in \novelClass;~ i = 1, \ldots, n_{\text{Novel}} \} \) from \emph{novel}, previously unseen classes \( \novelClass \), satisfying \(\baseClass \cap \novelClass = \emptyset \).
    In addition, we have a test set \( \testSet,\; \testSet \cap \novelSet = \emptyset \), also from \( \novelClass \).

    In a \( K \)--way \( N_S\)--shot FSL problem, we create randomly sampled \emph{tasks} (or episodes), with data from \( K \) randomly chosen novel classes.
    Each task consists of a \emph{support} set \( \supportSet \subset \novelSet \) and a \emph{query} set \( \querySet \subset \testSet \).
    The support set contains \( |\supportSet| = N_S \cdot K \) random examples (\( N_S \) random examples from each of the \( K \) classes).
    The query set contains \( |\querySet|= N_Q \cdot K \) random examples, sampled from the same \( K \) classes.
    The goal of FSL is then to predict the class of samples \( \vec x \in \querySet \) by exploiting the labeled support set \( \supportSet \), using a model trained on the base classes \( \baseClass \).
    We assume a fixed feature extractor, trained on the base classes, which maps the raw input data to the representations \( \vec x_i \).

\subsection{\method: \methodLongWithMarks}
    \label{subsec:nohub}
    We design an embedding method that encourages uniformity on the hypersphere, and simultaneously preserves local similarity structure.
    Given the support and query representations $\vec x_1, \dots, \vec x_n \in \real^k$, \(n = K(N_S + N_Q)\) , we wish to find suitable embeddings $\vec z_1, \dots, \vec z_n \in \SPHERE_d$, where local similarities are preserved.
    For both representations and embeddings, we quantify similarities using a softmax over pairwise cosine similarities
    \begin{align}
        \label{eq:pij}
        p_{ij} = \frac{p_{i|j} + p_{j | i}}{2}, \quad p_{i|j} = \frac{\exp(\kappa_i \frac{\vec x_i\T \vec x_j}{||\vec x_i||\cdot||\vec x_j||} )}{\sums{l, m}{}\exp(\kappa_i \frac{\vec x_l\T \vec x_m}{||\vec x_l||\cdot||\vec x_m||})}
    \end{align}
    and
    \begin{align}
        \label{eq:qij}
        q_{ij} = \frac{\exp(\kappa \vec z_i\T \vec z_j)}{\sums{l, m}{} \exp(\kappa \vec z_l\T \vec z_m)},
    \end{align}
    where $\kappa_i$ is chosen such that the effective number of neighbours of $\vec x_i$ equals a pre-defined perplexity\footnote{Details on the computation of the \( \kappa_i \) are provided in the supplementary.}.
    As in~\cite{Maaten,vmfsne}, local similarity preservation can now be achieved by minimizing the Kullback-Leibler (KL) divergence between the \( p_{ij} \) and the \( q_{ij}\)
    \begin{align}
        \label{eq:kl}
       KL(P || Q) = \sums{i, j}{} p_{ij} \log \frac{p_{ij}}{q_{ij}}.
    \end{align}
    However, instead of directly minimizing \( KL(P || Q) \), we find that the minimization problem is equivalent to minimizing the sum of two loss functions\footnote{Intermediate steps are provided in the supplementary.}
    \begin{align}
        \ARGMIN KL(P||Q) = \ARGMIN \Lalign + \Lunif
    \end{align}
    where
    \begin{align}
        & \Lalign = - \kappa \sums{i, j}{} p_{ij} \vec z_i\T \vec z_j, \\
        & \Lunif = \log \sums{l, m}{} \exp(\kappa \vec z_l\T \vec z_m).
    \end{align}

    \begin{customalgorithm}
        \algFontSize
        \DontPrintSemicolon %
\KwIn{Features $\in \real^k$, $\{\vec x_1, \dots, \vec x_n\}$;
perplexity, $P$; number of iterations, \( T \) ; learning rate, $\eta$.}
\KwOut{Embeddings $\in \SPHERE_d$, $\{\vec z_1, \dots, \vec z_n\}$}

Compute $p_{ij}$ from Eq~\eqref{eq:pij}\;
Initialize solution $\vec Z^0 = \{\vec z_1, \dots, \vec z_n\}$ with PCA\;
\For{$i \gets 1$ \textbf{to} $T$}{
  Compute $q_{ij}$ from Eq.~\eqref{eq:qij}\;
  Compute gradients $\frac{\dd\Lfinal}{\dd \vec Z}$, using loss from Eq.~\eqref{eq:Lfinal}\;
  Update $\vec Z^t$ using the ADAM optimizer with learning rate $\eta$~\cite{adam}\;
  Re-normalize elements of $\vec Z^t$ using \( L_2 \) normalization\;
}
\Return{$\vec Z^T$}\;
        \caption{\method algorithm for embeddings on the hypersphere}
        \label{alg:nohub}
    \end{customalgorithm}

    In Sec.~\ref{sec:theoreticalResults} we provide a thorough theoretical analysis of these losses, and how they relate to LSP and uniformity on the hypersphere.
    Essentially, \( \Lalign \) is responsible for the local similarity preservation by ensuring that the embedding similarities (\( \vec z_i\T \vec z_j \)) are high whenever the representation similarities (\(p_{ij}\)) are high. \( \Lunif \) on the other hand, can be interpreted as a negative entropy on \( \SPHERE_d \), and is thus minimized when the embeddings are uniformly distributed on \( \SPHERE_d \).
    This is discussed in more detail in Sec.~\ref{sec:theoreticalResults}.

    Based on the decomposition of the KL divergence, and the subsequent interpretation of the two terms, we formulate the loss in \method as the following tradeoff between LSP and uniformity
    \begin{equation}
        \label{eq:Lfinal}
        \Lfinal = \alpha\Lalign + (1-\alpha) \Lunif
    \end{equation}
    where $\alpha$ is a weight parameter quantifying the tradeoff.
    \( \Lfinal \) can then be optimized directly with gradient descent.
    The entire procedure is outlined in Algorithm~\ref{alg:nohub}.

\subsection{\methodS: \methodSLongWithMarks}
    \label{subsec:nohubs}
    In order to strengthen the class structure in the embedding space, we modify \(\Lalign \) and \(\Lunif\) by exploiting the additional information provided by the support labels.
    For \( \Lalign \), we change the similarity function in $p_{ij}$ such that
    \begin{align}
        \label{eq:modpij1}
        p_{i|j} = \frac{\exp(\kappa_i s_{\vec x}(\vec x_i, \vec x_j))}{\sums{l, m}{}\exp(\kappa_i s_{\vec x}(\vec x_l, \vec x_m))}
    \end{align}
    where
    \begin{align}
        \label{eq:modpij2}
        s_{\vec x}(\vec x_i, \vec x_j) = \begin{cases}
            1 &\text{if $\vec x_i, \vec x_j \in \supportSet$, and $y_i = y_j$}\\
            -1 &\text{if $\vec x_i, \vec x_j \in \supportSet$, and $y_i \neq y_j$}\\
            $$\vec x_i\T \vec x_j$$ &\text{otherwise}
        \end{cases}.
    \end{align}
    With this, we encourage embeddings for support samples in the \emph{same class} to be maximally similar, and support samples in \emph{different classes} to be maximally dissimilar.
    Similarly, for \( \Lunif \)
    \begin{align}
        \Lunif = \log \sums{l, m}{} \exp(\kappa s_{\vec z}(\vec z_i, \vec z_j))
    \end{align}
    where
    \begin{align}
        \label{eq:moduni}
        s_{\vec z}(\vec z_i, \vec z_j) = \begin{cases}
            -\infty, &\text{if $\vec z_i, \vec z_j \in \supportSet$, and $y_i = y_j$}\\
            \varepsilon ~ \vec z_i\T \vec z_j, &\text{if $\vec z_i, \vec z_j \in \supportSet$, and $y_i \neq y_j$}\\
            \vec z_i\T, \vec z_j &\text{otherwise}
        \end{cases}
    \end{align}
    where \( \varepsilon \) is a hyperparameter.
    This puts more emphasis on between-class uniformity by weighting the similarity higher for embeddings belonging to different classes (\( \varepsilon > 1 \)), and ignoring the similarity between embeddings belonging to the same class\footnote{Although any constant value would achieve the same result, we set the similarity to \(-\infty\) in this case to remove the contribution to the final loss.}.
    The final loss function is the same as Eq.~\eqref{eq:Lfinal}, but with the additional label-informed similarities in Eqs.~\eqref{eq:modpij1}--\eqref{eq:moduni}.

    \section{Theoretical Results}
        \label{sec:theoreticalResults}
        In this section we provide a theoretical analysis of \( \Lalign \) and \( \Lunif \).
Based on our analysis, we interpret these losses with regards to the Laplacian Eigenmaps algorithm and R\'enyi entropy, respectively.

\propositionConnectionLE
\propositionMaxEntropy
\definitionNormalizedCountingMeasure
\definitionSurfaceAreaMeasure
\definitionWeakStarConvergence
\propositionLunifMinimizer

\customparagraph{Interpretation of Proposition~\ref{prop:conLE}--\ref{prop:minLunif}}
    Proposition~\ref{prop:conLE} states an alternative formulation of \( \Lalign \), under the hyperspherical assumption.
    We recognize this formulation as the loss function in Laplacian Eigenmaps~\cite{belkinLaplacianEigenmaps2003}, which is known to produce \emph{local similarity-preserving} embeddings from graph data.
    When unconstrained, this loss has a trivial solution where the embeddings for all representations are equal.
    This is avoided in our case since \( \Lfinal \) (Eq.~\eqref{eq:Lfinal}) can be interpreted as the Lagrangian of minimizing \( \Lalign \) subject to a specified level of \emph{entropy}, by Proposition~\ref{prop:maxEntropy}.
    
    Finally, Proposition~\ref{prop:minLunif} states that the normalized counting measure associated with the set of points that minimize \( \Lunif \), converges to the normalized surface area measure on the sphere.
    Since \( \SPHEREUNIFORM \) is the density function associated with this measure, the points that minimize \( \Lunif\) will tend to be uniform on the sphere.
    Consequently, minimizing \( \Lalign \) also minimizes hubness, by Propositions~\ref{prop:zeroMean} and~\ref{prop:zeroGradient}.

    \section{Experiments}
        \subsection{Setup}
    \begin{table*}
    \setlength{\tabcolsep}{3.5mm}
        \centering
        {\tableFontSize \begin{tabular}{rccccccc} \toprule
    && \multicolumn{2}{c}{mini} & \multicolumn{2}{c}{tiered} & \multicolumn{2}{c}{CUB} \\
    Embedding & Feature Extractor & 1-shot \HIGHERBETTER & 5-shot \HIGHERBETTER & 1-shot \HIGHERBETTER & 5-shot \HIGHERBETTER & 1-shot \HIGHERBETTER & 5-shot \HIGHERBETTER \\ \midrule %
    \TableMethodN & ResNet-18 & \MTCWITHCONF{20.0}{0.0}\( ^* \)                & \MTCWITHCONF{20.0}{0.0}\( ^* \)                & \MTCWITHCONF{20.0}{0.0}\( ^* \)                & \MTCWITHCONF{20.0}{0.0}\( ^* \)                & \MTCWITHCONF{20.0}{0.0}\( ^* \) & \MTCWITHCONF{20.0}{0.0}\( ^* \) \\
    \TableMethodL & ResNet-18 & \MTCWITHCONF{73.77}{0.24}              & \MTCWITHCONF{83.14}{0.14}              & \MTCWITHCONF{80.46}{0.26}              & \MTCWITHCONF{87.04}{0.16}              & \MTCWITHCONF{83.1}{0.23} & \MTCWITHCONF{89.48}{0.12} \\
    \TableMethodC & ResNet-18 & \MTCWITHCONF{75.56}{0.26}              & \MTCWITHCONF{84.04}{0.15}              & \MTCWITHCONF{82.1}{0.26}               & \MTCWITHCONF{87.9}{0.16}               & \MTCWITHCONF{84.35}{0.24} & \MTCWITHCONF{90.14}{0.12} \\
    \TableMethodZ & ResNet-18 & \MTCWITHCONF{20.0}{0.0}\( ^* \)                & \MTCWITHCONF{20.0}{0.0}\( ^* \)                & \MTCWITHCONF{20.0}{0.0}\( ^* \)                & \MTCWITHCONF{20.0}{0.0}\( ^* \)                & \MTCWITHCONF{20.0}{0.0}\( ^* \) & \MTCWITHCONF{20.0}{0.0}\( ^* \) \\
    \TableMethodR & ResNet-18 & \MTCWITHCONF{20.0}{0.0}\( ^* \)                & \MTCWITHCONF{20.0}{0.0}\( ^* \)                & \MTCWITHCONF{20.0}{0.0}\( ^* \)                & \MTCWITHCONF{20.0}{0.0}\( ^* \)                & \MTCWITHCONF{20.0}{0.0}\( ^* \) & \MTCWITHCONF{20.0}{0.0}\( ^* \) \\
    \TableMethodE & ResNet-18 & \MTCWITHCONF{76.05}{0.27}              & \SECONDBEST{\MTCWITHCONF{84.61}{0.15}} & \MTCWITHCONF{82.57}{0.27}              & \SECONDBEST{\MTCWITHCONF{88.33}{0.16}} & \MTCWITHCONF{85.24}{0.24} & \MTCWITHCONF{90.42}{0.12} \\
    \TableMethodT & ResNet-18 & \MTCWITHCONF{75.99}{0.26}              & \MTCWITHCONF{84.39}{0.15}              & \MTCWITHCONF{82.65}{0.26}              & \MTCWITHCONF{88.26}{0.16}              & \MTCWITHCONF{85.34}{0.23} & \SECONDBEST{\MTCWITHCONF{90.5}{0.11}} \\
    \TableMethodOurs & ResNet-18 & \SECONDBEST{\MTCWITHCONF{76.65}{0.28}} & \MTCWITHCONF{84.05}{0.16}              & \SECONDBEST{\MTCWITHCONF{82.94}{0.27}} & \MTCWITHCONF{87.87}{0.17}              & \BEST{\MTCWITHCONF{85.88}{0.24}} & \MTCWITHCONF{90.34}{0.12} \\
    \TableMethodOursS & ResNet-18 & \BEST{\MTCWITHCONF{76.68}{0.28}}       & \BEST{\MTCWITHCONF{84.67}{0.15}}       & \BEST{\MTCWITHCONF{83.09}{0.27}}       & \BEST{\MTCWITHCONF{88.43}{0.16}}       & \SECONDBEST{\MTCWITHCONF{85.81}{0.24}} & \BEST{\MTCWITHCONF{90.52}{0.12}} \\
    \midrule
    \TableMethodN & WideRes28-10 & \MTCWITHCONF{45.69}{0.31}              & \MTCWITHCONF{58.82}{0.31}              & \MTCWITHCONF{75.29}{0.28}             & \MTCWITHCONF{82.56}{0.22}               & \MTCWITHCONF{61.36}{0.55} & \MTCWITHCONF{82.22}{0.37} \\
    \TableMethodL & WideRes28-10 & \MTCWITHCONF{80.2}{0.23}               & \MTCWITHCONF{87.11}{0.13}              & \MTCWITHCONF{80.89}{0.26}             & \MTCWITHCONF{87.34}{0.15}               & \MTCWITHCONF{91.98}{0.18} & \MTCWITHCONF{94.15}{0.1} \\
    \TableMethodC & WideRes28-10 & \MTCWITHCONF{75.23}{0.27}              & \MTCWITHCONF{83.99}{0.16}              & \MTCWITHCONF{79.59}{0.27}             & \MTCWITHCONF{86.71}{0.16}               & \MTCWITHCONF{92.17}{0.18} & \MTCWITHCONF{94.48}{0.09} \\
    \TableMethodZ & WideRes28-10 & \MTCWITHCONF{20.0}{0.0}\( ^* \)                & \MTCWITHCONF{20.0}{0.0}\( ^* \)                & \MTCWITHCONF{20.0}{0.0}\( ^* \)               & \MTCWITHCONF{20.0}{0.0}\( ^* \)                 & \MTCWITHCONF{20.0}{0.0}\( ^* \) & \MTCWITHCONF{20.0}{0.0}\( ^* \) \\
    \TableMethodR & WideRes28-10 & \MTCWITHCONF{36.69}{0.28}              & \MTCWITHCONF{36.41}{0.3}               & \MTCWITHCONF{67.41}{0.29}             & \MTCWITHCONF{76.49}{0.24}               & \MTCWITHCONF{57.62}{0.56} & \MTCWITHCONF{60.36}{0.6} \\
    \TableMethodE & WideRes28-10 & \MTCWITHCONF{81.19}{0.25}              & \SECONDBEST{\MTCWITHCONF{87.82}{0.13}} & \MTCWITHCONF{82.04}{0.26}             & \SECONDBEST{\MTCWITHCONF{88.06}{0.16}}  & \MTCWITHCONF{91.99}{0.19} & \MTCWITHCONF{94.36}{0.09} \\
    \TableMethodT & WideRes28-10 & \MTCWITHCONF{81.27}{0.24}              & \MTCWITHCONF{87.8}{0.13}               & \MTCWITHCONF{81.89}{0.26}             & \MTCWITHCONF{87.95}{0.16}               & \MTCWITHCONF{91.91}{0.18} & \MTCWITHCONF{94.25}{0.1} \\
    \TableMethodOurs & WideRes28-10 & \SECONDBEST{\MTCWITHCONF{81.97}{0.25}} & \MTCWITHCONF{87.78}{0.14}              & \SECONDBEST{\MTCWITHCONF{82.8}{0.27}} & \MTCWITHCONF{87.99}{0.17}               & \SECONDBEST{\MTCWITHCONF{92.53}{0.18}} & \SECONDBEST{\MTCWITHCONF{94.56}{0.09}} \\
    \TableMethodOursS & WideRes28-10 & \BEST{\MTCWITHCONF{82.0}{0.26}}        & \BEST{\MTCWITHCONF{88.03}{0.13}}       & \BEST{\MTCWITHCONF{82.85}{0.27}}      & \BEST{\MTCWITHCONF{88.31}{0.16}}        & \BEST{\MTCWITHCONF{92.63}{0.18}} & \BEST{\MTCWITHCONF{94.69}{0.09}} \\
    \bottomrule
\end{tabular}

}
        \caption{Accuracies~({\footnotesize Confidence interval}) with the SIAMESE~\cite{zhuEASEUnsupervisedDiscriminant2022} classifier for different embedding approaches. Best and second best performance are denoted in  {\textbf{bold}} and \underline{underlined}, respectively. $^*$The SIAMESE classifier is sensitive to the norm of the embedding, thus leading to detrimental performance for some of the embedding approaches.}
        \label{tab:siameseAcc}
    \end{table*}
    \vspace{-0.2cm}
    \customparagraph{Implementation details}
        Our implementation is in PyTorch~\cite{paszkePytorch2019}.
        We optimize \method and \methodS for \( T = 150 \) iterations, using the Adam optimizer~\cite{adam} with learning rate \( \eta = 0.1 \).
        The other hyperparameters were chosen based on validation performance on the respective datasets\footnote{Hyperparameter configurations for all experiments are included in the supplementary.}.
        We analyze the effect of \( \alpha \) in Sec.~\ref{subsec:results}.
        Analyses of the \( \kappa \) and \( \epsilon \) hyperparameters are provided in the supplementary.

    \customparagraph{Initialization}
        Since \method and \methodS reduce the embedding dimensionality (\( d = 400 \)), we initialize embeddings with Principal Component Analysis (PCA)~\cite{jolliffe2002principal}, instead of a naïve, random initialization.
        The PCA initialization is computationally efficient, and approximately preserves global structure.
        It also resulted in faster convergence and better performance, compared to random initialization.

    \customparagraph{Base feature extractors}
        We use the standard networks
        ResNet-18~\cite{HeResnet2016} and Wide-Res28-10~\cite{ZagoruykoWideResNet2016} as the base feature extractors with pretrained weights from~\cite{veilleuxRealisticEvaluationTransductive2021} and~\cite{manglaChartingRightManifold2020}, respectively.
        
    \customparagraph{Datasets}
        Following common practice, we evaluate FSL performance on the \textit{mini-ImageNet (mini)}~\cite{vinyalsMatchingNetworksOne2016}, \textit{tiered-ImageNet (tiered)}~\cite{renMetalearningSemisupervisedFewshot2018}, and \textit{CUB-200 (CUB)}~\cite{CUB200} datasets.

    \customparagraph{Classifiers}
        We evaluate the baseline embeddings and our proposed methods using both established and recent FSL classifiers: \textit{SimpleShot}~\cite{wangSimpleShotRevisitingNearestNeighbor2019}, \textit{LaplacianShot}~\cite{zikoLaplacianRegularizedFewShot2020}, $\mathit{\alpha}-$\textit{TIM}~\cite{veilleuxRealisticEvaluationTransductive2021}, \textit{Oblique Manifold (OM)}~\cite{qiTransductiveFewShotClassification2021}, \textit{iLPC}~\cite{lazarouIterativeLabelCleaning2021}, and \textit{SIAMESE}~\cite{zhuEASEUnsupervisedDiscriminant2022}.

    \customparagraph{Baseline Embeddings}
        We compare our proposed method with a wide range of techniques for embedding the base features: \textit{None} (No embedding of base features), \textit{L2}~\cite{wangSimpleShotRevisitingNearestNeighbor2019}, \textit{Centered L2}~\cite{wangSimpleShotRevisitingNearestNeighbor2019}, \textit{ZN}~\cite{feiZScoreNormalizationHubness2021}, \textit{ReRep}~\cite{cuiParameterlessTransductiveFeature2021},  \textit{EASE}~\cite{zhuEASEUnsupervisedDiscriminant2022}, and \textit{TCPR}~\cite{xuAlleviatingSampleSelection2022}.

    \customparagraph{Evaluation protocol}
        We follow the standard evaluation protocol in FSL and calculate the accuracy for 1-shot and 5-shot classification with 15 images per class in the query set.
        We evaluate on \( 10000 \)  episodes, as is standard practice in FSL.
        Additionally, we evaluate the hubness of the representations after embedding using two common hubness metrics, namely the skewness (Sk) of the k-occurrence distribution~\cite{radovanovicHubsSpacePopular2010} and the hub occurrence (HO)~\cite{flexer2015choosing}, which measures the percentage of hubs in the nearest neighbour lists of all points.

\subsection{Results}
    \label{subsec:results}
    \customparagraph{Comparison to the state-of-the-art}
        To illustrate the effectiveness of \method and \methodS as an embedding approach for FSL, we consider the current state-of-the-art FSL method, which leverages the EASE embedding and obtains query predictions with SIAMESE~\cite{zhuEASEUnsupervisedDiscriminant2022}.
        We replace EASE with our proposed embedding approaches \method and \methodS, as well as other baseline embeddings, and evaluate performance on all datasets in the 1 and 5-shot setting.
        As shown in Table~\ref{tab:siameseAcc}, \method and \methodS outperform all baseline approaches in both settings across all datasets, illustrating \method's and \methodS' ability to provide useful FSL embeddings, and updating the state-of-the-art in transductive FSL.

    \customparagraph{Aggregated FSL performance}
        To further evaluate the general applicability of \method and \methodS as embedding approaches, we perform extensive experiments for all classifiers and all baseline embeddings on all datasets.
        Tables~\ref{tab:agg1shot} and~\ref{tab:agg5shot} provide the results averaged over classifiers\footnote{The detailed results for all classifiers are provided in the supplementary.}.
        To clearly present the results, we aggregate the accuracy and a ranking \emph{score} for each embedding method across all classifiers.
        The ranking score is calculated by performing a paired Student's t-test between all pairwise embedding methods for each classifier.
        We then average the ranking scores across all classifiers.
        A high ranking score then indicates that a method often significantly outperforms the competing embedding methods.
        We set the significance level to 5\%.
        \method and \methodS consistently outperform previous embedding approaches -- sometimes by a large margin.
        Overall, we further observe that \methodS outperforms \method in most settings and is particular beneficial in the 1-shot setting, which is more challenging, given that fewer samples are likely to generate noisy embeddings.

        \begin{table}
            \setlength{\tabcolsep}{0.9mm}
            \begin{subtable}{\columnwidth}
                \centering
                {\tableFontSize \begin{tabular}{lrcccccc}
\toprule
& & \multicolumn{2}{c}{mini} & \multicolumn{2}{c}{tiered} & \multicolumn{2}{c}{CUB} \\
& Embedding & Acc \HIGHERBETTER & Score \HIGHERBETTER & Acc \HIGHERBETTER & Score \HIGHERBETTER & Acc \HIGHERBETTER & Score \HIGHERBETTER \\ \cmidrule(lr){2-2} \cmidrule(lr){3-4} \cmidrule(lr){5-6} \cmidrule(lr){7-8}

\multirow{9}{*}{\rotatebox{90}{ResNet18}}
& \TableMethodN & 55.74 & 0.17 & 62.61 & 0.0 & 63.78 & 0.17 \\
& \TableMethodL & 68.22 & 2.33 & 75.94 & 2.17 & 78.09 & 2.33 \\
& \TableMethodC & 69.56 & 2.83 & 76.97 & 3.0 & 78.26 & 2.83 \\
& \TableMethodZ & 60.0 & 2.33 & 66.21 & 2.5 & 67.43 & 2.67 \\
& \TableMethodR & 60.76 & 4.0 & 67.07 & 3.67 & 69.6 & 4.17 \\
& \TableMethodE & 69.63 & 3.67 & 77.05 & 4.0 & 78.84 & 3.67 \\
& \TableMethodT & 69.97 & 4.0 & 77.18 & 3.33 & 78.83 & 4.0 \\
& \TableMethodOurs & \SECONDBEST{72.58} & \SECONDBEST{6.83} & \SECONDBEST{79.77} & \SECONDBEST{6.83} & \SECONDBEST{81.91} & \SECONDBEST{6.83} \\
& \TableMethodOursS & \BEST{73.64} & \BEST{7.67} & \BEST{80.6} & \BEST{7.67} & \BEST{83.1} & \BEST{7.67} \\\midrule
\multirow{9}{*}{\rotatebox{90}{WideRes28-10}}
& \TableMethodN & 63.59 & 1.0 & 71.29 & 0.83 & 79.23 & 1.17 \\
& \TableMethodL & 74.3 & 3.0 & 76.19 & 2.67 & 88.61 & 3.5 \\
& \TableMethodC & 71.32 & 1.33 & 75.17 & 2.0 & 88.52 & 3.33 \\
& \TableMethodZ & 64.27 & 2.5 & 65.64 & 2.5 & 76.0 & 1.5 \\
& \TableMethodR & 65.51 & 3.0 & 71.83 & 3.17 & 83.1 & 3.5 \\
& \TableMethodE & 74.95 & 4.33 & 76.59 & 3.67 & 88.51 & 3.5 \\
& \TableMethodT & 75.64 & 4.83 & 76.51 & 4.0 & 88.22 & 2.5 \\
& \TableMethodOurs & \SECONDBEST{78.22} & \SECONDBEST{7.0} & \SECONDBEST{79.76} & \SECONDBEST{7.0} & \SECONDBEST{90.25} & \SECONDBEST{5.67} \\
& \TableMethodOursS & \BEST{79.24} & \BEST{7.67} & \BEST{80.46} & \BEST{7.67} & \BEST{90.82} & \BEST{7.67} \\
\bottomrule
\end{tabular}

}
                \caption{ 1-shot}
                \label{tab:agg1shot}
            \end{subtable}
            \begin{subtable}{\columnwidth}
                \centering
                {\tableFontSize \begin{tabular}{lrcccccc}
\toprule
& & \multicolumn{2}{c}{mini} & \multicolumn{2}{c}{tiered} & \multicolumn{2}{c}{CUB} \\
& & Acc \HIGHERBETTER & Score \HIGHERBETTER & Acc \HIGHERBETTER & Score \HIGHERBETTER & Acc \HIGHERBETTER & Score \HIGHERBETTER \\ \cmidrule(lr){3-4} \cmidrule(lr){5-6} \cmidrule(lr){7-8}

\multirow{9}{*}{\rotatebox{90}{ResNet18}}
& \TableMethodN & 69.83 & 0.83 & 74.38 & 0.67 & 76.01 & 1.17 \\
& \TableMethodL & 81.58 & 2.33 & 86.05 & 1.83 & 88.43 & 2.83 \\
& \TableMethodC & 81.95 & 2.67 & 86.43 & 3.0 & 88.49 & 2.5 \\
& \TableMethodZ & 71.49 & 4.0 & 75.32 & 3.83 & 76.92 & 3.5 \\
& \TableMethodR & 70.25 & 2.5 & 74.52 & 1.83 & 76.43 & 2.5 \\
& \TableMethodE & 81.84 & 3.5 & 86.4 & 3.17 & 88.57 & 3.5 \\
& \TableMethodT & 82.1 & 4.0 & 86.54 & 3.83 & 88.79 & 4.33 \\
& \TableMethodOurs & \SECONDBEST{82.58} & \SECONDBEST{5.5} & \SECONDBEST{86.9} & \SECONDBEST{4.5} & \BEST{89.13} & \BEST{6.0} \\
& \TableMethodOursS & \BEST{82.61} & \BEST{6.5} & \BEST{87.13} & \BEST{6.67} & \SECONDBEST{88.93} & \SECONDBEST{5.33} \\
\midrule
\multirow{9}{*}{\rotatebox{90}{WideRes28-10}}
& \TableMethodN & 78.77 & 1.5 & 84.1 & 1.67 & 89.49 & 1.67 \\
& \TableMethodL & 85.65 & 4.0 & 86.29 & 3.83 & 93.47 & 3.67 \\
& \TableMethodC & 83.14 & 1.33 & 85.47 & 1.5 & 93.49 & 4.0 \\
& \TableMethodZ & 74.61 & 4.33 & 75.34 & 5.0 & 81.02 & 3.17 \\
& \TableMethodR & 73.86 & 1.83 & 81.51 & 1.67 & 87.2 & 2.0 \\
& \TableMethodE & 85.51 & 3.5 & 86.29 & 3.33 & 93.34 & 3.5 \\
& \TableMethodT & \SECONDBEST{86.03} & \BEST{6.0} & 86.37 & 4.0 & 93.3 & 3.0 \\
& \TableMethodOurs & \BEST{86.44} & \SECONDBEST{5.67} & \BEST{87.07} & \SECONDBEST{5.5} & \SECONDBEST{93.65} & \SECONDBEST{4.17} \\
& \TableMethodOursS & 85.95 & 5.5 & \SECONDBEST{87.05} & \BEST{5.83} & \BEST{93.76} & \BEST{5.0} \\
\bottomrule
\end{tabular}
}
                \caption{5-shot}
                \label{tab:agg5shot}
            \end{subtable}
            \caption{Aggregated FSL performance for all embedding approaches on the mini-ImageNet, tiered-ImageNet, and CUB-200 datasets. Results are averaged over FSL classifiers. Best and second best performance are denoted in  {\textbf{bold}} and \underline{underlined}, respectively.}
            \label{tab:agg}
            \vspace{-0.3cm}
        \end{table}

\customparagraph{Hubness metrics}
        To further validate \method's and \methodS' ability to reduce hubness, we follow the same procedure of aggregating results for the hubness metrics and average over classifiers.
        Compared to the current state-of-the-art embedding approaches, Table~\ref{tab:hubnessMetrics} illustrates that \method and \methodS consistently result in embeddings with lower hubness.
        
        \begin{table}
            \centering
            \setlength{\tabcolsep}{1.2mm}
            \begin{subtable}{\columnwidth}
                \centering
                {\tableFontSize \begin{tabular}{rrcccccc}
\toprule
 &  & \multicolumn{2}{c}{mini} & \multicolumn{2}{c}{tiered} & \multicolumn{2}{c}{CUB} \\
                                                  & & Sk \LOWERBETTER  & HO \LOWERBETTER\ & Sk \LOWERBETTER  & HO \LOWERBETTER\ & Sk \LOWERBETTER  & HO \LOWERBETTER\ \\ \cmidrule(lr){3-4} \cmidrule(lr){5-6} \cmidrule(lr){7-8}
\multirow[c]{9}{*}{\rotatebox{90}{ResNet18}}
                                                 & \TableMethodN & 1.349              & 0.407                   & 1.211              & 0.408                   & 0.887              & 0.341 \\
                                                 & \TableMethodL & 0.937              & 0.301                   & 0.812              & 0.265                   & 0.691              & 0.236 \\
                                                 & \TableMethodC & 0.667              & 0.233                   & 0.679              & 0.249                   & 0.549              & 0.201 \\
                                                 & \TableMethodZ & 0.68               & 0.231                   & 0.698              & 0.264                   & 0.564              & 0.216 \\
                                                 & \TableMethodR & 3.655              & 0.548                   & 3.604              & 0.549                   & 3.565              & 0.513 \\
                                                 & \TableMethodE & 0.521              & 0.16                    & 0.479              & 0.158                   & 0.466              & \SECONDBEST{0.153} \\
                                                 & \TableMethodT & 0.651              & 0.228                   & 0.65               & 0.25                    & 0.532              & 0.204 \\
                                                 & \TableMethodOurs & \SECONDBEST{0.315} & \BEST{0.095}            & \SECONDBEST{0.303} & \BEST{0.102}            & \SECONDBEST{0.32}  & \BEST{0.112} \\
                                                 & \TableMethodOursS & \BEST{0.276}       & \SECONDBEST{0.13}       & \BEST{0.283}       & \SECONDBEST{0.127}      & \BEST{0.296}       & 0.162 \\
\midrule
\multirow[c]{9}{*}{\rotatebox{90}{WideRes28-10}}
                                                 & \TableMethodN & 1.6                & 0.459                   & 1.81               & 0.494                   & 1.073              & 0.369 \\
                                                 & \TableMethodL & 0.781              & 0.296                   & 0.737              & 0.275                   & 0.475              & 0.228 \\
                                                 & \TableMethodC & 0.981              & 0.288                   & 0.817              & 0.307                   & 0.52               & 0.267 \\
                                                 & \TableMethodZ & 0.73               & 0.287                   & 0.769              & 0.302                   & 0.517              & 0.263 \\
                                                 & \TableMethodR & 3.56               & 0.704                   & 3.55               & 0.777                   & 3.026              & 0.47 \\
                                                 & \TableMethodE & 0.47               & 0.177                   & 0.477              & 0.175                   & 0.437              & 0.213 \\
                                                 & \TableMethodT & 0.589              & 0.236                   & 0.685              & 0.264                   & 0.477              & 0.231 \\
                                                 & \TableMethodOurs & \SECONDBEST{0.29}  & \BEST{0.111}            & \SECONDBEST{0.301} & \BEST{0.111}            & \SECONDBEST{0.188} & \BEST{0.108} \\
                                                 & \TableMethodOursS & \BEST{0.258}       & \SECONDBEST{0.148}      & \BEST{0.274}       & \SECONDBEST{0.135}      & \BEST{0.162}       & \SECONDBEST{0.13} \\

\bottomrule
\end{tabular}
}
                \caption{ 1-shot}
                \label{tab:hubness1shot}
            \end{subtable}
            \begin{subtable}{\columnwidth}
                \centering
                {\tableFontSize \begin{tabular}{rrcccccc}
\toprule
 &  & \multicolumn{2}{c}{mini} & \multicolumn{2}{c}{tiered} & \multicolumn{2}{c}{CUB} \\
                                                  & & Sk \LOWERBETTER  & HO \LOWERBETTER\ & Sk \LOWERBETTER  & HO \LOWERBETTER\ & Sk \LOWERBETTER  & HO \LOWERBETTER\ \\ \cmidrule(lr){3-4} \cmidrule(lr){5-6} \cmidrule(lr){7-8}
\multirow[c]{9}{*}{\rotatebox{90}{ResNet18}}
                                                 & \TableMethodN & 1.436              & 0.422                   & 1.339              & 0.432                   & 0.987              & 0.364 \\
                                                 & \TableMethodL & 1.04               & 0.318                   & 0.914              & 0.287                   & 0.812              & 0.263 \\
                                                 & \TableMethodC & 0.786              & 0.264                   & 0.821              & 0.28                    & 0.698              & 0.236 \\
                                                 & \TableMethodZ & 0.806              & 0.264                   & 0.839              & 0.296                   & 0.716              & 0.25 \\
                                                 & \TableMethodR & 1.631              & 0.863                   & 1.721              & 0.872                   & 1.432              & 0.869 \\
                                                 & \TableMethodE & 0.624              & 0.186                   & 0.598              & 0.183                   & 0.607              & 0.186 \\
                                                 & \TableMethodT & 0.78               & 0.259                   & 0.796              & 0.283                   & 0.687              & 0.235 \\
                                                 & \TableMethodOurs & \SECONDBEST{0.286} & \SECONDBEST{0.096}      & \SECONDBEST{0.289} & \SECONDBEST{0.104}      & \BEST{0.329}       & \SECONDBEST{0.12} \\
                                                 & \TableMethodOursS & \BEST{0.25}        & \BEST{0.074}            & \BEST{0.213}       & \BEST{0.078}            & \SECONDBEST{0.433} & \BEST{0.097} \\
\midrule
\multirow[c]{9}{*}{\rotatebox{90}{WideRes28-10}}
                                                 & \TableMethodN  & 1.709              & 0.473                   & 1.937              & 0.51                    & 1.16               & 0.395 \\
                                                 & \TableMethodL  & 0.887              & 0.322                   & 0.86               & 0.305                   & 0.632              & 0.266 \\
                                                 & \TableMethodC  & 1.12               & 0.318                   & 0.956              & 0.337                   & 0.701              & 0.31 \\
                                                 & \TableMethodZ  & 0.858              & 0.32                    & 0.912              & 0.335                   & 0.699              & 0.305 \\
                                                 & \TableMethodR  & 1.597              & 0.819                   & 1.617              & 0.846                   & 1.299              & 0.549 \\
                                                 & \TableMethodE  & 0.579              & 0.199                   & 0.585              & 0.193                   & 0.572              & 0.241 \\
                                                 & \TableMethodT  & 0.717              & 0.27                    & 0.815              & 0.294                   & 0.634              & 0.264 \\
                                                 & \TableMethodOurs  & \BEST{0.294}       & \SECONDBEST{0.115}      & \BEST{0.298}       & \BEST{0.115}            & \BEST{0.195}       & \BEST{0.1} \\
                                                 & \TableMethodOursS  & \SECONDBEST{0.494} & \BEST{0.103}            & \SECONDBEST{0.407} & \SECONDBEST{0.12}       & \SECONDBEST{0.421} & \SECONDBEST{0.127} \\

\bottomrule
\end{tabular}
}
                \caption{5-shot}
                \label{tab:hubness5shot}
            \end{subtable}
            \caption{Aggregated hubness metrics for all embedding approaches on the Mini-ImageNet, Tiered-ImageNet and CUB-200 dataset. Results are averaged over FSL classifiers. Best and second best performance are denoted in  {\textbf{bold}} and \underline{underlined}, respectively.}
            \label{tab:hubnessMetrics}
            \vspace{-0.3cm}
        \end{table}

    \customparagraph{Visualization of similarity matrices}
        As discussed in Sec.~\ref{sec:method}, completely eliminating hubness by distributing points uniformly on the hypersphere is not sufficient to obtain good FSL performance.
        Instead, representations need to also capture the inherent class structure of the data.
        To further evaluate the embedding approaches, we therefore compute the pairwise inner products for the embeddings of a random 5-shot episode on tiered-ImageNet with ResNet-18 features in Figure~\ref{fig:innerProductMatrices}.
        It can be observed that the block structure is considerably more distinct for \method and \methodS, with \methodS slightly improving upon \method.
        These results indicate that (i) samples are more uniform, indicating the reduced hubness;
        and (ii) classes are better separated, due to the local similarity preservation.

        \begin{figure*}
            \centering
            \includegraphics[width=\textwidth]{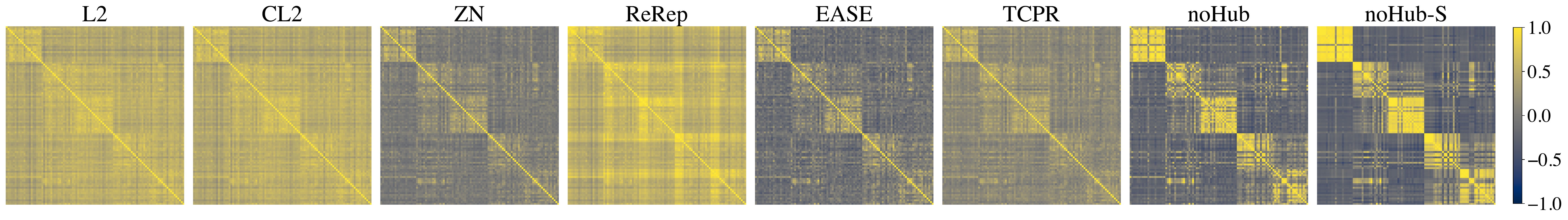}
            \caption{Inner product matrices between features for a random episode for all embedding approaches.}
            \label{fig:innerProductMatrices}
        \end{figure*}

    \customparagraph{Tradeoff between uniformity and similarity preservation}
        We analyze the effect of \(\alpha \) on the tradeoff between LSP and Uniformity in the loss function in Eq.~\eqref{eq:Lfinal}, on tiered-ImageNet with ResNet-18 features in the 5-shot setting and with the SIAMESE~\cite{zhuEASEUnsupervisedDiscriminant2022} classifier.
        The results are visualized in Figure~\ref{fig:alphas}.
        We notice a sharp increase in performance when we have a high emphasis on uniformity.
        This demonstrates the impact of hubness on accuracy in FSL performance.
        As we keep increasing the emphasis on LSP, however, after a certain point we notice a sharp drop off in performance.
        This is due to the fact that the classifier does not take into account the uniformity constraint on the features, resulting in a large number of misclassifications.
        In general, we observe that \methodS is slightly more robust compared to \method.
    
        \begin{figure}
            \centering
            \includegraphics[width=0.7\columnwidth]{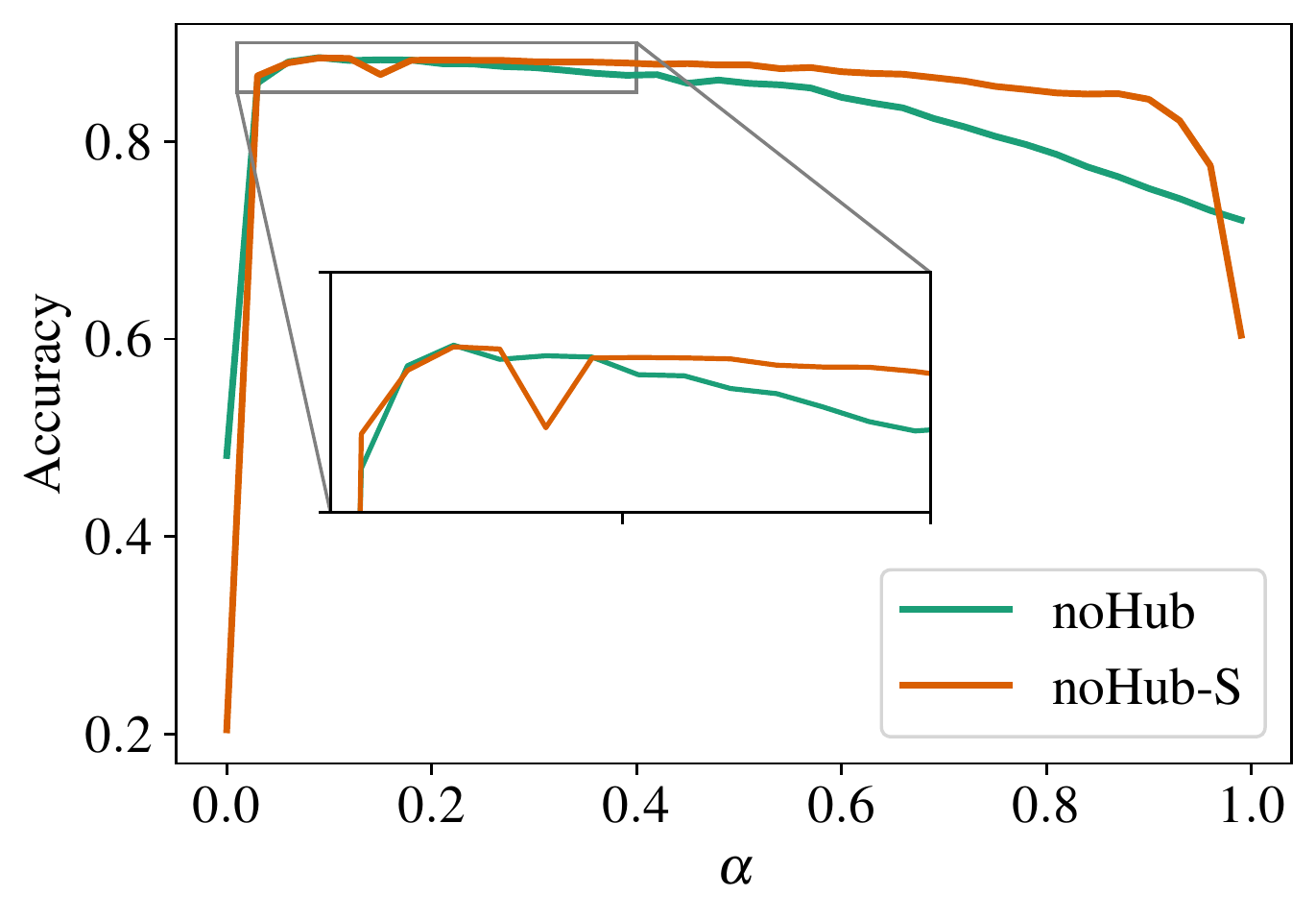}
            \vspace{-.3cm}
            \caption{Accuracies for different values of the weighting parameter, \( \alpha \), which quantifies the tradeoff between $\Lalign$ and $\Lunif$.}
            \label{fig:alphas}
        \end{figure}

        \begin{figure}
            \centering
            \includegraphics[width=0.7\columnwidth]{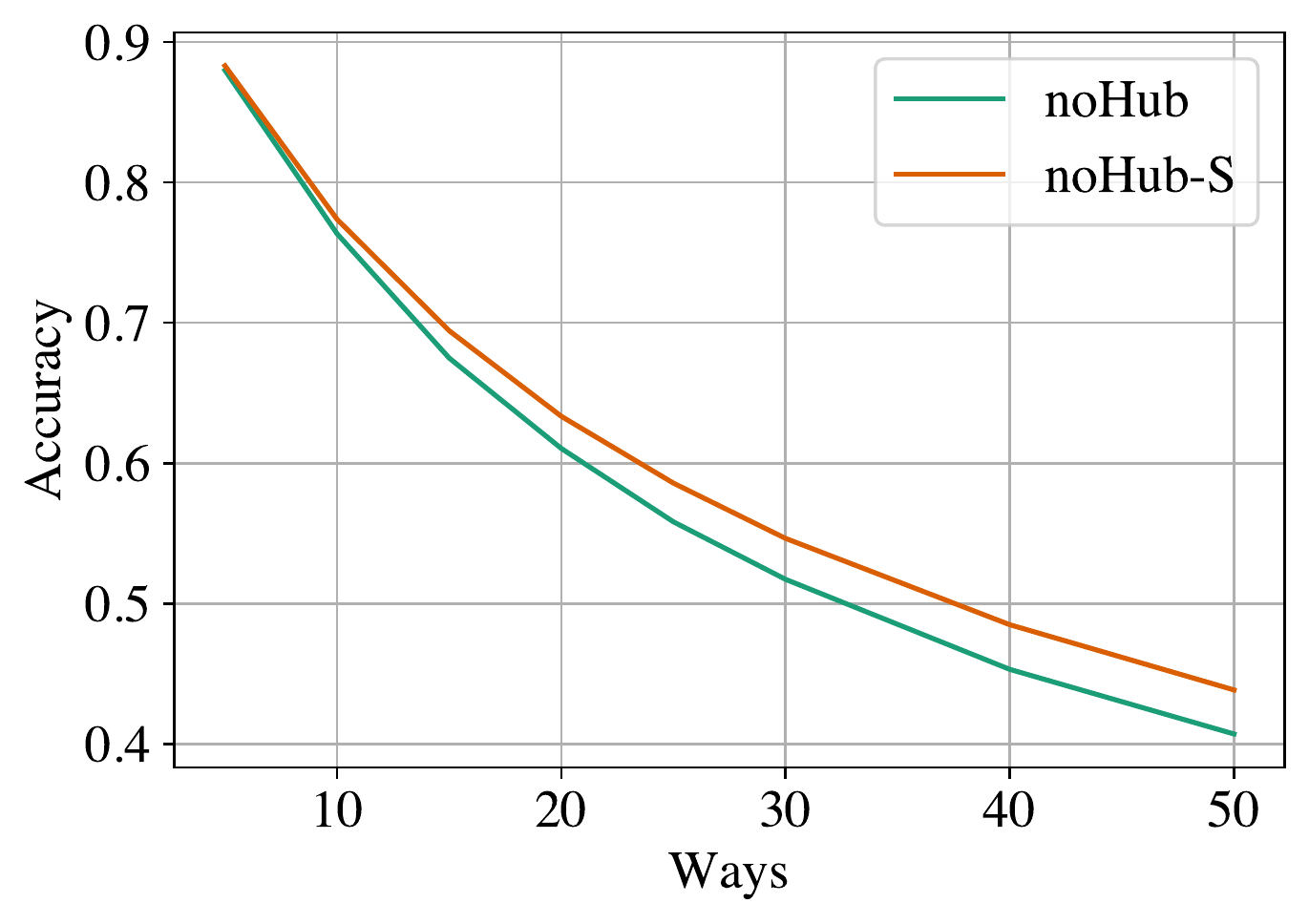}
            \vspace{-.3cm}
            \caption{Accuracies for an increasing number of classes (ways) for \method and \methodS.}
            \label{fig:increasingWays}
        \end{figure}
    
    \customparagraph{Increasing number of classes}
    We analyze the behavior of \method and \methodS for an increasing number of classes (ways) on the tiered-ImageNet dataset with SIAMESE~\cite{zhuEASEUnsupervisedDiscriminant2022} as classifier.
        While classification accuracy generally decreases with an increasing number of classes, which is expected, we observe from Figure~\ref{fig:increasingWays} that \methodS has a slower decay and is able to leverage the label guidance to obtain better performance for a larger number of classes.
    
    \customparagraph{Effect of label information in $\Lalign$ and $\Lunif$}
        To validate the effectiveness of using label guidance in \methodS, we study the result of including label information in \(\Lalign\) and \(\Lunif\) (Eqs.~\eqref{eq:modpij1}--\eqref{eq:moduni}).
        We note that the default setting of \method is that none of the two losses include label information.
        Ablation experiments are performed on tiered-ImageNet with the ResNet-18 feature extractor and the SimpleShot and SIAMESE classifier~\cite{zhuEASEUnsupervisedDiscriminant2022}.
         In Table~\ref{tab:ablation}, we generally see improvements of \methodS when \textit{both} the loss terms are label-informed, indicating the usefulness of label guidance.

         We further observe that incorporating label information in \( \Lunif \) tends to have a larger contribution than doing the same for \( \Lalign \).
        This aligns with our observations in Figure~\ref{fig:alphas}, where a small \( \alpha \) yielded the best performance.

    \begin{table}
        \flushleft
        {\setlength{\tabcolsep}{0.5mm} \tableFontSize 
\begin{tabular}{rcccccc} \toprule
    & \multicolumn{2}{c}{Label-informed} & \multicolumn{2}{c}{SimpleShot~\cite{wangSimpleShotRevisitingNearestNeighbor2019}} & \multicolumn{2}{c}{SIAMESE~\cite{zhuEASEUnsupervisedDiscriminant2022}} \\
    &  \( \Lalign \) & \( \Lunif \) & 1-shot \HIGHERBETTER & 5-shot \HIGHERBETTER & 1-shot \HIGHERBETTER & 5-shot \HIGHERBETTER \\ \cmidrule(lr){2-3} \cmidrule(lr){4-5} \cmidrule(lr){6-7}
    \method  &  \FALSE      &    \FALSE &    \MTCWITHCONF{76.72}{0.23} &  \BEST{\MTCWITHCONF{86.31}{0.16}} & \MTCWITHCONF{82.94}{0.27}       & \MTCWITHCONF{87.87}{0.17} \\
    \methodS &     \TRUE  &      \FALSE &    \MTCWITHCONF{78.25}{0.24} & \MTCWITHCONF{85.46}{0.16} & \MTCWITHCONF{82.56}{0.28}      &  \MTCWITHCONF{88.07}{0.17} \\
    \methodS &    \FALSE  &       \TRUE &    \MTCWITHCONF{78.33}{0.23} &  \MTCWITHCONF{86.15}{0.15} & \MTCWITHCONF{82.81}{0.27}      &  \BEST{\MTCWITHCONF{88.43}{0.16}} \\
    \methodS &      \TRUE  &        \TRUE &  \BEST{\MTCWITHCONF{78.35}{0.23}} &  \MTCWITHCONF{86.22}{0.15} & \BEST{\MTCWITHCONF{83.09}{0.27}} & \BEST{\MTCWITHCONF{88.43}{0.16}} \\
\bottomrule
\end{tabular}

}
        \caption{Ablation study with the label-informed losses in \methodS. Check marks (\TRUE) indicate that the loss uses information from the support labels.}
        \label{tab:ablation}
    \end{table}

    \section{Conclusion}
        In this paper we have addressed the hubness problem in FSL.
We have shown that hubness is eliminated by embedding representations uniformly on the hypersphere.
The hyperspherical uniform distribution has zero mean and zero density gradient at all points along all directions tangent to the hypersphere -- both of which are identified as causes of hubness in previous work~\cite{radovanovicHubsSpacePopular2010,haraFlatteningDensityGradient2016}.
Based on our theoretical findings about hubness and hyperspheres, we proposed two new methods to embed representations on the hypersphere for FSL.
The proposed \method and \methodS leverage a decomposition of the KL divergence between similarity distributions, and optimize a tradeoff between LSP and uniformity on the hypersphere -- thus reducing hubness while maintaining the class structure in the representation space.
We have provided theoretical analyses and interpretations of the LSP and uniformity losses, proving that they optimize LSP and uniformity, respectively.
We comprehensively evaluate the proposed methods on several datasets, features extractors, and classifiers, and compare to a number of recent state-of-the-art baselines. 
Our results illustrate the effectiveness of our proposed methods and show that we achieve state-of-the-art performance in transductive FSL.

    \section*{Acknowledgements}
        This work was financially supported by the Research Council of Norway (RCN), through its Centre for Research-based Innovation funding scheme (Visual Intelligence, grant no.\ 309439), and Consortium Partners.
        It was further funded by RCN FRIPRO grant no.\ 315029, RCN IKTPLUSS grant no.\ 303514, and the UiT Thematic Initiative ``Data-Driven Health Technology''.

    \clearpage
    \appendix
    \noindent
    {\huge\bfseries Supplementary material}

    \section{Introduction}
        Here we provide proofs for our theoretical results on the hyperspherical uniform and hubness;
the decomposition of the KL divergence;
and the minima of our methods' loss functions.
We also give additional details on the implementation and hyperparameters for \method and \methodS -- and include the complete tables of 1-shot and 5-shot results for all classifiers, datasets and feature extractors.
Finally, we briefly reflect on potential negative societal impacts of our work.

    \section{Hyperspherical Uniform Eliminates Hubness}
        
\subsection*{Proof of Proposition~\ref{prop:zeroMean}}

    \def\SPHEREPOS{\SPHERE_d^{i, +}}
    \def\SPHERENEG{\SPHERE_d^{i, -}}
    \def\SPHEREZER{\SPHERE_d^{i, 0}}

    \begin{lemma}[Trisection of hypersphere]
        \label{lemma:sphereSymm}
        The trisection of the hypersphere along coordinate \( i \) is given by the three-tuple of disjoint sets \( (\SPHEREPOS, \SPHERENEG, \SPHEREZER) \) where
        \begin{align}
            & \SPHEREPOS = \{ \vec x = [x^1, \dots, x^d]\T \in \SPHERE_d \mid x^i > 0 \} \\
            & \SPHERENEG = \{ \vec x = [x^1, \dots, x^d]\T \in \SPHERE_d \mid x^i < 0 \} \\
            & \SPHEREZER = \{ \vec x = [x^1, \dots, x^d]\T \in \SPHERE_d \mid x^i = 0 \}
        \end{align}
        and
        \begin{align}
            \SPHEREPOS \cup \SPHERENEG \cup \SPHEREZER = \SPHERE_d
        \end{align}

        Then we have
        \begin{align}
            \SPHEREPOS = - \SPHERENEG = \{-\vec x \mid \vec x \in \SPHERENEG \}
        \end{align}
    \end{lemma}
    \begin{proof}
        Let \( \vec x \in \SPHEREPOS \), then
        \begin{align}
            || (-x) || = || x || = 1,
        \end{align}
        and
        \begin{align}
            -(x^i) < 0.
        \end{align}
        Hence \( \vec x \in -\SPHERENEG \), and \( \SPHEREPOS \subseteq - \SPHERENEG \).

        Similarly, let \( - \vec x \in -\SPHERENEG \), then
        \begin{align}
            ||-(-x)|| = ||x|| = 1,
        \end{align}
        and
        \begin{align}
            -(-x^i) = x^i > 0.
        \end{align}
        Hence \( \vec x \in \SPHEREPOS \), and \( -\SPHERENEG \subseteq \SPHEREPOS \).

        It then follows that \( \SPHEREPOS = - \SPHERENEG \).
    \end{proof}

    \begin{proof}[Proof (Proposition~\ref{prop:zeroMean})]
        The expectation \( \E (\vec X) \) is given by
        \begin{align}
            \E (\vec X) = \int_{\real^d} \vec x \SPHEREUNIFORM(\vec x) \dd \vec x
        \end{align}
        Since \( \SPHEREUNIFORM \) is non-zero only on the hypersphere \( \SPHERE_d \), the integral can be rewritten as a surface integral over \( \SPHERE_d \)
        \begin{align}
            \E(\vec X) = \int_{\SPHERE_d} \vec x A_d^{-1} \dd S.
        \end{align}
        Decomposing the integral over the trisection of \( \SPHERE_d \) along coordinate \( i \) gives
        \begin{align}
            & \int_{\SPHERE_d} \vec x A_d^{-1} \dd S =\\
            & A_d^{-1} \lrp{
                \int_{\SPHEREPOS} \vec x \dd S
                + \int_{\SPHERENEG} \vec x \dd S
                + \int_{\SPHEREZER} \vec x \dd S
            }.
        \end{align}
        By Lemma~\ref{lemma:sphereSymm} we have
        \begin{align}
            \SPHEREPOS = - \SPHERENEG \Rightarrow \int_{\SPHEREPOS} \vec x \dd S = - \int_{\SPHERENEG} \vec x \dd S.
        \end{align}
        Furthermore, since the set \( \SPHEREZER \) has zero width along coordinate \( i \), \( \int_{\SPHEREZER} \vec x \dd S = 0 \).
        Hence
        \begin{align}
            &\E(\vec X) = \\
            & A_d^{-1} \lrp{
                \int_{\SPHEREPOS} \vec x \dd S
                - \int_{\SPHEREPOS} \vec x \dd S
                + \int_{\SPHEREZER} \vec x \dd S
            } = 0
        \end{align}
    \end{proof}

\subsection*{Proof of Proposition~\ref{prop:zeroGradient}}
    \begin{proof}
        \( \SPHEREUNIFORM(\vec x) \) can be written in polar coordinates as
        \begin{align}
            \SPHEREUNIFORM(\vec x(r, \vec \theta)) = \SPHEREUNIFORM^{\text{Polar}}(r, \vec\theta) = A_d^{-1}\delta(r - 1)
        \end{align}
        The gradient of \( \SPHEREUNIFORM^{\text{Polar}}(r, \vec\theta) \) is then
        \begin{align}
            \nabla_{(r, \vec \theta)} \SPHEREUNIFORM^{\text{Polar}}(r, \vec\theta) =
            \begin{bmatrix}
                \frac{\partial}{\partial r} \SPHEREUNIFORM^{\text{Polar}}(r, \vec\theta) \\[0.2cm]
                0 \\
                \vdots \\
                0
            \end{bmatrix}
        \end{align}

    For an arbitrary point \( \vec p \in \SPHERE_d \), an arbitrary unit vector (direction), \( \vec \theta^{*} \), in the tangent plane \( \Pi_{\vec p} \) is given by
    \begin{align}
        \vec \theta^{*} =
        \begin{bmatrix}
            0\\
            \theta^{*}_1 \\
            \vdots \\
            \theta^{*}_{d-1}
        \end{bmatrix}
    \end{align}
    The directional derivative of \( \SPHEREUNIFORM(\vec x) \) along \( \vec \theta^{*} \) is then
    \begin{align}
        \nabla_{ \vec \theta^{*}} \SPHEREUNIFORM(\vec x) =
        \begin{bmatrix}
            \frac{\partial}{\partial r} \SPHEREUNIFORM^{\text{Polar}}(r, \vec\theta) \\[0.2cm]
            0 \\
            \vdots \\
            0
        \end{bmatrix}\T
        \cdot
        \begin{bmatrix}
            0\\
            \theta^{*}_1 \\
            \vdots \\
            \theta^{*}_{d-1}
        \end{bmatrix}
        = 0
    \end{align}
\end{proof}

    \section{Method}
        \customparagraph{Computing \( \kappa_i \)}
    Following~\cite{Maaten}, we compute \( \kappa_i \) using a binary search such that
    \begin{align}
        | \log_2(P) - H(P_i) | \le 0.1 \cdot \log_2(P)
    \end{align}
    where \( P \) is a hyperparameter, and \( H(P_i) \) is the Shannon entropy of similarities for representation \( i \)
    \begin{align}
        H(P_i) = \sums{j=1}{n} p_{i|j} \log_2(p_{i|j}).
    \end{align}

\customparagraph{Decomposition of \( KL(P||Q) \)}
    Recall that
    \begin{align}
        p_{ij} = \frac{p_{i|j} + p_{j | i}}{2}, \quad p_{i|j} = \frac{\exp(\kappa_i \vec x_i\T \vec x_j)}{\sums{l, m}{}\exp(\kappa_i \vec x_l\T \vec x_m)}
    \end{align}
    and
    \begin{align}
        q_{ij} = \frac{\exp(\kappa \vec z_i\T \vec z_j)}{\sums{l, m}{} \exp(\kappa \vec z_l\T \vec z_m)}.
    \end{align}

    Since \( p_{ij} \) is constant \wrt \( q_{ij} \), we have
    \begin{align}
        &\ARGMIN KL(P||Q) = \ARGMIN \sums{i, j}{} p_{ij} \log \frac{p_{ij}}{q_{ij}} \\
        &= \ARGMIN \underbrace{\sums{i, j}{} p_{ij} \log p_{ij}}_{\text{constant}} - \sums{i, j}{} p_{ij} \log q_{ij} \\
        &= \ARGMIN \underbrace{- \sums{i, j}{} p_{ij} \log q_{ij}}_{\eqqcolon \LVMFT}
    \end{align}
    Minimizing \( KL(P||Q) \) over \( \TARGETS \) is therefore equivalent to minimizing \( \LVMFT \).

    Decomposing \( \LVMFT \) gives
    \begin{align}
        \LVMFT &= - \sums{i, j}{} p_{ij} \kappa \vec z_i\T \vec z_j +\\
        & \sums{i, j}{} \lrp{p_{ij} \log \sums{l, m}{} \exp(\kappa \vec z_l\T \vec z_m)} \\
        &= - \sums{i, j}{} p_{ij} \kappa \vec z_i\T \vec z_j \\
        & + \lrp{\log \sums{l, m}{} \exp(\kappa \vec z_l\T \vec z_m)} \cdot \underbrace{\lrp{\sums{i, j}{} p_{ij}}}_{= 1} \\
        &= \underbrace{- \sums{i, j}{} p_{ij} \kappa \vec z_i\T \vec z_j}_{\eqqcolon~\Lalign} + \underbrace{\log \sums{l, m}{} \exp(\kappa \vec z_l\T \vec z_m)}_{\eqqcolon~\Lunif}
    \end{align}
    Thus, we have shown that
    \begin{align}
        \ARGMIN KL(P||Q) = \ARGMIN \Lalign + \Lunif.
    \end{align}

    \section{Theoretical Results}
        \subsection*{Proof of Proposition~\ref{prop:conLE}}
    \begin{proof}
        We have
        \begin{align}
            \Lalign &= -\kappa \sums{i, j}{} p_{ij} \vec z_i\T \vec z_j \\
                    &= - 2 \sums{i, j}{} \frac{1}{2} \kappa p_{ij} \vec z_i\T \vec z_j + \sums{i, j}{} 2 \frac{1}{2} \kappa p_{ij} - \kappa  \\
                    & \nonumber(\sums{i, j}{} p_{ij} = 1) \\
                    &= - 2 \sums{i, j}{} \vec z_i\T \vec z_j W_{ij} + \sums{i, j}{}(\|\vec z_i\| + \|\vec z_j\|) W_{ij} - \kappa \\
                    &\nonumber (\|\vec z_i\| = \|\vec z_j\| = 1) \\
                    &= \sums{i, j}{}( \| \vec z_i \| - 2 \vec z_i\T \vec z_j + \| \vec z_j \|) W_{ij} - \kappa \\
                    &= \sums{i, j}{} {| \vec z_i - \vec z_j\|^2 W_{ij}} - \kappa.
        \end{align}
    \end{proof}

\subsection*{Proof of Proposition~\ref{prop:maxEntropy}}
    \begin{proof}
        Using a Gaussian kernel, the 2-order R\'enyi entropy can be estimated as~\cite[Eq.~(2.13)]{principeInformationTheoreticLearning2010}
        \begin{align}
            H_2(\vec z_1, \dots, \vec z_n) = -\log \left( \frac{1}{n^2} \sums{l, m}{} \exp(-\frac{1}{2}\kappa \|\vec z_l - \vec z_m \|^2) \right)
        \end{align}
        Thus, we have
        \begin{align}
            &\ARGMAX H_2(\vec z_1, \dots, \vec z_n) \\
            &\qquad= \ARGMAX -\log \Biggl( \frac{1}{n^2} \sums{l, m}{} \exp(-\frac{1}{2}\kappa \|\vec z_l - \vec z_m \|^2) \Biggr) \\
            &\qquad= \ARGMIN \log \Biggl( \sums{l, m}{} \exp(-\frac{1}{2}\kappa \|\vec z_l - \vec z_m \|^2) \Biggr) \\
            &\qquad= \ARGMIN \log \Biggl( \sums{l, m}{} \exp(-\frac{1}{2}\kappa (\|\vec z_l\|^2 \\
            &\qquad - 2 \vec z_l\T \vec z_m + \|\vec z_m \|^2) \Biggr) \\
            &\qquad= \ARGMIN \log \Biggl( \sums{l, m}{} \exp(-\kappa (1 - \vec z_l\T \vec z_m)) \Biggr) \\
            &\nonumber\qquad (\|\vec z_l\| = \|\vec z_m\| = 1) \\
            &\qquad= \ARGMIN \log \Biggl( \exp(-\kappa) \sums{l, m}{} \exp(\kappa \vec z_l\T \vec z_m) \Biggr) \\
            &\qquad= \ARGMIN \log \sums{l, m}{} \exp(\kappa \vec z_l\T \vec z_m) \\
            &\qquad= \ARGMIN \Lunif.
        \end{align}
    \end{proof}

\subsection*{Proof of Proposition~\ref{prop:minLunif}}
    \begin{proof}
        We have
        \begin{align}
            &\ARGMIN \Lunif \\
            &\qquad= \ARGMIN \log \sums{l, m}{} \exp(\kappa \vec z_l\T \vec z_m) \\
            &\qquad= \ARGMIN \sums{l, m}{} \exp(\kappa \vec z_l\T \vec z_m) \\
            & \nonumber\qquad\text{(monotonicity of logarithm)} \\
            &\qquad= \ARGMIN \sums{1 \le l < m \le n}{} \exp(\kappa \vec z_l\T \vec z_m) \\
            & \nonumber\qquad\text{(symmetry of inner product)} \\
            &\qquad= \ARGMIN \sums{1 \le l < m \le n}{} \underbrace{\exp(-\kappa ||\vec z_l - \vec z_m||_2^2)}_{\eqqcolon ~ G(\vec z_l, \vec z_m)} \\
            & \nonumber\qquad\text{(multiplication by positive constant)} \\
            &\qquad= \ARGMIN \sums{1 \le l < m \le n}{} G(\vec z_l, \vec z_m)
        \end{align}
        The result then follows directly from~\cite[Proposition 2]{wangUnderstandingContrastiveRepresentation2020}.
    \end{proof}

    \section{Experiments}
        \begin{figure}[t]
    \centering
    \begin{subfigure}{0.99\columnwidth}
        \centering
        \includegraphics[width=0.7\textwidth]{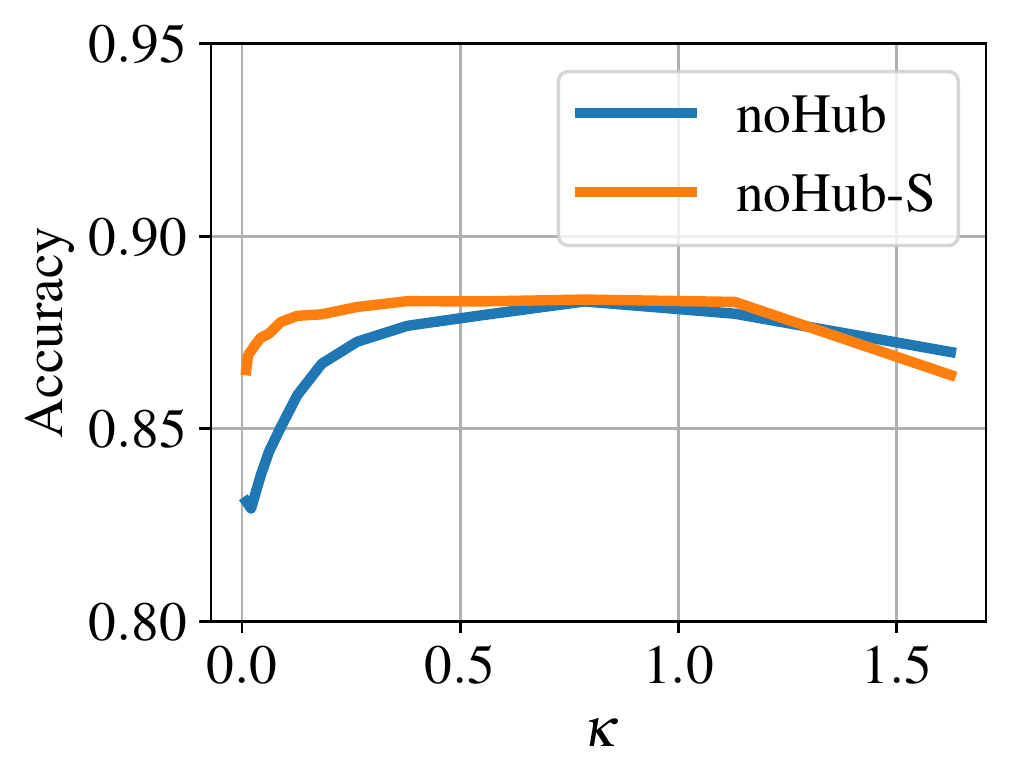}
    \end{subfigure}
    \begin{subfigure}{0.99\columnwidth}
        \centering
        \includegraphics[width=0.7\textwidth]{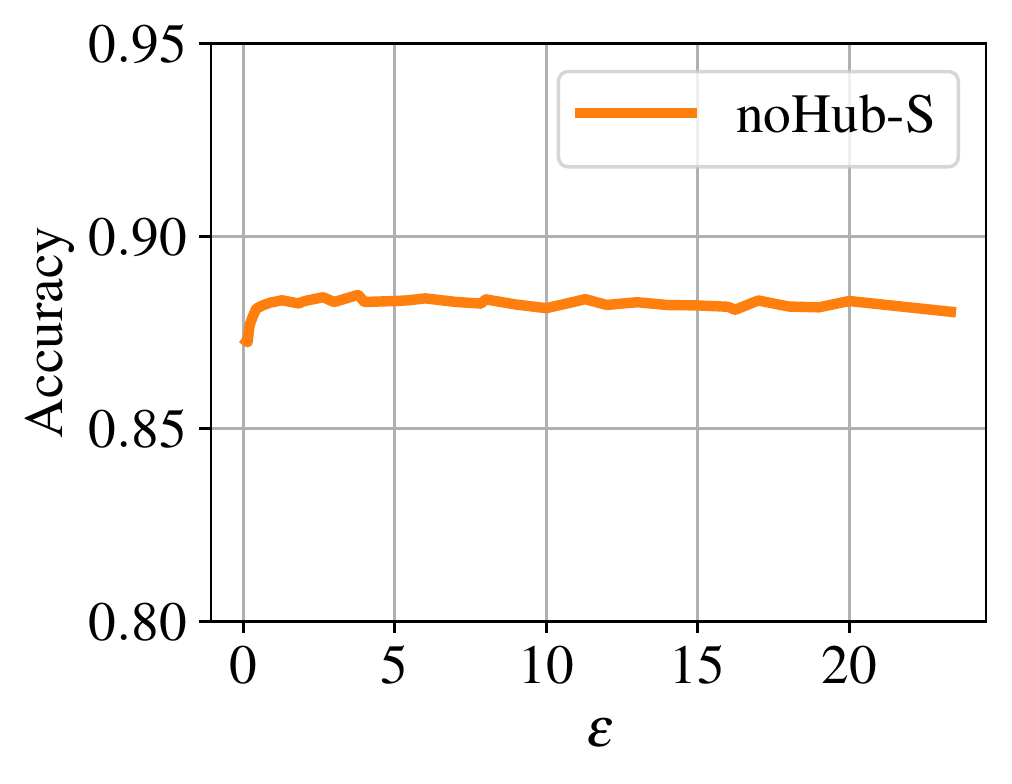}
    \end{subfigure}
    \caption{Accuracy for different values for \( \kappa \) and \( \epsilon \). Neither \method nor \methodS are particularly sensitive the the choice of these parameters.}
    \label{fig:kappaEpsilon}
\end{figure}

\subsection{Implementation details}
    This section covers the additional implementation details not provided in the main paper.
    These include the initialization of the embeddings in Algorithm 1, hyperparameters, additional transformations wherever required, the architectures used, and a note on accessing the code, datasets, and dataset splits.

    \customparagraph{Initialization and normalization}
        Instead of a random initialization of our embeddings $Z_0$, we follow a PCA based initialization, as in~\cite{Maaten}.
        The weights are computed using the cached features from the base classes, the support and query features are then transformed using these weights.
        This procedure is also fast as we do not need to compute the PCA weights on every episode.
        To ensure that the resulting features lie on the hypersphere after each gradient update in \method and \methodS, we re-normalize the embeddings using L2 normalization.
    
    \customparagraph{Hyperparameters}
        \method and \methodS have the following hyperparameters.
        \begin{itemize}
            \item \( P \) -- perplexity for computing the \( \kappa_i \).
            \item \( T \) -- number of iterations.
            \item \( \alpha \) -- tradeoff parameter in the loss (\( \Lfinal = \alpha\Lalign + (1-\alpha) \Lunif \)).
            \item \( \eta \) -- learning rate for the Adam optimizer.
            \item \( \kappa \) -- concentration parameter for the embeddings.
            \item \( \epsilon \) -- exaggeration of similarities between supports from different classes.
            \item \( d \) -- dimensionality of embeddings.
        \end{itemize}
        All hyperparameter values used in in \method and \methodS are given in Table~\ref{tab:hyperparameters}

        \begin{table*}
            \centering
            \small
            \begin{tabular}{ccccccccc}
\toprule
 &  & & \multicolumn{2}{c}{mini} & \multicolumn{2}{c}{tiered} & \multicolumn{2}{c}{CUB} \\
 Arch.\ & Param.\ & Method & 1-shot & 5-shot & 1-shot & 5-shot & 1-shot & 5-shot \\
\cmidrule(lr){1-3}\cmidrule(lr){4-5}\cmidrule(lr){6-7}\cmidrule(lr){8-9}
\multirow[c]{14}{*}{ResNet18} & \multirow[c]{2}{*}{\( P \) } & \method & 45 & 45 & 45 & 45 & 45 & 45 \\
 &  & \methodS & 45 & 45 & 40 & 45 & 45 & 45 \\ \cmidrule(lr){2-9}
 & \multirow[c]{2}{*}{\( T \)} & \method & 50 & 50 & 50 & 50 & 50 & 50 \\
 &  & \methodS & 150 & 150 & 150 & 150 & 150 & 150 \\ \cmidrule(lr){2-9}
 & \multirow[c]{2}{*}{\( \alpha \)} & \method & 0.2 & 0.2 & 0.2 & 0.2 & 0.2 & 0.2 \\
 &  & \methodS & 0.3 & 0.2 & 0.2 & 0.2 & 0.3 & 0.2 \\ \cmidrule(lr){2-9}
 & \multirow[c]{2}{*}{\( \eta \)} & \method & 0.1 & 0.1 & 0.1 & 0.1 & 0.1 & 0.1 \\
 &  & \methodS & 0.1 & 0.1 & 0.1 & 0.1 & 0.1 & 0.1 \\ \cmidrule(lr){2-9}
 & \multirow[c]{2}{*}{\( \kappa \)} & \method & 0.5 & 0.5 & 0.5 & 0.5 & 0.5 & 0.5 \\
 &  & \methodS & 0.5 & 0.5 & 0.5 & 0.5 & 0.5 & 0.5 \\ \cmidrule(lr){2-9}
 & \multirow[c]{2}{*}{\( \varepsilon \)} & \method & -- & -- & -- & -- & -- & -- \\
 &  & \methodS & 8 & 8 & 5 & 8 & 8 & 8 \\ \cmidrule(lr){2-9}
 & \multirow[c]{2}{*}{\( d \)} & \method & 400 & 400 & 400 & 400 & 400 & 400 \\
 &  & \methodS & 400 & 400 & 400 & 400 & 400 & 400 \\ \cmidrule(lr){1-9}
\multirow[c]{14}{*}{WideRes28-10} & \multirow[c]{2}{*}{\( P \) } & \method & 45 & 45 & 45 & 45 & 45 & 45 \\
 &  & \methodS & 45 & 45 & 40 & 35 & 45 & 30 \\ \cmidrule(lr){2-9}
 & \multirow[c]{2}{*}{\( T \)} & \method & 50 & 50 & 50 & 50 & 50 & 50 \\
 &  & \methodS & 150 & 150 & 150 & 150 & 150 & 150 \\ \cmidrule(lr){2-9}
 & \multirow[c]{2}{*}{\( \alpha \)} & \method & 0.2 & 0.2 & 0.2 & 0.2 & 0.2 & 0.2 \\
 &  & \methodS & 0.3 & 0.2 & 0.2 & 0.1 & 0.3 & 0.1 \\ \cmidrule(lr){2-9}
 & \multirow[c]{2}{*}{\( \eta \)} & \method & 0.1 & 0.1 & 0.1 & 0.1 & 0.1 & 0.1 \\
 &  & \methodS & 0.1 & 0.1 & 0.1 & 0.1 & 0.1 & 0.1 \\ \cmidrule(lr){2-9}
 & \multirow[c]{2}{*}{\( \kappa \)} & \method & 0.5 & 0.5 & 0.5 & 0.5 & 0.5 & 0.5 \\
 &  & \methodS & 0.5 & 0.5 & 0.5 & 0.2 & 0.5 & 0.2 \\ \cmidrule(lr){2-9}
 & \multirow[c]{2}{*}{\( \varepsilon \)} & \method & -- & -- & -- & -- & -- & -- \\
 &  & \methodS & 8 & 8 & 5 & 12 & 8 & 8 \\ \cmidrule(lr){2-9}
 & \multirow[c]{2}{*}{\( d \)} & \method & 400 & 400 & 400 & 400 & 400 & 400 \\
 &  & \methodS & 400 & 400 & 400 & 400 & 400 & 400 \\
\bottomrule
\end{tabular}
            \caption{Hyperparameter values used in our experiments.}
            \label{tab:hyperparameters}
        \end{table*}

    \customparagraph{Code}
        The code for our experiments is available at: \githubLink

    \customparagraph{Data splits}
        Details to access the datasets used with the requisite splits (both are consistent with~\cite{veilleuxRealisticEvaluationTransductive2021}) are available in the code repository.
    
    \customparagraph{Base feature extractors}
        \begin{itemize}
            \item \textbf{Resnet-18}: As in~\cite{boudiafTransductiveInformationMaximization2020,veilleuxRealisticEvaluationTransductive2021}, we use the  weights from~\cite{veilleuxRealisticEvaluationTransductive2021}.
                The model is trained using a cross-entropy loss on the base classes.
            \item \textbf{WideRes28-10}: Following~\cite{manglaChartingRightManifold2020,zhuEASEUnsupervisedDiscriminant2022}, we use the weights from~\cite{manglaChartingRightManifold2020}.
                The model is pre-trained using a combination of cross-entropy and rotation prediction~\cite{gidarisUnsupervisedRepresentationLearning2018}, and then fine-tuned with Manifold Mixup~\cite{vermaManifoldMixupBetter2019}.
        \end{itemize}

\subsection{Results}
    \customparagraph{FSL performance}
        The complete lists of accuracies and hubness metrics for all embeddings, classifiers, and feature extractors, are given in Tables~\ref{tab:main-tim-resnet18-1},~\ref{tab:main-s2m2-wrn-s2m2-1},~\ref{tab:main-tim-resnet18-5}, and~\ref{tab:main-s2m2-wrn-s2m2-5}.
        The exhaustive results in these tables form the basis of Table 1, Table 2 and Table 3 in the main text.
        The two proposed approaches consistently outperform prior embeddings across several classifiers, feature extractors and datasets.

    \customparagraph{Effect of the \( \kappa \) and \( \epsilon \) hyperparameters}
        The plots in Figure~\ref{fig:kappaEpsilon} show accuracy on \textit{tiered} \(5\)-shot with SIAMESE for increasing \( \kappa \) and \( \varepsilon \).
        Neither method is particularly sensitive to the choice of \( \kappa \) and \( \varepsilon \), and \methodS is less sensitive to variations in \( \kappa \), than \method.
        Choosing \( \kappa \in [0.5, 1] \) and \(\epsilon \in [3, 20] \) will result in high classification accuracy

    \begin{table*}
        {\scriptsize\centering\begin{tabular}{llllllllllll}
\toprule
 &  &  & \multicolumn{3}{c}{mini} & \multicolumn{3}{c}{tiered} & \multicolumn{3}{c}{CUB} \\
 &  &  & Acc & Skew & Hub.~Occ. & Acc & Skew & Hub.~Occ. & Acc & Skew & Hub.~Occ. \\
Arch. & Clf. & Emb. &  &  &  &  &  &  &  &  &  \\
\midrule
\multirow[c]{54}{*}{\rotatebox{90}{ResNet18}} & \multirow[c]{9}{*}{\rotatebox{90}{ILPC}} & None & \MTCWITHCONF{64.07}{0.28} & \MTCWITHCONF{1.411}{0.01} & \MTCWITHCONF{0.408}{0.001} & \MTCWITHCONF{75.5}{0.28} & \MTCWITHCONF{1.213}{0.009} & \MTCWITHCONF{0.41}{0.001} & \MTCWITHCONF{76.06}{0.27} & \MTCWITHCONF{0.886}{0.006} & \MTCWITHCONF{0.34}{0.001} \\
 &  & L2 & \MTCWITHCONF{69.28}{0.27} & \MTCWITHCONF{0.966}{0.007} & \MTCWITHCONF{0.298}{0.001} & \MTCWITHCONF{77.84}{0.28} & \MTCWITHCONF{0.811}{0.007} & \MTCWITHCONF{0.267}{0.001} & \MTCWITHCONF{79.91}{0.26} & \MTCWITHCONF{0.688}{0.006} & \MTCWITHCONF{0.236}{0.001} \\
 &  & CL2 & \MTCWITHCONF{71.48}{0.27} & \MTCWITHCONF{0.661}{0.005} & \MTCWITHCONF{0.229}{0.001} & \MTCWITHCONF{79.8}{0.27} & \MTCWITHCONF{0.679}{0.006} & \MTCWITHCONF{0.249}{0.001} & \MTCWITHCONF{80.97}{0.26} & \MTCWITHCONF{0.553}{0.005} & \MTCWITHCONF{0.203}{0.001} \\
 &  & ZN & \MTCWITHCONF{71.48}{0.27} & \MTCWITHCONF{0.677}{0.006} & \MTCWITHCONF{0.227}{0.001} & \MTCWITHCONF{79.95}{0.27} & \MTCWITHCONF{0.694}{0.006} & \MTCWITHCONF{0.263}{0.001} & \MTCWITHCONF{81.49}{0.25} & \MTCWITHCONF{0.57}{0.005} & \MTCWITHCONF{0.217}{0.001} \\
 &  & ReRep & \MTCWITHCONF{65.49}{0.28} & \MTCWITHCONF{3.688}{0.007} & \MTCWITHCONF{0.559}{0.001} & \MTCWITHCONF{76.75}{0.28} & \MTCWITHCONF{3.61}{0.01} & \MTCWITHCONF{0.55}{0.001} & \MTCWITHCONF{77.73}{0.26} & \MTCWITHCONF{3.563}{0.007} & \MTCWITHCONF{0.512}{0.001} \\
 &  & EASE & \MTCWITHCONF{71.79}{0.28} & \MTCWITHCONF{0.515}{0.005} & \MTCWITHCONF{0.157}{0.001} & \MTCWITHCONF{80.2}{0.27} & \MTCWITHCONF{0.48}{0.005} & \MTCWITHCONF{0.158}{0.001} & \MTCWITHCONF{81.88}{0.25} & \MTCWITHCONF{0.463}{0.004} & \SECONDBEST{\MTCWITHCONF{0.153}{0.001}} \\
 &  & TCPR & \MTCWITHCONF{71.77}{0.28} & \MTCWITHCONF{0.647}{0.005} & \MTCWITHCONF{0.223}{0.001} & \MTCWITHCONF{80.01}{0.28} & \MTCWITHCONF{0.652}{0.006} & \MTCWITHCONF{0.249}{0.001} & \MTCWITHCONF{81.75}{0.25} & \MTCWITHCONF{0.534}{0.004} & \MTCWITHCONF{0.203}{0.001} \\
 &  & noHub & \SECONDBEST{\MTCWITHCONF{73.18}{0.28}} & \SECONDBEST{\MTCWITHCONF{0.308}{0.005}} & \BEST{\MTCWITHCONF{0.094}{0.001}} & \SECONDBEST{\MTCWITHCONF{80.76}{0.28}} & \SECONDBEST{\MTCWITHCONF{0.296}{0.004}} & \BEST{\MTCWITHCONF{0.101}{0.001}} & \SECONDBEST{\MTCWITHCONF{82.74}{0.26}} & \SECONDBEST{\MTCWITHCONF{0.32}{0.004}} & \BEST{\MTCWITHCONF{0.112}{0.001}} \\
 &  & noHub-S & \BEST{\MTCWITHCONF{74.02}{0.28}} & \BEST{\MTCWITHCONF{0.276}{0.004}} & \SECONDBEST{\MTCWITHCONF{0.13}{0.001}} & \BEST{\MTCWITHCONF{81.34}{0.27}} & \BEST{\MTCWITHCONF{0.281}{0.004}} & \SECONDBEST{\MTCWITHCONF{0.127}{0.001}} & \BEST{\MTCWITHCONF{83.92}{0.25}} & \BEST{\MTCWITHCONF{0.296}{0.003}} & \MTCWITHCONF{0.163}{0.001} \\
\cline{2-12}
 & \multirow[c]{9}{*}{\rotatebox{90}{LaplacianShot}} & None & \MTCWITHCONF{68.92}{0.23} & \MTCWITHCONF{1.341}{0.009} & \MTCWITHCONF{0.408}{0.001} & \MTCWITHCONF{76.43}{0.25} & \MTCWITHCONF{1.214}{0.009} & \MTCWITHCONF{0.41}{0.001} & \MTCWITHCONF{79.17}{0.23} & \MTCWITHCONF{0.887}{0.006} & \MTCWITHCONF{0.34}{0.001} \\
 &  & L2 & \MTCWITHCONF{69.3}{0.23} & \MTCWITHCONF{0.945}{0.007} & \MTCWITHCONF{0.302}{0.001} & \MTCWITHCONF{77.2}{0.25} & \MTCWITHCONF{0.808}{0.007} & \MTCWITHCONF{0.265}{0.001} & \MTCWITHCONF{79.65}{0.23} & \MTCWITHCONF{0.682}{0.006} & \MTCWITHCONF{0.236}{0.001} \\
 &  & CL2 & \MTCWITHCONF{70.68}{0.23} & \MTCWITHCONF{0.661}{0.005} & \MTCWITHCONF{0.231}{0.001} & \MTCWITHCONF{77.98}{0.24} & \MTCWITHCONF{0.689}{0.006} & \MTCWITHCONF{0.248}{0.001} & \MTCWITHCONF{79.99}{0.22} & \MTCWITHCONF{0.547}{0.005} & \MTCWITHCONF{0.201}{0.001} \\
 &  & ZN & \MTCWITHCONF{70.51}{0.23} & \MTCWITHCONF{0.688}{0.006} & \MTCWITHCONF{0.233}{0.001} & \MTCWITHCONF{77.51}{0.24} & \MTCWITHCONF{0.697}{0.006} & \MTCWITHCONF{0.264}{0.001} & \MTCWITHCONF{79.86}{0.22} & \MTCWITHCONF{0.564}{0.005} & \MTCWITHCONF{0.217}{0.001} \\
 &  & ReRep & \MTCWITHCONF{72.75}{0.24} & \MTCWITHCONF{3.653}{0.007} & \MTCWITHCONF{0.548}{0.001} & \MTCWITHCONF{78.95}{0.25} & \MTCWITHCONF{3.605}{0.011} & \MTCWITHCONF{0.549}{0.001} & \MTCWITHCONF{82.38}{0.22} & \MTCWITHCONF{3.565}{0.007} & \MTCWITHCONF{0.512}{0.001} \\
 &  & EASE & \MTCWITHCONF{72.19}{0.23} & \MTCWITHCONF{0.526}{0.005} & \MTCWITHCONF{0.161}{0.001} & \MTCWITHCONF{79.34}{0.24} & \MTCWITHCONF{0.481}{0.005} & \MTCWITHCONF{0.158}{0.001} & \MTCWITHCONF{81.5}{0.22} & \MTCWITHCONF{0.459}{0.004} & \SECONDBEST{\MTCWITHCONF{0.152}{0.001}} \\
 &  & TCPR & \MTCWITHCONF{71.79}{0.23} & \MTCWITHCONF{0.654}{0.005} & \MTCWITHCONF{0.228}{0.001} & \MTCWITHCONF{78.41}{0.24} & \MTCWITHCONF{0.651}{0.005} & \MTCWITHCONF{0.249}{0.001} & \MTCWITHCONF{80.86}{0.22} & \MTCWITHCONF{0.537}{0.004} & \MTCWITHCONF{0.203}{0.001} \\
 &  & noHub & \SECONDBEST{\MTCWITHCONF{73.63}{0.25}} & \SECONDBEST{\MTCWITHCONF{0.305}{0.005}} & \BEST{\MTCWITHCONF{0.094}{0.001}} & \BEST{\MTCWITHCONF{80.84}{0.25}} & \SECONDBEST{\MTCWITHCONF{0.3}{0.005}} & \BEST{\MTCWITHCONF{0.101}{0.001}} & \SECONDBEST{\MTCWITHCONF{83.23}{0.22}} & \SECONDBEST{\MTCWITHCONF{0.318}{0.004}} & \BEST{\MTCWITHCONF{0.112}{0.001}} \\
 &  & noHub-S & \BEST{\MTCWITHCONF{73.79}{0.25}} & \BEST{\MTCWITHCONF{0.276}{0.004}} & \SECONDBEST{\MTCWITHCONF{0.13}{0.001}} & \SECONDBEST{\MTCWITHCONF{80.83}{0.25}} & \BEST{\MTCWITHCONF{0.275}{0.004}} & \SECONDBEST{\MTCWITHCONF{0.125}{0.001}} & \BEST{\MTCWITHCONF{83.47}{0.22}} & \BEST{\MTCWITHCONF{0.299}{0.003}} & \MTCWITHCONF{0.164}{0.001} \\
\cline{2-12}
 & \multirow[c]{9}{*}{\rotatebox{90}{ObliqueManifold}} & None & \MTCWITHCONF{68.89}{0.23} & \MTCWITHCONF{1.412}{0.01} & \MTCWITHCONF{0.407}{0.001} & \MTCWITHCONF{77.07}{0.25} & \MTCWITHCONF{1.21}{0.009} & \MTCWITHCONF{0.409}{0.001} & \MTCWITHCONF{79.4}{0.22} & \MTCWITHCONF{0.887}{0.006} & \MTCWITHCONF{0.341}{0.001} \\
 &  & L2 & \MTCWITHCONF{68.92}{0.23} & \MTCWITHCONF{0.964}{0.007} & \MTCWITHCONF{0.299}{0.001} & \MTCWITHCONF{77.17}{0.25} & \MTCWITHCONF{0.806}{0.007} & \MTCWITHCONF{0.266}{0.001} & \MTCWITHCONF{79.32}{0.22} & \MTCWITHCONF{0.691}{0.006} & \MTCWITHCONF{0.237}{0.001} \\
 &  & CL2 & \MTCWITHCONF{70.86}{0.24} & \MTCWITHCONF{0.66}{0.005} & \MTCWITHCONF{0.228}{0.001} & \MTCWITHCONF{78.92}{0.25} & \MTCWITHCONF{0.68}{0.006} & \MTCWITHCONF{0.249}{0.001} & \MTCWITHCONF{80.29}{0.23} & \MTCWITHCONF{0.547}{0.005} & \MTCWITHCONF{0.202}{0.001} \\
 &  & ZN & \MTCWITHCONF{71.25}{0.24} & \MTCWITHCONF{0.679}{0.006} & \MTCWITHCONF{0.227}{0.001} & \MTCWITHCONF{79.54}{0.25} & \MTCWITHCONF{0.697}{0.006} & \MTCWITHCONF{0.263}{0.001} & \MTCWITHCONF{81.38}{0.23} & \MTCWITHCONF{0.562}{0.005} & \MTCWITHCONF{0.216}{0.001} \\
 &  & ReRep & \SECONDBEST{\MTCWITHCONF{73.3}{0.25}} & \MTCWITHCONF{3.682}{0.007} & \MTCWITHCONF{0.559}{0.001} & \SECONDBEST{\MTCWITHCONF{80.26}{0.26}} & \MTCWITHCONF{3.608}{0.01} & \MTCWITHCONF{0.551}{0.001} & \BEST{\MTCWITHCONF{83.84}{0.23}} & \MTCWITHCONF{3.559}{0.008} & \MTCWITHCONF{0.513}{0.001} \\
 &  & EASE & \MTCWITHCONF{68.4}{0.24} & \MTCWITHCONF{0.516}{0.005} & \MTCWITHCONF{0.156}{0.001} & \MTCWITHCONF{77.33}{0.25} & \MTCWITHCONF{0.477}{0.004} & \MTCWITHCONF{0.158}{0.001} & \MTCWITHCONF{79.03}{0.24} & \MTCWITHCONF{0.461}{0.004} & \SECONDBEST{\MTCWITHCONF{0.152}{0.001}} \\
 &  & TCPR & \MTCWITHCONF{70.74}{0.24} & \MTCWITHCONF{0.646}{0.005} & \MTCWITHCONF{0.223}{0.001} & \MTCWITHCONF{78.92}{0.25} & \MTCWITHCONF{0.649}{0.005} & \MTCWITHCONF{0.249}{0.001} & \MTCWITHCONF{80.18}{0.23} & \MTCWITHCONF{0.537}{0.004} & \MTCWITHCONF{0.204}{0.001} \\
 &  & noHub & \MTCWITHCONF{72.55}{0.26} & \SECONDBEST{\MTCWITHCONF{0.309}{0.005}} & \BEST{\MTCWITHCONF{0.095}{0.001}} & \MTCWITHCONF{79.97}{0.26} & \SECONDBEST{\MTCWITHCONF{0.302}{0.005}} & \BEST{\MTCWITHCONF{0.102}{0.001}} & \MTCWITHCONF{82.21}{0.24} & \SECONDBEST{\MTCWITHCONF{0.319}{0.004}} & \BEST{\MTCWITHCONF{0.112}{0.001}} \\
 &  & noHub-S & \BEST{\MTCWITHCONF{74.24}{0.26}} & \BEST{\MTCWITHCONF{0.274}{0.004}} & \SECONDBEST{\MTCWITHCONF{0.13}{0.001}} & \BEST{\MTCWITHCONF{80.84}{0.26}} & \BEST{\MTCWITHCONF{0.282}{0.004}} & \SECONDBEST{\MTCWITHCONF{0.127}{0.001}} & \SECONDBEST{\MTCWITHCONF{83.67}{0.23}} & \BEST{\MTCWITHCONF{0.294}{0.003}} & \MTCWITHCONF{0.162}{0.001} \\
\cline{2-12}
 & \multirow[c]{9}{*}{\rotatebox{90}{SIAMESE}} & None & \MTCWITHCONF{20.0}{0.0} & \MTCWITHCONF{1.345}{0.009} & \MTCWITHCONF{0.407}{0.001} & \MTCWITHCONF{20.0}{0.0} & \MTCWITHCONF{1.222}{0.009} & \MTCWITHCONF{0.41}{0.001} & \MTCWITHCONF{20.0}{0.0} & \MTCWITHCONF{0.885}{0.006} & \MTCWITHCONF{0.339}{0.001} \\
 &  & L2 & \MTCWITHCONF{73.77}{0.24} & \MTCWITHCONF{0.949}{0.007} & \MTCWITHCONF{0.301}{0.001} & \MTCWITHCONF{80.46}{0.26} & \MTCWITHCONF{0.811}{0.007} & \MTCWITHCONF{0.265}{0.001} & \MTCWITHCONF{83.1}{0.23} & \MTCWITHCONF{0.691}{0.006} & \MTCWITHCONF{0.237}{0.001} \\
 &  & CL2 & \MTCWITHCONF{75.56}{0.26} & \MTCWITHCONF{0.666}{0.005} & \MTCWITHCONF{0.232}{0.001} & \MTCWITHCONF{82.1}{0.26} & \MTCWITHCONF{0.68}{0.006} & \MTCWITHCONF{0.248}{0.001} & \MTCWITHCONF{84.35}{0.24} & \MTCWITHCONF{0.549}{0.005} & \MTCWITHCONF{0.201}{0.001} \\
 &  & ZN & \MTCWITHCONF{20.0}{0.0} & \MTCWITHCONF{0.686}{0.006} & \MTCWITHCONF{0.232}{0.001} & \MTCWITHCONF{20.0}{0.0} & \MTCWITHCONF{0.69}{0.006} & \MTCWITHCONF{0.262}{0.001} & \MTCWITHCONF{20.0}{0.0} & \MTCWITHCONF{0.565}{0.005} & \MTCWITHCONF{0.217}{0.001} \\
 &  & ReRep & \MTCWITHCONF{20.0}{0.0} & \MTCWITHCONF{3.653}{0.007} & \MTCWITHCONF{0.549}{0.001} & \MTCWITHCONF{20.0}{0.0} & \MTCWITHCONF{3.616}{0.01} & \MTCWITHCONF{0.549}{0.001} & \MTCWITHCONF{20.0}{0.0} & \MTCWITHCONF{3.559}{0.007} & \MTCWITHCONF{0.512}{0.001} \\
 &  & EASE & \MTCWITHCONF{76.05}{0.27} & \MTCWITHCONF{0.529}{0.005} & \MTCWITHCONF{0.162}{0.001} & \MTCWITHCONF{82.57}{0.27} & \MTCWITHCONF{0.485}{0.005} & \MTCWITHCONF{0.159}{0.001} & \MTCWITHCONF{85.24}{0.24} & \MTCWITHCONF{0.464}{0.004} & \SECONDBEST{\MTCWITHCONF{0.153}{0.001}} \\
 &  & TCPR & \MTCWITHCONF{75.99}{0.26} & \MTCWITHCONF{0.655}{0.005} & \MTCWITHCONF{0.227}{0.001} & \MTCWITHCONF{82.65}{0.26} & \MTCWITHCONF{0.651}{0.005} & \MTCWITHCONF{0.249}{0.001} & \MTCWITHCONF{85.34}{0.23} & \MTCWITHCONF{0.535}{0.004} & \MTCWITHCONF{0.203}{0.001} \\
 &  & noHub & \SECONDBEST{\MTCWITHCONF{76.65}{0.28}} & \SECONDBEST{\MTCWITHCONF{0.308}{0.005}} & \BEST{\MTCWITHCONF{0.095}{0.001}} & \SECONDBEST{\MTCWITHCONF{82.94}{0.27}} & \SECONDBEST{\MTCWITHCONF{0.303}{0.004}} & \BEST{\MTCWITHCONF{0.101}{0.001}} & \BEST{\MTCWITHCONF{85.88}{0.24}} & \SECONDBEST{\MTCWITHCONF{0.322}{0.004}} & \BEST{\MTCWITHCONF{0.112}{0.001}} \\
 &  & noHub-S & \BEST{\MTCWITHCONF{76.68}{0.28}} & \BEST{\MTCWITHCONF{0.275}{0.004}} & \SECONDBEST{\MTCWITHCONF{0.13}{0.001}} & \BEST{\MTCWITHCONF{83.09}{0.27}} & \BEST{\MTCWITHCONF{0.281}{0.004}} & \SECONDBEST{\MTCWITHCONF{0.128}{0.001}} & \SECONDBEST{\MTCWITHCONF{85.81}{0.24}} & \BEST{\MTCWITHCONF{0.295}{0.003}} & \MTCWITHCONF{0.161}{0.001} \\
\cline{2-12}
 & \multirow[c]{9}{*}{\rotatebox{90}{SimpleShot}} & None & \MTCWITHCONF{56.14}{0.2} & \MTCWITHCONF{1.349}{0.009} & \MTCWITHCONF{0.407}{0.001} & \MTCWITHCONF{63.34}{0.23} & \MTCWITHCONF{1.211}{0.009} & \MTCWITHCONF{0.408}{0.001} & \MTCWITHCONF{64.02}{0.21} & \MTCWITHCONF{0.887}{0.006} & \MTCWITHCONF{0.341}{0.001} \\
 &  & L2 & \MTCWITHCONF{60.15}{0.2} & \MTCWITHCONF{0.937}{0.007} & \MTCWITHCONF{0.301}{0.001} & \MTCWITHCONF{68.02}{0.23} & \MTCWITHCONF{0.812}{0.007} & \MTCWITHCONF{0.265}{0.001} & \MTCWITHCONF{69.05}{0.21} & \MTCWITHCONF{0.691}{0.006} & \MTCWITHCONF{0.236}{0.001} \\
 &  & CL2 & \MTCWITHCONF{63.1}{0.2} & \MTCWITHCONF{0.667}{0.005} & \MTCWITHCONF{0.233}{0.001} & \MTCWITHCONF{69.76}{0.22} & \MTCWITHCONF{0.679}{0.006} & \MTCWITHCONF{0.249}{0.001} & \MTCWITHCONF{70.16}{0.2} & \MTCWITHCONF{0.549}{0.005} & \MTCWITHCONF{0.201}{0.001} \\
 &  & ZN & \MTCWITHCONF{63.39}{0.2} & \MTCWITHCONF{0.68}{0.005} & \MTCWITHCONF{0.231}{0.001} & \MTCWITHCONF{70.04}{0.22} & \MTCWITHCONF{0.698}{0.006} & \MTCWITHCONF{0.264}{0.001} & \MTCWITHCONF{71.03}{0.2} & \MTCWITHCONF{0.564}{0.005} & \MTCWITHCONF{0.216}{0.001} \\
 &  & ReRep & \MTCWITHCONF{66.66}{0.22} & \MTCWITHCONF{3.655}{0.007} & \MTCWITHCONF{0.548}{0.001} & \MTCWITHCONF{73.23}{0.23} & \MTCWITHCONF{3.604}{0.01} & \MTCWITHCONF{0.549}{0.001} & \MTCWITHCONF{76.8}{0.21} & \MTCWITHCONF{3.565}{0.007} & \MTCWITHCONF{0.513}{0.001} \\
 &  & EASE & \MTCWITHCONF{64.0}{0.2} & \MTCWITHCONF{0.521}{0.005} & \MTCWITHCONF{0.16}{0.001} & \MTCWITHCONF{71.0}{0.21} & \MTCWITHCONF{0.479}{0.005} & \MTCWITHCONF{0.158}{0.001} & \MTCWITHCONF{72.38}{0.2} & \MTCWITHCONF{0.466}{0.004} & \SECONDBEST{\MTCWITHCONF{0.153}{0.001}} \\
 &  & TCPR & \MTCWITHCONF{63.33}{0.2} & \MTCWITHCONF{0.651}{0.005} & \MTCWITHCONF{0.228}{0.001} & \MTCWITHCONF{69.82}{0.22} & \MTCWITHCONF{0.65}{0.005} & \MTCWITHCONF{0.25}{0.001} & \MTCWITHCONF{70.75}{0.2} & \MTCWITHCONF{0.532}{0.004} & \MTCWITHCONF{0.204}{0.001} \\
 &  & noHub & \SECONDBEST{\MTCWITHCONF{69.38}{0.22}} & \SECONDBEST{\MTCWITHCONF{0.315}{0.005}} & \BEST{\MTCWITHCONF{0.095}{0.001}} & \SECONDBEST{\MTCWITHCONF{76.72}{0.23}} & \SECONDBEST{\MTCWITHCONF{0.303}{0.004}} & \BEST{\MTCWITHCONF{0.102}{0.001}} & \SECONDBEST{\MTCWITHCONF{78.21}{0.21}} & \SECONDBEST{\MTCWITHCONF{0.32}{0.004}} & \BEST{\MTCWITHCONF{0.112}{0.001}} \\
 &  & noHub-S & \BEST{\MTCWITHCONF{71.1}{0.22}} & \BEST{\MTCWITHCONF{0.276}{0.004}} & \SECONDBEST{\MTCWITHCONF{0.13}{0.001}} & \BEST{\MTCWITHCONF{78.35}{0.23}} & \BEST{\MTCWITHCONF{0.283}{0.004}} & \SECONDBEST{\MTCWITHCONF{0.127}{0.001}} & \BEST{\MTCWITHCONF{80.31}{0.21}} & \BEST{\MTCWITHCONF{0.296}{0.003}} & \MTCWITHCONF{0.162}{0.001} \\
\cline{2-12}
 & \multirow[c]{9}{*}{\rotatebox{90}{\( \alpha \)-TIM}} & None & \MTCWITHCONF{56.39}{0.2} & \MTCWITHCONF{1.342}{0.009} & \MTCWITHCONF{0.406}{0.001} & \MTCWITHCONF{63.32}{0.23} & \MTCWITHCONF{1.216}{0.009} & \MTCWITHCONF{0.411}{0.001} & \MTCWITHCONF{64.02}{0.22} & \MTCWITHCONF{0.886}{0.006} & \MTCWITHCONF{0.341}{0.001} \\
 &  & L2 & \MTCWITHCONF{67.91}{0.23} & \MTCWITHCONF{0.942}{0.007} & \MTCWITHCONF{0.301}{0.001} & \MTCWITHCONF{74.94}{0.24} & \MTCWITHCONF{0.814}{0.007} & \MTCWITHCONF{0.266}{0.001} & \MTCWITHCONF{77.49}{0.23} & \MTCWITHCONF{0.694}{0.006} & \MTCWITHCONF{0.236}{0.001} \\
 &  & CL2 & \MTCWITHCONF{65.68}{0.21} & \MTCWITHCONF{0.665}{0.005} & \MTCWITHCONF{0.232}{0.001} & \MTCWITHCONF{73.23}{0.23} & \MTCWITHCONF{0.681}{0.006} & \MTCWITHCONF{0.248}{0.001} & \MTCWITHCONF{73.79}{0.21} & \MTCWITHCONF{0.552}{0.005} & \MTCWITHCONF{0.202}{0.001} \\
 &  & ZN & \MTCWITHCONF{63.36}{0.2} & \MTCWITHCONF{0.682}{0.005} & \MTCWITHCONF{0.232}{0.001} & \MTCWITHCONF{70.19}{0.22} & \MTCWITHCONF{0.693}{0.006} & \MTCWITHCONF{0.263}{0.001} & \MTCWITHCONF{70.85}{0.21} & \MTCWITHCONF{0.566}{0.005} & \MTCWITHCONF{0.215}{0.001} \\
 &  & ReRep & \MTCWITHCONF{66.37}{0.22} & \MTCWITHCONF{3.656}{0.007} & \MTCWITHCONF{0.55}{0.001} & \MTCWITHCONF{73.24}{0.24} & \MTCWITHCONF{3.605}{0.011} & \MTCWITHCONF{0.55}{0.001} & \MTCWITHCONF{76.86}{0.21} & \MTCWITHCONF{3.555}{0.007} & \MTCWITHCONF{0.514}{0.001} \\
 &  & EASE & \MTCWITHCONF{65.32}{0.2} & \MTCWITHCONF{0.526}{0.005} & \MTCWITHCONF{0.163}{0.001} & \MTCWITHCONF{71.88}{0.22} & \MTCWITHCONF{0.477}{0.005} & \MTCWITHCONF{0.158}{0.001} & \MTCWITHCONF{73.03}{0.21} & \MTCWITHCONF{0.459}{0.004} & \SECONDBEST{\MTCWITHCONF{0.151}{0.001}} \\
 &  & TCPR & \MTCWITHCONF{66.19}{0.21} & \MTCWITHCONF{0.65}{0.005} & \MTCWITHCONF{0.227}{0.001} & \MTCWITHCONF{73.24}{0.23} & \MTCWITHCONF{0.649}{0.005} & \MTCWITHCONF{0.25}{0.001} & \MTCWITHCONF{74.07}{0.21} & \MTCWITHCONF{0.532}{0.004} & \MTCWITHCONF{0.203}{0.001} \\
 &  & noHub & \SECONDBEST{\MTCWITHCONF{70.08}{0.23}} & \SECONDBEST{\MTCWITHCONF{0.312}{0.005}} & \BEST{\MTCWITHCONF{0.094}{0.001}} & \SECONDBEST{\MTCWITHCONF{77.39}{0.24}} & \SECONDBEST{\MTCWITHCONF{0.304}{0.004}} & \BEST{\MTCWITHCONF{0.101}{0.001}} & \SECONDBEST{\MTCWITHCONF{79.19}{0.22}} & \SECONDBEST{\MTCWITHCONF{0.319}{0.004}} & \BEST{\MTCWITHCONF{0.112}{0.001}} \\
 &  & noHub-S & \BEST{\MTCWITHCONF{72.04}{0.23}} & \BEST{\MTCWITHCONF{0.273}{0.004}} & \SECONDBEST{\MTCWITHCONF{0.13}{0.001}} & \BEST{\MTCWITHCONF{79.13}{0.24}} & \BEST{\MTCWITHCONF{0.282}{0.004}} & \SECONDBEST{\MTCWITHCONF{0.126}{0.001}} & \BEST{\MTCWITHCONF{81.42}{0.22}} & \BEST{\MTCWITHCONF{0.296}{0.003}} & \MTCWITHCONF{0.161}{0.001} \\
\cline{1-12} \cline{2-12}
\bottomrule
\end{tabular}
}
        \caption{Resnet-18: 1-shot}
        \label{tab:main-tim-resnet18-1}
    \end{table*}
    \begin{table*}
        {\scriptsize\centering\begin{tabular}{llllllllllll}
\toprule
 &  &  & \multicolumn{3}{c}{mini} & \multicolumn{3}{c}{tiered} & \multicolumn{3}{c}{CUB} \\
 &  &  & Acc & Skew & Hub.~Occ. & Acc & Skew & Hub.~Occ. & Acc & Skew & Hub.~Occ. \\
Arch. & Clf. & Emb. &  &  &  &  &  &  &  &  &  \\
\midrule
\multirow[c]{54}{*}{\rotatebox{90}{WideRes28-10}} & \multirow[c]{9}{*}{\rotatebox{90}{ILPC}} & None & \MTCWITHCONF{71.27}{0.28} & \MTCWITHCONF{1.595}{0.01} & \MTCWITHCONF{0.46}{0.001} & \MTCWITHCONF{75.01}{0.28} & \MTCWITHCONF{1.807}{0.01} & \MTCWITHCONF{0.494}{0.001} & \MTCWITHCONF{89.75}{0.19} & \MTCWITHCONF{1.072}{0.009} & \MTCWITHCONF{0.367}{0.001} \\
 &  & L2 & \MTCWITHCONF{76.41}{0.26} & \MTCWITHCONF{0.773}{0.006} & \MTCWITHCONF{0.295}{0.001} & \MTCWITHCONF{78.25}{0.27} & \MTCWITHCONF{0.731}{0.006} & \MTCWITHCONF{0.274}{0.001} & \MTCWITHCONF{90.27}{0.2} & \MTCWITHCONF{0.473}{0.004} & \MTCWITHCONF{0.228}{0.001} \\
 &  & CL2 & \MTCWITHCONF{74.13}{0.27} & \MTCWITHCONF{0.993}{0.009} & \MTCWITHCONF{0.29}{0.001} & \MTCWITHCONF{78.2}{0.27} & \MTCWITHCONF{0.815}{0.006} & \MTCWITHCONF{0.306}{0.001} & \MTCWITHCONF{90.34}{0.2} & \MTCWITHCONF{0.524}{0.004} & \MTCWITHCONF{0.267}{0.001} \\
 &  & ZN & \MTCWITHCONF{77.76}{0.26} & \MTCWITHCONF{0.728}{0.005} & \MTCWITHCONF{0.287}{0.001} & \MTCWITHCONF{79.42}{0.27} & \MTCWITHCONF{0.776}{0.006} & \MTCWITHCONF{0.302}{0.001} & \MTCWITHCONF{90.21}{0.2} & \MTCWITHCONF{0.516}{0.004} & \MTCWITHCONF{0.263}{0.001} \\
 &  & ReRep & \MTCWITHCONF{62.51}{0.34} & \MTCWITHCONF{3.56}{0.002} & \MTCWITHCONF{0.704}{0.001} & \MTCWITHCONF{60.66}{0.37} & \MTCWITHCONF{3.55}{0.002} & \MTCWITHCONF{0.776}{0.001} & \MTCWITHCONF{87.44}{0.25} & \MTCWITHCONF{3.033}{0.008} & \MTCWITHCONF{0.472}{0.001} \\
 &  & EASE & \MTCWITHCONF{78.01}{0.26} & \MTCWITHCONF{0.47}{0.004} & \MTCWITHCONF{0.176}{0.001} & \MTCWITHCONF{79.64}{0.27} & \MTCWITHCONF{0.479}{0.004} & \MTCWITHCONF{0.175}{0.001} & \MTCWITHCONF{90.76}{0.19} & \MTCWITHCONF{0.437}{0.003} & \MTCWITHCONF{0.212}{0.001} \\
 &  & TCPR & \MTCWITHCONF{78.37}{0.26} & \MTCWITHCONF{0.584}{0.005} & \MTCWITHCONF{0.237}{0.001} & \MTCWITHCONF{79.55}{0.28} & \MTCWITHCONF{0.683}{0.006} & \MTCWITHCONF{0.265}{0.001} & \MTCWITHCONF{90.77}{0.19} & \MTCWITHCONF{0.476}{0.004} & \MTCWITHCONF{0.23}{0.001} \\
 &  & noHub & \SECONDBEST{\MTCWITHCONF{78.84}{0.27}} & \SECONDBEST{\MTCWITHCONF{0.293}{0.004}} & \BEST{\MTCWITHCONF{0.112}{0.001}} & \SECONDBEST{\MTCWITHCONF{80.75}{0.28}} & \SECONDBEST{\MTCWITHCONF{0.3}{0.004}} & \BEST{\MTCWITHCONF{0.112}{0.001}} & \SECONDBEST{\MTCWITHCONF{90.91}{0.2}} & \SECONDBEST{\MTCWITHCONF{0.189}{0.004}} & \BEST{\MTCWITHCONF{0.109}{0.001}} \\
 &  & noHub-S & \BEST{\MTCWITHCONF{79.77}{0.26}} & \BEST{\MTCWITHCONF{0.262}{0.004}} & \SECONDBEST{\MTCWITHCONF{0.148}{0.001}} & \BEST{\MTCWITHCONF{81.24}{0.27}} & \BEST{\MTCWITHCONF{0.278}{0.004}} & \SECONDBEST{\MTCWITHCONF{0.135}{0.001}} & \BEST{\MTCWITHCONF{91.28}{0.19}} & \BEST{\MTCWITHCONF{0.16}{0.004}} & \SECONDBEST{\MTCWITHCONF{0.13}{0.001}} \\
\cline{2-12}
 & \multirow[c]{9}{*}{\rotatebox{90}{LaplacianShot}} & None & \MTCWITHCONF{72.56}{0.23} & \MTCWITHCONF{1.599}{0.01} & \MTCWITHCONF{0.459}{0.001} & \MTCWITHCONF{75.58}{0.25} & \MTCWITHCONF{1.795}{0.01} & \MTCWITHCONF{0.495}{0.001} & \MTCWITHCONF{88.71}{0.19} & \MTCWITHCONF{1.071}{0.009} & \MTCWITHCONF{0.369}{0.001} \\
 &  & L2 & \MTCWITHCONF{75.18}{0.23} & \MTCWITHCONF{0.777}{0.006} & \MTCWITHCONF{0.296}{0.001} & \MTCWITHCONF{77.03}{0.24} & \MTCWITHCONF{0.732}{0.006} & \MTCWITHCONF{0.274}{0.001} & \MTCWITHCONF{89.73}{0.17} & \MTCWITHCONF{0.474}{0.004} & \MTCWITHCONF{0.229}{0.001} \\
 &  & CL2 & \MTCWITHCONF{71.29}{0.24} & \MTCWITHCONF{0.987}{0.009} & \MTCWITHCONF{0.29}{0.001} & \MTCWITHCONF{75.42}{0.25} & \MTCWITHCONF{0.819}{0.006} & \MTCWITHCONF{0.309}{0.001} & \MTCWITHCONF{89.61}{0.18} & \MTCWITHCONF{0.52}{0.004} & \MTCWITHCONF{0.268}{0.001} \\
 &  & ZN & \MTCWITHCONF{75.18}{0.22} & \MTCWITHCONF{0.724}{0.005} & \MTCWITHCONF{0.286}{0.001} & \MTCWITHCONF{77.0}{0.24} & \MTCWITHCONF{0.768}{0.006} & \MTCWITHCONF{0.301}{0.001} & \MTCWITHCONF{89.22}{0.18} & \MTCWITHCONF{0.517}{0.004} & \MTCWITHCONF{0.263}{0.001} \\
 &  & ReRep & \MTCWITHCONF{75.25}{0.22} & \MTCWITHCONF{3.562}{0.002} & \MTCWITHCONF{0.704}{0.001} & \MTCWITHCONF{77.12}{0.24} & \MTCWITHCONF{3.548}{0.002} & \MTCWITHCONF{0.776}{0.001} & \MTCWITHCONF{88.98}{0.18} & \MTCWITHCONF{3.024}{0.008} & \MTCWITHCONF{0.47}{0.001} \\
 &  & EASE & \MTCWITHCONF{77.29}{0.22} & \MTCWITHCONF{0.473}{0.004} & \MTCWITHCONF{0.177}{0.001} & \MTCWITHCONF{78.97}{0.24} & \MTCWITHCONF{0.475}{0.004} & \MTCWITHCONF{0.175}{0.001} & \MTCWITHCONF{90.06}{0.17} & \MTCWITHCONF{0.435}{0.003} & \MTCWITHCONF{0.213}{0.001} \\
 &  & TCPR & \MTCWITHCONF{76.77}{0.22} & \MTCWITHCONF{0.593}{0.005} & \MTCWITHCONF{0.236}{0.001} & \MTCWITHCONF{77.49}{0.24} & \MTCWITHCONF{0.686}{0.006} & \MTCWITHCONF{0.264}{0.001} & \MTCWITHCONF{89.42}{0.17} & \MTCWITHCONF{0.475}{0.004} & \MTCWITHCONF{0.231}{0.001} \\
 &  & noHub & \BEST{\MTCWITHCONF{79.13}{0.23}} & \SECONDBEST{\MTCWITHCONF{0.29}{0.004}} & \BEST{\MTCWITHCONF{0.111}{0.001}} & \SECONDBEST{\MTCWITHCONF{80.5}{0.25}} & \SECONDBEST{\MTCWITHCONF{0.302}{0.004}} & \BEST{\MTCWITHCONF{0.112}{0.001}} & \BEST{\MTCWITHCONF{90.73}{0.18}} & \SECONDBEST{\MTCWITHCONF{0.19}{0.004}} & \BEST{\MTCWITHCONF{0.109}{0.001}} \\
 &  & noHub-S & \SECONDBEST{\MTCWITHCONF{79.13}{0.23}} & \BEST{\MTCWITHCONF{0.259}{0.004}} & \SECONDBEST{\MTCWITHCONF{0.147}{0.001}} & \BEST{\MTCWITHCONF{80.59}{0.24}} & \BEST{\MTCWITHCONF{0.277}{0.004}} & \SECONDBEST{\MTCWITHCONF{0.135}{0.001}} & \SECONDBEST{\MTCWITHCONF{90.61}{0.17}} & \BEST{\MTCWITHCONF{0.164}{0.004}} & \SECONDBEST{\MTCWITHCONF{0.13}{0.001}} \\
\cline{2-12}
 & \multirow[c]{9}{*}{\rotatebox{90}{ObliqueManifold}} & None & \MTCWITHCONF{76.02}{0.22} & \MTCWITHCONF{1.599}{0.01} & \MTCWITHCONF{0.46}{0.001} & \MTCWITHCONF{77.75}{0.25} & \MTCWITHCONF{1.801}{0.01} & \MTCWITHCONF{0.494}{0.001} & \MTCWITHCONF{90.82}{0.18} & \MTCWITHCONF{1.07}{0.009} & \MTCWITHCONF{0.368}{0.001} \\
 &  & L2 & \MTCWITHCONF{76.11}{0.22} & \MTCWITHCONF{0.779}{0.006} & \MTCWITHCONF{0.295}{0.001} & \MTCWITHCONF{77.74}{0.25} & \MTCWITHCONF{0.731}{0.006} & \MTCWITHCONF{0.274}{0.001} & \MTCWITHCONF{90.89}{0.18} & \MTCWITHCONF{0.475}{0.004} & \MTCWITHCONF{0.228}{0.001} \\
 &  & CL2 & \MTCWITHCONF{74.43}{0.24} & \MTCWITHCONF{0.985}{0.009} & \MTCWITHCONF{0.289}{0.001} & \MTCWITHCONF{77.98}{0.25} & \MTCWITHCONF{0.816}{0.007} & \MTCWITHCONF{0.307}{0.001} & \MTCWITHCONF{90.6}{0.18} & \MTCWITHCONF{0.523}{0.004} & \MTCWITHCONF{0.267}{0.001} \\
 &  & ZN & \MTCWITHCONF{77.69}{0.23} & \MTCWITHCONF{0.724}{0.005} & \MTCWITHCONF{0.286}{0.001} & \MTCWITHCONF{79.32}{0.24} & \MTCWITHCONF{0.767}{0.006} & \MTCWITHCONF{0.301}{0.001} & \MTCWITHCONF{90.73}{0.18} & \MTCWITHCONF{0.519}{0.004} & \MTCWITHCONF{0.263}{0.001} \\
 &  & ReRep & \MTCWITHCONF{78.08}{0.23} & \MTCWITHCONF{3.56}{0.002} & \MTCWITHCONF{0.703}{0.001} & \MTCWITHCONF{79.46}{0.25} & \MTCWITHCONF{3.549}{0.002} & \MTCWITHCONF{0.777}{0.001} & \SECONDBEST{\MTCWITHCONF{91.16}{0.18}} & \MTCWITHCONF{3.032}{0.008} & \MTCWITHCONF{0.471}{0.001} \\
 &  & EASE & \MTCWITHCONF{74.77}{0.23} & \MTCWITHCONF{0.472}{0.004} & \MTCWITHCONF{0.178}{0.001} & \MTCWITHCONF{77.07}{0.25} & \MTCWITHCONF{0.473}{0.004} & \MTCWITHCONF{0.174}{0.001} & \MTCWITHCONF{89.2}{0.18} & \MTCWITHCONF{0.439}{0.003} & \MTCWITHCONF{0.212}{0.001} \\
 &  & TCPR & \MTCWITHCONF{77.39}{0.23} & \MTCWITHCONF{0.587}{0.005} & \MTCWITHCONF{0.236}{0.001} & \MTCWITHCONF{78.75}{0.24} & \MTCWITHCONF{0.687}{0.006} & \MTCWITHCONF{0.265}{0.001} & \MTCWITHCONF{89.93}{0.19} & \MTCWITHCONF{0.474}{0.004} & \MTCWITHCONF{0.23}{0.001} \\
 &  & noHub & \SECONDBEST{\MTCWITHCONF{78.44}{0.24}} & \SECONDBEST{\MTCWITHCONF{0.292}{0.004}} & \BEST{\MTCWITHCONF{0.112}{0.001}} & \SECONDBEST{\MTCWITHCONF{79.99}{0.26}} & \SECONDBEST{\MTCWITHCONF{0.302}{0.004}} & \BEST{\MTCWITHCONF{0.113}{0.001}} & \MTCWITHCONF{90.59}{0.19} & \SECONDBEST{\MTCWITHCONF{0.185}{0.004}} & \BEST{\MTCWITHCONF{0.108}{0.001}} \\
 &  & noHub-S & \BEST{\MTCWITHCONF{79.89}{0.24}} & \BEST{\MTCWITHCONF{0.259}{0.004}} & \SECONDBEST{\MTCWITHCONF{0.148}{0.001}} & \BEST{\MTCWITHCONF{80.67}{0.26}} & \BEST{\MTCWITHCONF{0.279}{0.004}} & \SECONDBEST{\MTCWITHCONF{0.137}{0.001}} & \BEST{\MTCWITHCONF{91.37}{0.18}} & \BEST{\MTCWITHCONF{0.162}{0.004}} & \SECONDBEST{\MTCWITHCONF{0.13}{0.001}} \\
\cline{2-12}
 & \multirow[c]{9}{*}{\rotatebox{90}{SIAMESE}} & None & \MTCWITHCONF{45.69}{0.31} & \MTCWITHCONF{1.594}{0.009} & \MTCWITHCONF{0.459}{0.001} & \MTCWITHCONF{75.29}{0.28} & \MTCWITHCONF{1.801}{0.01} & \MTCWITHCONF{0.495}{0.001} & \MTCWITHCONF{61.36}{0.55} & \MTCWITHCONF{1.074}{0.009} & \MTCWITHCONF{0.37}{0.001} \\
 &  & L2 & \MTCWITHCONF{80.2}{0.23} & \MTCWITHCONF{0.776}{0.006} & \MTCWITHCONF{0.296}{0.001} & \MTCWITHCONF{80.89}{0.26} & \MTCWITHCONF{0.735}{0.006} & \MTCWITHCONF{0.275}{0.001} & \MTCWITHCONF{91.98}{0.18} & \MTCWITHCONF{0.476}{0.004} & \MTCWITHCONF{0.23}{0.001} \\
 &  & CL2 & \MTCWITHCONF{75.23}{0.27} & \MTCWITHCONF{0.988}{0.009} & \MTCWITHCONF{0.289}{0.001} & \MTCWITHCONF{79.59}{0.27} & \MTCWITHCONF{0.82}{0.006} & \MTCWITHCONF{0.307}{0.001} & \MTCWITHCONF{92.17}{0.18} & \MTCWITHCONF{0.518}{0.004} & \MTCWITHCONF{0.266}{0.001} \\
 &  & ZN & \MTCWITHCONF{20.0}{0.0} & \MTCWITHCONF{0.726}{0.005} & \MTCWITHCONF{0.286}{0.001} & \MTCWITHCONF{20.0}{0.0} & \MTCWITHCONF{0.775}{0.006} & \MTCWITHCONF{0.302}{0.001} & \MTCWITHCONF{20.0}{0.0} & \MTCWITHCONF{0.517}{0.004} & \MTCWITHCONF{0.264}{0.001} \\
 &  & ReRep & \MTCWITHCONF{36.69}{0.28} & \MTCWITHCONF{3.561}{0.002} & \MTCWITHCONF{0.705}{0.001} & \MTCWITHCONF{67.41}{0.29} & \MTCWITHCONF{3.55}{0.002} & \MTCWITHCONF{0.776}{0.001} & \MTCWITHCONF{57.62}{0.56} & \MTCWITHCONF{3.027}{0.008} & \MTCWITHCONF{0.472}{0.001} \\
 &  & EASE & \MTCWITHCONF{81.19}{0.25} & \MTCWITHCONF{0.474}{0.004} & \MTCWITHCONF{0.178}{0.001} & \MTCWITHCONF{82.04}{0.26} & \MTCWITHCONF{0.476}{0.004} & \MTCWITHCONF{0.176}{0.001} & \MTCWITHCONF{91.99}{0.19} & \MTCWITHCONF{0.436}{0.003} & \MTCWITHCONF{0.213}{0.001} \\
 &  & TCPR & \MTCWITHCONF{81.27}{0.24} & \MTCWITHCONF{0.582}{0.005} & \MTCWITHCONF{0.236}{0.001} & \MTCWITHCONF{81.89}{0.26} & \MTCWITHCONF{0.681}{0.006} & \MTCWITHCONF{0.264}{0.001} & \MTCWITHCONF{91.91}{0.18} & \MTCWITHCONF{0.477}{0.004} & \MTCWITHCONF{0.232}{0.001} \\
 &  & noHub & \SECONDBEST{\MTCWITHCONF{81.97}{0.25}} & \SECONDBEST{\MTCWITHCONF{0.291}{0.004}} & \BEST{\MTCWITHCONF{0.111}{0.001}} & \SECONDBEST{\MTCWITHCONF{82.8}{0.27}} & \SECONDBEST{\MTCWITHCONF{0.298}{0.004}} & \BEST{\MTCWITHCONF{0.112}{0.001}} & \SECONDBEST{\MTCWITHCONF{92.53}{0.18}} & \SECONDBEST{\MTCWITHCONF{0.189}{0.004}} & \BEST{\MTCWITHCONF{0.109}{0.001}} \\
 &  & noHub-S & \BEST{\MTCWITHCONF{82.0}{0.26}} & \BEST{\MTCWITHCONF{0.258}{0.004}} & \SECONDBEST{\MTCWITHCONF{0.148}{0.001}} & \BEST{\MTCWITHCONF{82.85}{0.27}} & \BEST{\MTCWITHCONF{0.278}{0.004}} & \SECONDBEST{\MTCWITHCONF{0.137}{0.001}} & \BEST{\MTCWITHCONF{92.63}{0.18}} & \BEST{\MTCWITHCONF{0.159}{0.004}} & \SECONDBEST{\MTCWITHCONF{0.13}{0.001}} \\
\cline{2-12}
 & \multirow[c]{9}{*}{\rotatebox{90}{SimpleShot}} & None & \MTCWITHCONF{55.66}{0.21} & \MTCWITHCONF{1.6}{0.01} & \MTCWITHCONF{0.459}{0.001} & \MTCWITHCONF{54.71}{0.22} & \MTCWITHCONF{1.81}{0.01} & \MTCWITHCONF{0.494}{0.001} & \MTCWITHCONF{70.92}{0.23} & \MTCWITHCONF{1.073}{0.009} & \MTCWITHCONF{0.369}{0.001} \\
 &  & L2 & \MTCWITHCONF{65.78}{0.2} & \MTCWITHCONF{0.781}{0.006} & \MTCWITHCONF{0.296}{0.001} & \MTCWITHCONF{68.75}{0.22} & \MTCWITHCONF{0.737}{0.006} & \MTCWITHCONF{0.275}{0.001} & \MTCWITHCONF{82.85}{0.19} & \MTCWITHCONF{0.475}{0.004} & \MTCWITHCONF{0.228}{0.001} \\
 &  & CL2 & \MTCWITHCONF{64.33}{0.2} & \MTCWITHCONF{0.981}{0.009} & \MTCWITHCONF{0.288}{0.001} & \MTCWITHCONF{67.66}{0.22} & \MTCWITHCONF{0.817}{0.006} & \MTCWITHCONF{0.307}{0.001} & \MTCWITHCONF{82.8}{0.19} & \MTCWITHCONF{0.52}{0.004} & \MTCWITHCONF{0.267}{0.001} \\
 &  & ZN & \MTCWITHCONF{67.31}{0.2} & \MTCWITHCONF{0.73}{0.005} & \MTCWITHCONF{0.287}{0.001} & \MTCWITHCONF{69.14}{0.22} & \MTCWITHCONF{0.769}{0.006} & \MTCWITHCONF{0.302}{0.001} & \MTCWITHCONF{82.79}{0.19} & \MTCWITHCONF{0.517}{0.004} & \MTCWITHCONF{0.263}{0.001} \\
 &  & ReRep & \MTCWITHCONF{67.38}{0.2} & \MTCWITHCONF{3.56}{0.002} & \MTCWITHCONF{0.704}{0.001} & \MTCWITHCONF{70.17}{0.22} & \MTCWITHCONF{3.55}{0.002} & \MTCWITHCONF{0.777}{0.001} & \MTCWITHCONF{84.86}{0.19} & \MTCWITHCONF{3.026}{0.008} & \MTCWITHCONF{0.47}{0.001} \\
 &  & EASE & \MTCWITHCONF{68.62}{0.2} & \MTCWITHCONF{0.47}{0.004} & \MTCWITHCONF{0.177}{0.001} & \MTCWITHCONF{70.26}{0.21} & \MTCWITHCONF{0.477}{0.004} & \MTCWITHCONF{0.175}{0.001} & \MTCWITHCONF{84.14}{0.18} & \MTCWITHCONF{0.437}{0.003} & \MTCWITHCONF{0.213}{0.001} \\
 &  & TCPR & \MTCWITHCONF{68.45}{0.2} & \MTCWITHCONF{0.589}{0.005} & \MTCWITHCONF{0.236}{0.001} & \MTCWITHCONF{68.68}{0.22} & \MTCWITHCONF{0.685}{0.006} & \MTCWITHCONF{0.264}{0.001} & \MTCWITHCONF{82.28}{0.19} & \MTCWITHCONF{0.477}{0.004} & \MTCWITHCONF{0.231}{0.001} \\
 &  & noHub & \SECONDBEST{\MTCWITHCONF{75.06}{0.21}} & \SECONDBEST{\MTCWITHCONF{0.29}{0.004}} & \BEST{\MTCWITHCONF{0.111}{0.001}} & \SECONDBEST{\MTCWITHCONF{76.7}{0.23}} & \SECONDBEST{\MTCWITHCONF{0.301}{0.004}} & \BEST{\MTCWITHCONF{0.111}{0.001}} & \SECONDBEST{\MTCWITHCONF{88.06}{0.18}} & \SECONDBEST{\MTCWITHCONF{0.188}{0.004}} & \BEST{\MTCWITHCONF{0.108}{0.001}} \\
 &  & noHub-S & \BEST{\MTCWITHCONF{76.86}{0.21}} & \BEST{\MTCWITHCONF{0.258}{0.004}} & \SECONDBEST{\MTCWITHCONF{0.148}{0.001}} & \BEST{\MTCWITHCONF{78.4}{0.23}} & \BEST{\MTCWITHCONF{0.274}{0.004}} & \SECONDBEST{\MTCWITHCONF{0.135}{0.001}} & \BEST{\MTCWITHCONF{89.25}{0.18}} & \BEST{\MTCWITHCONF{0.162}{0.004}} & \SECONDBEST{\MTCWITHCONF{0.13}{0.001}} \\
\cline{2-12}
 & \multirow[c]{9}{*}{\rotatebox{90}{\( \alpha \)-TIM}} & None & \MTCWITHCONF{60.31}{0.2} & \MTCWITHCONF{1.603}{0.01} & \MTCWITHCONF{0.458}{0.001} & \MTCWITHCONF{69.42}{0.25} & \MTCWITHCONF{1.811}{0.01} & \MTCWITHCONF{0.494}{0.001} & \MTCWITHCONF{73.83}{0.21} & \MTCWITHCONF{1.072}{0.009} & \MTCWITHCONF{0.369}{0.001} \\
 &  & L2 & \MTCWITHCONF{72.11}{0.22} & \MTCWITHCONF{0.778}{0.006} & \MTCWITHCONF{0.295}{0.001} & \MTCWITHCONF{74.45}{0.23} & \MTCWITHCONF{0.73}{0.006} & \MTCWITHCONF{0.275}{0.001} & \MTCWITHCONF{85.96}{0.19} & \MTCWITHCONF{0.476}{0.004} & \MTCWITHCONF{0.229}{0.001} \\
 &  & CL2 & \MTCWITHCONF{68.5}{0.21} & \MTCWITHCONF{0.988}{0.009} & \MTCWITHCONF{0.29}{0.001} & \MTCWITHCONF{72.17}{0.23} & \MTCWITHCONF{0.811}{0.006} & \MTCWITHCONF{0.306}{0.001} & \MTCWITHCONF{85.6}{0.18} & \MTCWITHCONF{0.522}{0.004} & \MTCWITHCONF{0.267}{0.001} \\
 &  & ZN & \MTCWITHCONF{67.69}{0.2} & \MTCWITHCONF{0.73}{0.005} & \MTCWITHCONF{0.287}{0.001} & \MTCWITHCONF{68.94}{0.22} & \MTCWITHCONF{0.769}{0.006} & \MTCWITHCONF{0.302}{0.001} & \MTCWITHCONF{83.03}{0.19} & \MTCWITHCONF{0.518}{0.004} & \MTCWITHCONF{0.263}{0.001} \\
 &  & ReRep & \MTCWITHCONF{73.15}{0.23} & \MTCWITHCONF{3.56}{0.002} & \MTCWITHCONF{0.704}{0.001} & \MTCWITHCONF{76.19}{0.25} & \MTCWITHCONF{3.551}{0.002} & \MTCWITHCONF{0.778}{0.001} & \MTCWITHCONF{88.55}{0.18} & \MTCWITHCONF{3.027}{0.008} & \MTCWITHCONF{0.472}{0.001} \\
 &  & EASE & \MTCWITHCONF{69.83}{0.2} & \MTCWITHCONF{0.468}{0.004} & \MTCWITHCONF{0.176}{0.001} & \MTCWITHCONF{71.54}{0.22} & \MTCWITHCONF{0.481}{0.004} & \MTCWITHCONF{0.175}{0.001} & \MTCWITHCONF{84.9}{0.19} & \MTCWITHCONF{0.436}{0.003} & \MTCWITHCONF{0.213}{0.001} \\
 &  & TCPR & \MTCWITHCONF{71.6}{0.21} & \MTCWITHCONF{0.586}{0.005} & \MTCWITHCONF{0.237}{0.001} & \MTCWITHCONF{72.71}{0.22} & \MTCWITHCONF{0.689}{0.006} & \MTCWITHCONF{0.264}{0.001} & \MTCWITHCONF{84.99}{0.19} & \MTCWITHCONF{0.479}{0.004} & \MTCWITHCONF{0.231}{0.001} \\
 &  & noHub & \SECONDBEST{\MTCWITHCONF{75.87}{0.22}} & \SECONDBEST{\MTCWITHCONF{0.29}{0.004}} & \BEST{\MTCWITHCONF{0.111}{0.001}} & \SECONDBEST{\MTCWITHCONF{77.83}{0.23}} & \SECONDBEST{\MTCWITHCONF{0.302}{0.004}} & \BEST{\MTCWITHCONF{0.112}{0.001}} & \SECONDBEST{\MTCWITHCONF{88.7}{0.17}} & \SECONDBEST{\MTCWITHCONF{0.189}{0.004}} & \BEST{\MTCWITHCONF{0.108}{0.001}} \\
 &  & noHub-S & \BEST{\MTCWITHCONF{77.76}{0.22}} & \BEST{\MTCWITHCONF{0.259}{0.004}} & \SECONDBEST{\MTCWITHCONF{0.147}{0.001}} & \BEST{\MTCWITHCONF{79.04}{0.24}} & \BEST{\MTCWITHCONF{0.276}{0.004}} & \SECONDBEST{\MTCWITHCONF{0.136}{0.001}} & \BEST{\MTCWITHCONF{89.77}{0.17}} & \BEST{\MTCWITHCONF{0.163}{0.003}} & \SECONDBEST{\MTCWITHCONF{0.13}{0.001}} \\
\cline{1-12} \cline{2-12}
\bottomrule
\end{tabular}
}
        \caption{WideRes28-10: 1-shot}
        \label{tab:main-s2m2-wrn-s2m2-1}
    \end{table*}
    \begin{table*}
            {\scriptsize\centering\begin{tabular}{llllllllllll}
\toprule
 &  &  & \multicolumn{3}{c}{mini} & \multicolumn{3}{c}{tiered} & \multicolumn{3}{c}{CUB} \\
 &  &  & Acc & Skew & Hub.~Occ. & Acc & Skew & Hub.~Occ. & Acc & Skew & Hub.~Occ. \\
Arch. & Clf. & Emb. &  &  &  &  &  &  &  &  &  \\
\midrule
\multirow[c]{54}{*}{\rotatebox{90}{ResNet18}} & \multirow[c]{9}{*}{\rotatebox{90}{ILPC}} & None & \MTCWITHCONF{76.46}{0.18} & \MTCWITHCONF{1.503}{0.01} & \MTCWITHCONF{0.421}{0.001} & \MTCWITHCONF{84.46}{0.18} & \MTCWITHCONF{1.334}{0.008} & \MTCWITHCONF{0.433}{0.001} & \MTCWITHCONF{85.86}{0.14} & \MTCWITHCONF{0.981}{0.005} & \MTCWITHCONF{0.364}{0.001} \\
 &  & L2 & \MTCWITHCONF{80.9}{0.16} & \MTCWITHCONF{1.051}{0.007} & \MTCWITHCONF{0.314}{0.001} & \MTCWITHCONF{86.23}{0.17} & \MTCWITHCONF{0.912}{0.006} & \MTCWITHCONF{0.289}{0.001} & \MTCWITHCONF{88.03}{0.13} & \MTCWITHCONF{0.808}{0.005} & \MTCWITHCONF{0.264}{0.001} \\
 &  & CL2 & \MTCWITHCONF{81.64}{0.16} & \MTCWITHCONF{0.778}{0.005} & \MTCWITHCONF{0.262}{0.001} & \MTCWITHCONF{86.88}{0.17} & \MTCWITHCONF{0.823}{0.006} & \MTCWITHCONF{0.281}{0.001} & \MTCWITHCONF{88.44}{0.13} & \MTCWITHCONF{0.695}{0.005} & \MTCWITHCONF{0.235}{0.001} \\
 &  & ZN & \MTCWITHCONF{81.61}{0.16} & \MTCWITHCONF{0.793}{0.005} & \MTCWITHCONF{0.258}{0.001} & \SECONDBEST{\MTCWITHCONF{86.9}{0.17}} & \MTCWITHCONF{0.841}{0.006} & \MTCWITHCONF{0.297}{0.001} & \MTCWITHCONF{88.44}{0.12} & \MTCWITHCONF{0.717}{0.004} & \MTCWITHCONF{0.25}{0.001} \\
 &  & ReRep & \MTCWITHCONF{74.83}{0.19} & \MTCWITHCONF{1.623}{0.003} & \MTCWITHCONF{0.871}{0.001} & \MTCWITHCONF{83.96}{0.19} & \MTCWITHCONF{1.722}{0.004} & \MTCWITHCONF{0.873}{0.001} & \MTCWITHCONF{84.54}{0.15} & \MTCWITHCONF{1.432}{0.003} & \MTCWITHCONF{0.869}{0.001} \\
 &  & EASE & \MTCWITHCONF{81.75}{0.16} & \MTCWITHCONF{0.618}{0.005} & \MTCWITHCONF{0.182}{0.001} & \MTCWITHCONF{86.84}{0.17} & \MTCWITHCONF{0.593}{0.004} & \MTCWITHCONF{0.181}{0.001} & \SECONDBEST{\MTCWITHCONF{88.85}{0.12}} & \MTCWITHCONF{0.606}{0.004} & \MTCWITHCONF{0.186}{0.001} \\
 &  & TCPR & \MTCWITHCONF{81.76}{0.16} & \MTCWITHCONF{0.766}{0.005} & \MTCWITHCONF{0.254}{0.001} & \MTCWITHCONF{86.78}{0.17} & \MTCWITHCONF{0.801}{0.005} & \MTCWITHCONF{0.284}{0.001} & \MTCWITHCONF{88.69}{0.13} & \MTCWITHCONF{0.683}{0.004} & \MTCWITHCONF{0.237}{0.001} \\
 &  & noHub & \SECONDBEST{\MTCWITHCONF{82.09}{0.16}} & \BEST{\MTCWITHCONF{0.295}{0.004}} & \SECONDBEST{\MTCWITHCONF{0.097}{0.001}} & \MTCWITHCONF{86.81}{0.17} & \BEST{\MTCWITHCONF{0.289}{0.004}} & \SECONDBEST{\MTCWITHCONF{0.102}{0.001}} & \MTCWITHCONF{88.85}{0.13} & \BEST{\MTCWITHCONF{0.333}{0.004}} & \SECONDBEST{\MTCWITHCONF{0.12}{0.001}} \\
 &  & noHub-S & \BEST{\MTCWITHCONF{82.33}{0.16}} & \SECONDBEST{\MTCWITHCONF{0.488}{0.006}} & \BEST{\MTCWITHCONF{0.086}{0.001}} & \BEST{\MTCWITHCONF{87.05}{0.17}} & \SECONDBEST{\MTCWITHCONF{0.475}{0.006}} & \BEST{\MTCWITHCONF{0.091}{0.001}} & \BEST{\MTCWITHCONF{89.12}{0.13}} & \SECONDBEST{\MTCWITHCONF{0.438}{0.006}} & \BEST{\MTCWITHCONF{0.097}{0.001}} \\
\cline{2-12}
 & \multirow[c]{9}{*}{\rotatebox{90}{LaplacianShot}} & None & \MTCWITHCONF{81.97}{0.15} & \MTCWITHCONF{1.442}{0.009} & \MTCWITHCONF{0.422}{0.001} & \MTCWITHCONF{86.17}{0.16} & \MTCWITHCONF{1.336}{0.008} & \MTCWITHCONF{0.432}{0.001} & \MTCWITHCONF{88.58}{0.12} & \MTCWITHCONF{0.985}{0.005} & \MTCWITHCONF{0.365}{0.001} \\
 &  & L2 & \MTCWITHCONF{81.89}{0.14} & \MTCWITHCONF{1.035}{0.007} & \MTCWITHCONF{0.319}{0.001} & \MTCWITHCONF{86.19}{0.16} & \MTCWITHCONF{0.913}{0.006} & \MTCWITHCONF{0.289}{0.001} & \MTCWITHCONF{88.52}{0.11} & \MTCWITHCONF{0.811}{0.005} & \MTCWITHCONF{0.264}{0.001} \\
 &  & CL2 & \MTCWITHCONF{81.93}{0.14} & \MTCWITHCONF{0.786}{0.005} & \MTCWITHCONF{0.265}{0.001} & \MTCWITHCONF{86.16}{0.16} & \MTCWITHCONF{0.82}{0.006} & \MTCWITHCONF{0.282}{0.001} & \MTCWITHCONF{88.46}{0.12} & \MTCWITHCONF{0.7}{0.005} & \MTCWITHCONF{0.235}{0.001} \\
 &  & ZN & \MTCWITHCONF{82.57}{0.14} & \MTCWITHCONF{0.803}{0.005} & \MTCWITHCONF{0.263}{0.001} & \MTCWITHCONF{86.67}{0.16} & \MTCWITHCONF{0.838}{0.006} & \MTCWITHCONF{0.296}{0.001} & \MTCWITHCONF{88.88}{0.11} & \MTCWITHCONF{0.714}{0.004} & \MTCWITHCONF{0.25}{0.001} \\
 &  & ReRep & \MTCWITHCONF{82.32}{0.14} & \MTCWITHCONF{1.633}{0.003} & \MTCWITHCONF{0.863}{0.001} & \MTCWITHCONF{86.09}{0.16} & \MTCWITHCONF{1.721}{0.004} & \MTCWITHCONF{0.873}{0.001} & \MTCWITHCONF{88.74}{0.12} & \MTCWITHCONF{1.431}{0.002} & \MTCWITHCONF{0.869}{0.001} \\
 &  & EASE & \SECONDBEST{\MTCWITHCONF{82.57}{0.14}} & \MTCWITHCONF{0.627}{0.005} & \MTCWITHCONF{0.186}{0.001} & \SECONDBEST{\MTCWITHCONF{86.82}{0.15}} & \MTCWITHCONF{0.596}{0.004} & \MTCWITHCONF{0.182}{0.001} & \MTCWITHCONF{88.94}{0.11} & \MTCWITHCONF{0.608}{0.004} & \MTCWITHCONF{0.185}{0.001} \\
 &  & TCPR & \MTCWITHCONF{82.24}{0.14} & \MTCWITHCONF{0.781}{0.005} & \MTCWITHCONF{0.259}{0.001} & \MTCWITHCONF{86.27}{0.16} & \MTCWITHCONF{0.797}{0.005} & \MTCWITHCONF{0.284}{0.001} & \MTCWITHCONF{88.63}{0.11} & \MTCWITHCONF{0.687}{0.004} & \MTCWITHCONF{0.236}{0.001} \\
 &  & noHub & \MTCWITHCONF{82.55}{0.15} & \SECONDBEST{\MTCWITHCONF{0.285}{0.004}} & \SECONDBEST{\MTCWITHCONF{0.096}{0.001}} & \MTCWITHCONF{86.75}{0.16} & \SECONDBEST{\MTCWITHCONF{0.29}{0.004}} & \SECONDBEST{\MTCWITHCONF{0.103}{0.001}} & \BEST{\MTCWITHCONF{89.08}{0.11}} & \BEST{\MTCWITHCONF{0.329}{0.004}} & \SECONDBEST{\MTCWITHCONF{0.12}{0.001}} \\
 &  & noHub-S & \BEST{\MTCWITHCONF{82.81}{0.14}} & \BEST{\MTCWITHCONF{0.25}{0.005}} & \BEST{\MTCWITHCONF{0.073}{0.001}} & \BEST{\MTCWITHCONF{87.12}{0.16}} & \BEST{\MTCWITHCONF{0.214}{0.005}} & \BEST{\MTCWITHCONF{0.077}{0.001}} & \SECONDBEST{\MTCWITHCONF{88.99}{0.11}} & \SECONDBEST{\MTCWITHCONF{0.438}{0.006}} & \BEST{\MTCWITHCONF{0.096}{0.001}} \\
\cline{2-12}
 & \multirow[c]{9}{*}{\rotatebox{90}{ObliqueManifold}} & None & \MTCWITHCONF{83.53}{0.15} & \MTCWITHCONF{1.497}{0.01} & \MTCWITHCONF{0.421}{0.001} & \MTCWITHCONF{87.85}{0.15} & \MTCWITHCONF{1.334}{0.009} & \MTCWITHCONF{0.433}{0.001} & \SECONDBEST{\MTCWITHCONF{90.28}{0.11}} & \MTCWITHCONF{0.987}{0.005} & \MTCWITHCONF{0.364}{0.001} \\
 &  & L2 & \SECONDBEST{\MTCWITHCONF{83.66}{0.15}} & \MTCWITHCONF{1.051}{0.007} & \MTCWITHCONF{0.314}{0.001} & \MTCWITHCONF{87.83}{0.15} & \MTCWITHCONF{0.922}{0.006} & \MTCWITHCONF{0.289}{0.001} & \MTCWITHCONF{90.21}{0.11} & \MTCWITHCONF{0.81}{0.005} & \MTCWITHCONF{0.263}{0.001} \\
 &  & CL2 & \MTCWITHCONF{83.62}{0.15} & \MTCWITHCONF{0.775}{0.005} & \MTCWITHCONF{0.261}{0.001} & \SECONDBEST{\MTCWITHCONF{88.1}{0.15}} & \MTCWITHCONF{0.823}{0.006} & \MTCWITHCONF{0.281}{0.001} & \MTCWITHCONF{90.09}{0.11} & \MTCWITHCONF{0.701}{0.005} & \MTCWITHCONF{0.236}{0.001} \\
 &  & ZN & \BEST{\MTCWITHCONF{83.86}{0.15}} & \MTCWITHCONF{0.795}{0.005} & \MTCWITHCONF{0.258}{0.001} & \BEST{\MTCWITHCONF{88.47}{0.15}} & \MTCWITHCONF{0.835}{0.006} & \MTCWITHCONF{0.296}{0.001} & \BEST{\MTCWITHCONF{90.47}{0.11}} & \MTCWITHCONF{0.716}{0.004} & \MTCWITHCONF{0.251}{0.001} \\
 &  & ReRep & \MTCWITHCONF{82.44}{0.15} & \MTCWITHCONF{1.62}{0.003} & \MTCWITHCONF{0.871}{0.001} & \MTCWITHCONF{86.85}{0.16} & \MTCWITHCONF{1.725}{0.004} & \MTCWITHCONF{0.872}{0.001} & \MTCWITHCONF{89.83}{0.11} & \MTCWITHCONF{1.431}{0.003} & \MTCWITHCONF{0.869}{0.001} \\
 &  & EASE & \MTCWITHCONF{82.83}{0.15} & \MTCWITHCONF{0.628}{0.005} & \MTCWITHCONF{0.185}{0.001} & \MTCWITHCONF{87.63}{0.16} & \MTCWITHCONF{0.597}{0.005} & \MTCWITHCONF{0.182}{0.001} & \MTCWITHCONF{89.74}{0.12} & \SECONDBEST{\MTCWITHCONF{0.609}{0.004}} & \MTCWITHCONF{0.186}{0.001} \\
 &  & TCPR & \MTCWITHCONF{83.51}{0.15} & \MTCWITHCONF{0.766}{0.005} & \MTCWITHCONF{0.255}{0.001} & \MTCWITHCONF{88.09}{0.15} & \MTCWITHCONF{0.795}{0.005} & \MTCWITHCONF{0.283}{0.001} & \MTCWITHCONF{90.28}{0.11} & \MTCWITHCONF{0.687}{0.004} & \MTCWITHCONF{0.236}{0.001} \\
 &  & noHub & \MTCWITHCONF{83.28}{0.15} & \BEST{\MTCWITHCONF{0.287}{0.004}} & \SECONDBEST{\MTCWITHCONF{0.096}{0.001}} & \MTCWITHCONF{87.58}{0.16} & \BEST{\MTCWITHCONF{0.288}{0.004}} & \SECONDBEST{\MTCWITHCONF{0.102}{0.001}} & \MTCWITHCONF{89.89}{0.12} & \BEST{\MTCWITHCONF{0.334}{0.004}} & \SECONDBEST{\MTCWITHCONF{0.121}{0.001}} \\
 &  & noHub-S & \MTCWITHCONF{83.25}{0.16} & \SECONDBEST{\MTCWITHCONF{0.487}{0.006}} & \BEST{\MTCWITHCONF{0.086}{0.001}} & \MTCWITHCONF{87.82}{0.16} & \SECONDBEST{\MTCWITHCONF{0.469}{0.006}} & \BEST{\MTCWITHCONF{0.091}{0.001}} & \MTCWITHCONF{89.38}{0.17} & \MTCWITHCONF{nan}{nan} & \BEST{\MTCWITHCONF{0.097}{0.001}} \\
\cline{2-12}
 & \multirow[c]{9}{*}{\rotatebox{90}{SIAMESE}} & None & \MTCWITHCONF{20.0}{0.0} & \MTCWITHCONF{1.441}{0.009} & \MTCWITHCONF{0.421}{0.001} & \MTCWITHCONF{20.0}{0.0} & \MTCWITHCONF{1.339}{0.009} & \MTCWITHCONF{0.433}{0.001} & \MTCWITHCONF{20.0}{0.0} & \MTCWITHCONF{0.984}{0.005} & \MTCWITHCONF{0.364}{0.001} \\
 &  & L2 & \MTCWITHCONF{83.14}{0.14} & \MTCWITHCONF{1.035}{0.007} & \MTCWITHCONF{0.319}{0.001} & \MTCWITHCONF{87.04}{0.16} & \MTCWITHCONF{0.912}{0.006} & \MTCWITHCONF{0.288}{0.001} & \MTCWITHCONF{89.48}{0.12} & \MTCWITHCONF{0.808}{0.005} & \MTCWITHCONF{0.264}{0.001} \\
 &  & CL2 & \MTCWITHCONF{84.04}{0.15} & \MTCWITHCONF{0.788}{0.005} & \MTCWITHCONF{0.264}{0.001} & \MTCWITHCONF{87.9}{0.16} & \MTCWITHCONF{0.816}{0.006} & \MTCWITHCONF{0.28}{0.001} & \MTCWITHCONF{90.14}{0.12} & \MTCWITHCONF{0.698}{0.005} & \MTCWITHCONF{0.235}{0.001} \\
 &  & ZN & \MTCWITHCONF{20.0}{0.0} & \MTCWITHCONF{0.8}{0.005} & \MTCWITHCONF{0.263}{0.001} & \MTCWITHCONF{20.0}{0.0} & \MTCWITHCONF{0.84}{0.006} & \MTCWITHCONF{0.296}{0.001} & \MTCWITHCONF{20.0}{0.0} & \MTCWITHCONF{0.713}{0.004} & \MTCWITHCONF{0.251}{0.001} \\
 &  & ReRep & \MTCWITHCONF{20.0}{0.0} & \MTCWITHCONF{1.633}{0.003} & \MTCWITHCONF{0.863}{0.001} & \MTCWITHCONF{20.0}{0.0} & \MTCWITHCONF{1.724}{0.004} & \MTCWITHCONF{0.872}{0.001} & \MTCWITHCONF{20.0}{0.0} & \MTCWITHCONF{1.428}{0.002} & \MTCWITHCONF{0.869}{0.001} \\
 &  & EASE & \SECONDBEST{\MTCWITHCONF{84.61}{0.15}} & \MTCWITHCONF{0.63}{0.005} & \MTCWITHCONF{0.187}{0.001} & \SECONDBEST{\MTCWITHCONF{88.33}{0.16}} & \MTCWITHCONF{0.594}{0.004} & \MTCWITHCONF{0.182}{0.001} & \MTCWITHCONF{90.42}{0.12} & \MTCWITHCONF{0.607}{0.004} & \MTCWITHCONF{0.185}{0.001} \\
 &  & TCPR & \MTCWITHCONF{84.39}{0.15} & \MTCWITHCONF{0.772}{0.005} & \MTCWITHCONF{0.259}{0.001} & \MTCWITHCONF{88.26}{0.16} & \MTCWITHCONF{0.791}{0.005} & \MTCWITHCONF{0.283}{0.001} & \SECONDBEST{\MTCWITHCONF{90.5}{0.11}} & \MTCWITHCONF{0.686}{0.004} & \MTCWITHCONF{0.235}{0.001} \\
 &  & noHub & \MTCWITHCONF{84.05}{0.16} & \SECONDBEST{\MTCWITHCONF{0.292}{0.004}} & \SECONDBEST{\MTCWITHCONF{0.096}{0.001}} & \MTCWITHCONF{87.87}{0.17} & \BEST{\MTCWITHCONF{0.291}{0.004}} & \SECONDBEST{\MTCWITHCONF{0.103}{0.001}} & \MTCWITHCONF{90.34}{0.12} & \BEST{\MTCWITHCONF{0.334}{0.004}} & \SECONDBEST{\MTCWITHCONF{0.12}{0.001}} \\
 &  & noHub-S & \BEST{\MTCWITHCONF{84.67}{0.15}} & \BEST{\MTCWITHCONF{0.247}{0.005}} & \BEST{\MTCWITHCONF{0.074}{0.001}} & \BEST{\MTCWITHCONF{88.43}{0.16}} & \SECONDBEST{\MTCWITHCONF{0.473}{0.006}} & \BEST{\MTCWITHCONF{0.092}{0.001}} & \BEST{\MTCWITHCONF{90.52}{0.12}} & \SECONDBEST{\MTCWITHCONF{0.443}{0.006}} & \BEST{\MTCWITHCONF{0.097}{0.001}} \\
\cline{2-12}
 & \multirow[c]{9}{*}{\rotatebox{90}{SimpleShot}} & None & \MTCWITHCONF{78.5}{0.14} & \MTCWITHCONF{1.436}{0.009} & \MTCWITHCONF{0.422}{0.001} & \MTCWITHCONF{83.95}{0.16} & \MTCWITHCONF{1.339}{0.008} & \MTCWITHCONF{0.432}{0.001} & \MTCWITHCONF{85.65}{0.12} & \MTCWITHCONF{0.987}{0.005} & \MTCWITHCONF{0.364}{0.001} \\
 &  & L2 & \MTCWITHCONF{79.89}{0.14} & \MTCWITHCONF{1.04}{0.007} & \MTCWITHCONF{0.318}{0.001} & \MTCWITHCONF{84.5}{0.16} & \MTCWITHCONF{0.914}{0.006} & \MTCWITHCONF{0.287}{0.001} & \MTCWITHCONF{86.46}{0.12} & \MTCWITHCONF{0.812}{0.005} & \MTCWITHCONF{0.263}{0.001} \\
 &  & CL2 & \MTCWITHCONF{80.0}{0.14} & \MTCWITHCONF{0.786}{0.005} & \MTCWITHCONF{0.264}{0.001} & \MTCWITHCONF{84.66}{0.16} & \MTCWITHCONF{0.821}{0.006} & \MTCWITHCONF{0.28}{0.001} & \MTCWITHCONF{86.3}{0.12} & \MTCWITHCONF{0.698}{0.005} & \MTCWITHCONF{0.236}{0.001} \\
 &  & ZN & \MTCWITHCONF{80.57}{0.14} & \MTCWITHCONF{0.806}{0.005} & \MTCWITHCONF{0.264}{0.001} & \MTCWITHCONF{84.97}{0.16} & \MTCWITHCONF{0.839}{0.006} & \MTCWITHCONF{0.296}{0.001} & \MTCWITHCONF{86.76}{0.12} & \MTCWITHCONF{0.716}{0.005} & \MTCWITHCONF{0.25}{0.001} \\
 &  & ReRep & \MTCWITHCONF{80.86}{0.14} & \MTCWITHCONF{1.631}{0.003} & \MTCWITHCONF{0.863}{0.001} & \MTCWITHCONF{85.05}{0.16} & \MTCWITHCONF{1.721}{0.004} & \MTCWITHCONF{0.872}{0.001} & \SECONDBEST{\MTCWITHCONF{87.83}{0.12}} & \MTCWITHCONF{1.432}{0.002} & \MTCWITHCONF{0.869}{0.001} \\
 &  & EASE & \MTCWITHCONF{80.13}{0.14} & \MTCWITHCONF{0.624}{0.005} & \MTCWITHCONF{0.186}{0.001} & \MTCWITHCONF{84.74}{0.16} & \MTCWITHCONF{0.598}{0.004} & \MTCWITHCONF{0.183}{0.001} & \MTCWITHCONF{86.76}{0.12} & \MTCWITHCONF{0.607}{0.004} & \MTCWITHCONF{0.186}{0.001} \\
 &  & TCPR & \MTCWITHCONF{80.15}{0.14} & \MTCWITHCONF{0.78}{0.005} & \MTCWITHCONF{0.259}{0.001} & \MTCWITHCONF{84.86}{0.15} & \MTCWITHCONF{0.796}{0.005} & \MTCWITHCONF{0.283}{0.001} & \MTCWITHCONF{86.8}{0.12} & \MTCWITHCONF{0.687}{0.004} & \MTCWITHCONF{0.235}{0.001} \\
 &  & noHub & \BEST{\MTCWITHCONF{82.13}{0.14}} & \SECONDBEST{\MTCWITHCONF{0.286}{0.004}} & \SECONDBEST{\MTCWITHCONF{0.096}{0.001}} & \BEST{\MTCWITHCONF{86.31}{0.16}} & \SECONDBEST{\MTCWITHCONF{0.289}{0.004}} & \SECONDBEST{\MTCWITHCONF{0.104}{0.001}} & \BEST{\MTCWITHCONF{88.46}{0.11}} & \BEST{\MTCWITHCONF{0.329}{0.004}} & \SECONDBEST{\MTCWITHCONF{0.12}{0.001}} \\
 &  & noHub-S & \SECONDBEST{\MTCWITHCONF{81.22}{0.14}} & \BEST{\MTCWITHCONF{0.25}{0.005}} & \BEST{\MTCWITHCONF{0.074}{0.001}} & \SECONDBEST{\MTCWITHCONF{86.22}{0.15}} & \BEST{\MTCWITHCONF{0.213}{0.005}} & \BEST{\MTCWITHCONF{0.078}{0.001}} & \MTCWITHCONF{87.6}{0.12} & \SECONDBEST{\MTCWITHCONF{0.433}{0.006}} & \BEST{\MTCWITHCONF{0.097}{0.001}} \\
\cline{2-12}
 & \multirow[c]{9}{*}{\rotatebox{90}{\( \alpha \)-TIM}} & None & \MTCWITHCONF{78.51}{0.15} & \MTCWITHCONF{1.45}{0.009} & \MTCWITHCONF{0.42}{0.001} & \MTCWITHCONF{83.86}{0.16} & \MTCWITHCONF{1.341}{0.009} & \MTCWITHCONF{0.433}{0.001} & \MTCWITHCONF{85.7}{0.12} & \MTCWITHCONF{0.981}{0.005} & \MTCWITHCONF{0.363}{0.001} \\
 &  & L2 & \MTCWITHCONF{80.02}{0.16} & \MTCWITHCONF{1.036}{0.007} & \MTCWITHCONF{0.318}{0.001} & \MTCWITHCONF{84.49}{0.18} & \MTCWITHCONF{0.92}{0.006} & \MTCWITHCONF{0.288}{0.001} & \MTCWITHCONF{87.88}{0.13} & \MTCWITHCONF{0.812}{0.005} & \MTCWITHCONF{0.264}{0.001} \\
 &  & CL2 & \MTCWITHCONF{80.46}{0.16} & \MTCWITHCONF{0.784}{0.005} & \MTCWITHCONF{0.264}{0.001} & \MTCWITHCONF{84.86}{0.17} & \MTCWITHCONF{0.82}{0.006} & \MTCWITHCONF{0.281}{0.001} & \MTCWITHCONF{87.53}{0.13} & \MTCWITHCONF{0.701}{0.005} & \MTCWITHCONF{0.235}{0.001} \\
 &  & ZN & \MTCWITHCONF{80.32}{0.14} & \MTCWITHCONF{0.802}{0.005} & \MTCWITHCONF{0.263}{0.001} & \MTCWITHCONF{84.93}{0.16} & \MTCWITHCONF{0.834}{0.006} & \MTCWITHCONF{0.295}{0.001} & \MTCWITHCONF{86.95}{0.12} & \MTCWITHCONF{0.715}{0.004} & \MTCWITHCONF{0.25}{0.001} \\
 &  & ReRep & \MTCWITHCONF{81.05}{0.14} & \MTCWITHCONF{1.63}{0.003} & \MTCWITHCONF{0.863}{0.001} & \MTCWITHCONF{85.18}{0.16} & \MTCWITHCONF{1.718}{0.004} & \MTCWITHCONF{0.872}{0.001} & \MTCWITHCONF{87.63}{0.12} & \MTCWITHCONF{1.43}{0.002} & \MTCWITHCONF{0.87}{0.001} \\
 &  & EASE & \MTCWITHCONF{79.13}{0.15} & \MTCWITHCONF{0.632}{0.005} & \MTCWITHCONF{0.188}{0.001} & \MTCWITHCONF{84.04}{0.17} & \MTCWITHCONF{0.596}{0.004} & \MTCWITHCONF{0.181}{0.001} & \MTCWITHCONF{86.7}{0.13} & \MTCWITHCONF{0.607}{0.004} & \MTCWITHCONF{0.186}{0.001} \\
 &  & TCPR & \MTCWITHCONF{80.52}{0.16} & \MTCWITHCONF{0.776}{0.005} & \MTCWITHCONF{0.259}{0.001} & \MTCWITHCONF{85.01}{0.17} & \MTCWITHCONF{0.796}{0.005} & \MTCWITHCONF{0.283}{0.001} & \MTCWITHCONF{87.81}{0.13} & \MTCWITHCONF{0.681}{0.004} & \MTCWITHCONF{0.234}{0.001} \\
 &  & noHub & \BEST{\MTCWITHCONF{81.39}{0.15}} & \SECONDBEST{\MTCWITHCONF{0.29}{0.004}} & \SECONDBEST{\MTCWITHCONF{0.096}{0.001}} & \SECONDBEST{\MTCWITHCONF{86.09}{0.16}} & \SECONDBEST{\MTCWITHCONF{0.292}{0.004}} & \SECONDBEST{\MTCWITHCONF{0.103}{0.001}} & \BEST{\MTCWITHCONF{88.16}{0.12}} & \BEST{\MTCWITHCONF{0.336}{0.004}} & \SECONDBEST{\MTCWITHCONF{0.121}{0.001}} \\
 &  & noHub-S & \SECONDBEST{\MTCWITHCONF{81.37}{0.15}} & \BEST{\MTCWITHCONF{0.253}{0.005}} & \BEST{\MTCWITHCONF{0.074}{0.001}} & \BEST{\MTCWITHCONF{86.14}{0.16}} & \BEST{\MTCWITHCONF{0.219}{0.005}} & \BEST{\MTCWITHCONF{0.078}{0.001}} & \SECONDBEST{\MTCWITHCONF{87.97}{0.12}} & \SECONDBEST{\MTCWITHCONF{0.437}{0.006}} & \BEST{\MTCWITHCONF{0.096}{0.001}} \\
\cline{1-12} \cline{2-12}
\bottomrule
\end{tabular}
}
        \caption{Resnet-18: 5-shot}
        \label{tab:main-tim-resnet18-5}
        \end{table*}
    \begin{table*}
        {\scriptsize\centering\begin{tabular}{llllllllllll}
\toprule
 &  &  & \multicolumn{3}{c}{mini} & \multicolumn{3}{c}{tiered} & \multicolumn{3}{c}{CUB} \\
 &  &  & Acc & Skew & Hub.~Occ. & Acc & Skew & Hub.~Occ. & Acc & Skew & Hub.~Occ. \\
Arch. & Clf. & Emb. &  &  &  &  &  &  &  &  &  \\
\midrule
\multirow[c]{54}{*}{\rotatebox{90}{WideRes28-10}} & \multirow[c]{9}{*}{\rotatebox{90}{ILPC}} & None & \MTCWITHCONF{81.93}{0.16} & \MTCWITHCONF{1.717}{0.01} & \MTCWITHCONF{0.473}{0.001} & \MTCWITHCONF{84.34}{0.17} & \MTCWITHCONF{1.927}{0.011} & \MTCWITHCONF{0.509}{0.001} & \MTCWITHCONF{93.18}{0.11} & \MTCWITHCONF{1.164}{0.008} & \MTCWITHCONF{0.396}{0.001} \\
 &  & L2 & \MTCWITHCONF{85.74}{0.14} & \MTCWITHCONF{0.888}{0.005} & \MTCWITHCONF{0.322}{0.001} & \MTCWITHCONF{86.26}{0.17} & \MTCWITHCONF{0.859}{0.005} & \MTCWITHCONF{0.306}{0.001} & \MTCWITHCONF{93.77}{0.1} & \MTCWITHCONF{0.636}{0.004} & \MTCWITHCONF{0.266}{0.001} \\
 &  & CL2 & \MTCWITHCONF{83.33}{0.16} & \MTCWITHCONF{1.12}{0.009} & \MTCWITHCONF{0.318}{0.001} & \MTCWITHCONF{85.99}{0.17} & \MTCWITHCONF{0.957}{0.006} & \MTCWITHCONF{0.338}{0.001} & \MTCWITHCONF{93.79}{0.1} & \MTCWITHCONF{0.703}{0.004} & \MTCWITHCONF{0.309}{0.001} \\
 &  & ZN & \MTCWITHCONF{85.96}{0.14} & \MTCWITHCONF{0.858}{0.005} & \MTCWITHCONF{0.32}{0.001} & \MTCWITHCONF{86.77}{0.16} & \MTCWITHCONF{0.909}{0.006} & \MTCWITHCONF{0.335}{0.001} & \MTCWITHCONF{93.73}{0.1} & \MTCWITHCONF{0.696}{0.004} & \MTCWITHCONF{0.305}{0.001} \\
 &  & ReRep & \MTCWITHCONF{72.11}{0.27} & \MTCWITHCONF{1.601}{0.003} & \MTCWITHCONF{0.819}{0.001} & \MTCWITHCONF{71.68}{0.3} & \MTCWITHCONF{1.616}{0.004} & \MTCWITHCONF{0.845}{0.001} & \MTCWITHCONF{91.52}{0.13} & \MTCWITHCONF{1.301}{0.005} & \MTCWITHCONF{0.548}{0.002} \\
 &  & EASE & \MTCWITHCONF{85.89}{0.14} & \MTCWITHCONF{0.577}{0.004} & \MTCWITHCONF{0.198}{0.001} & \MTCWITHCONF{86.83}{0.17} & \MTCWITHCONF{0.583}{0.004} & \MTCWITHCONF{0.193}{0.001} & \BEST{\MTCWITHCONF{93.87}{0.1}} & \MTCWITHCONF{0.576}{0.004} & \MTCWITHCONF{0.242}{0.001} \\
 &  & TCPR & \SECONDBEST{\MTCWITHCONF{86.29}{0.14}} & \MTCWITHCONF{0.715}{0.004} & \MTCWITHCONF{0.27}{0.001} & \SECONDBEST{\MTCWITHCONF{86.96}{0.17}} & \MTCWITHCONF{0.819}{0.005} & \MTCWITHCONF{0.295}{0.001} & \SECONDBEST{\MTCWITHCONF{93.82}{0.1}} & \MTCWITHCONF{0.634}{0.004} & \MTCWITHCONF{0.265}{0.001} \\
 &  & noHub & \MTCWITHCONF{86.07}{0.15} & \BEST{\MTCWITHCONF{0.295}{0.004}} & \SECONDBEST{\MTCWITHCONF{0.115}{0.001}} & \MTCWITHCONF{86.75}{0.17} & \BEST{\MTCWITHCONF{0.299}{0.004}} & \BEST{\MTCWITHCONF{0.115}{0.001}} & \MTCWITHCONF{93.72}{0.1} & \BEST{\MTCWITHCONF{0.2}{0.004}} & \BEST{\MTCWITHCONF{0.101}{0.001}} \\
 &  & noHub-S & \BEST{\MTCWITHCONF{86.41}{0.14}} & \SECONDBEST{\MTCWITHCONF{0.499}{0.006}} & \BEST{\MTCWITHCONF{0.104}{0.001}} & \BEST{\MTCWITHCONF{87.31}{0.17}} & \SECONDBEST{\MTCWITHCONF{0.406}{0.005}} & \SECONDBEST{\MTCWITHCONF{0.121}{0.001}} & \MTCWITHCONF{93.79}{0.1} & \SECONDBEST{\MTCWITHCONF{0.416}{0.005}} & \SECONDBEST{\MTCWITHCONF{0.126}{0.001}} \\
\cline{2-12}
 & \multirow[c]{9}{*}{\rotatebox{90}{LaplacianShot}} & None & \MTCWITHCONF{85.23}{0.13} & \MTCWITHCONF{1.711}{0.01} & \MTCWITHCONF{0.474}{0.001} & \MTCWITHCONF{86.14}{0.15} & \MTCWITHCONF{1.921}{0.011} & \MTCWITHCONF{0.509}{0.001} & \MTCWITHCONF{92.61}{0.1} & \MTCWITHCONF{1.164}{0.008} & \MTCWITHCONF{0.395}{0.001} \\
 &  & L2 & \MTCWITHCONF{85.9}{0.13} & \MTCWITHCONF{0.892}{0.006} & \MTCWITHCONF{0.321}{0.001} & \MTCWITHCONF{86.47}{0.15} & \MTCWITHCONF{0.867}{0.006} & \MTCWITHCONF{0.304}{0.001} & \MTCWITHCONF{93.17}{0.09} & \MTCWITHCONF{0.635}{0.004} & \MTCWITHCONF{0.267}{0.001} \\
 &  & CL2 & \MTCWITHCONF{82.08}{0.15} & \MTCWITHCONF{1.112}{0.009} & \MTCWITHCONF{0.318}{0.001} & \MTCWITHCONF{84.62}{0.16} & \MTCWITHCONF{0.954}{0.006} & \MTCWITHCONF{0.34}{0.001} & \MTCWITHCONF{93.01}{0.1} & \MTCWITHCONF{0.702}{0.004} & \MTCWITHCONF{0.309}{0.001} \\
 &  & ZN & \MTCWITHCONF{85.97}{0.13} & \MTCWITHCONF{0.86}{0.005} & \MTCWITHCONF{0.319}{0.001} & \MTCWITHCONF{86.67}{0.15} & \MTCWITHCONF{0.912}{0.006} & \MTCWITHCONF{0.335}{0.001} & \MTCWITHCONF{93.3}{0.1} & \MTCWITHCONF{0.698}{0.004} & \MTCWITHCONF{0.305}{0.001} \\
 &  & ReRep & \MTCWITHCONF{84.34}{0.14} & \MTCWITHCONF{1.599}{0.003} & \MTCWITHCONF{0.819}{0.001} & \MTCWITHCONF{85.61}{0.16} & \MTCWITHCONF{1.615}{0.004} & \MTCWITHCONF{0.845}{0.001} & \MTCWITHCONF{92.2}{0.1} & \MTCWITHCONF{1.304}{0.005} & \MTCWITHCONF{0.549}{0.002} \\
 &  & EASE & \SECONDBEST{\MTCWITHCONF{86.24}{0.13}} & \MTCWITHCONF{0.573}{0.004} & \MTCWITHCONF{0.198}{0.001} & \SECONDBEST{\MTCWITHCONF{86.74}{0.15}} & \MTCWITHCONF{0.582}{0.004} & \MTCWITHCONF{0.194}{0.001} & \MTCWITHCONF{93.31}{0.09} & \MTCWITHCONF{0.578}{0.004} & \MTCWITHCONF{0.243}{0.001} \\
 &  & TCPR & \MTCWITHCONF{86.16}{0.13} & \MTCWITHCONF{0.712}{0.004} & \MTCWITHCONF{0.269}{0.001} & \MTCWITHCONF{85.72}{0.16} & \MTCWITHCONF{0.813}{0.005} & \MTCWITHCONF{0.293}{0.001} & \MTCWITHCONF{92.99}{0.1} & \MTCWITHCONF{0.638}{0.004} & \MTCWITHCONF{0.264}{0.001} \\
 &  & noHub & \BEST{\MTCWITHCONF{86.25}{0.13}} & \BEST{\MTCWITHCONF{0.292}{0.004}} & \SECONDBEST{\MTCWITHCONF{0.115}{0.001}} & \BEST{\MTCWITHCONF{86.78}{0.16}} & \BEST{\MTCWITHCONF{0.299}{0.004}} & \BEST{\MTCWITHCONF{0.115}{0.001}} & \BEST{\MTCWITHCONF{93.38}{0.09}} & \BEST{\MTCWITHCONF{0.197}{0.004}} & \BEST{\MTCWITHCONF{0.1}{0.001}} \\
 &  & noHub-S & \MTCWITHCONF{85.79}{0.13} & \SECONDBEST{\MTCWITHCONF{0.494}{0.006}} & \BEST{\MTCWITHCONF{0.103}{0.001}} & \MTCWITHCONF{86.44}{0.16} & \SECONDBEST{\MTCWITHCONF{0.397}{0.005}} & \SECONDBEST{\MTCWITHCONF{0.12}{0.001}} & \SECONDBEST{\MTCWITHCONF{93.36}{0.1}} & \SECONDBEST{\MTCWITHCONF{0.42}{0.005}} & \SECONDBEST{\MTCWITHCONF{0.126}{0.001}} \\
\cline{2-12}
 & \multirow[c]{9}{*}{\rotatebox{90}{ObliqueManifold}} & None & \MTCWITHCONF{87.46}{0.13} & \MTCWITHCONF{1.712}{0.01} & \MTCWITHCONF{0.472}{0.001} & \SECONDBEST{\MTCWITHCONF{88.16}{0.15}} & \MTCWITHCONF{1.913}{0.01} & \MTCWITHCONF{0.509}{0.001} & \MTCWITHCONF{94.75}{0.09} & \MTCWITHCONF{1.161}{0.008} & \MTCWITHCONF{0.395}{0.001} \\
 &  & L2 & \MTCWITHCONF{87.61}{0.13} & \MTCWITHCONF{0.889}{0.005} & \MTCWITHCONF{0.321}{0.001} & \MTCWITHCONF{88.14}{0.15} & \MTCWITHCONF{0.862}{0.006} & \MTCWITHCONF{0.306}{0.001} & \BEST{\MTCWITHCONF{94.8}{0.09}} & \MTCWITHCONF{0.642}{0.004} & \MTCWITHCONF{0.268}{0.001} \\
 &  & CL2 & \MTCWITHCONF{86.03}{0.14} & \MTCWITHCONF{1.112}{0.009} & \MTCWITHCONF{0.317}{0.001} & \MTCWITHCONF{87.64}{0.16} & \MTCWITHCONF{0.949}{0.006} & \MTCWITHCONF{0.338}{0.001} & \MTCWITHCONF{94.67}{0.09} & \MTCWITHCONF{0.703}{0.004} & \MTCWITHCONF{0.31}{0.001} \\
 &  & ZN & \SECONDBEST{\MTCWITHCONF{87.88}{0.13}} & \MTCWITHCONF{0.852}{0.005} & \MTCWITHCONF{0.32}{0.001} & \BEST{\MTCWITHCONF{88.43}{0.15}} & \MTCWITHCONF{0.908}{0.006} & \MTCWITHCONF{0.335}{0.001} & \SECONDBEST{\MTCWITHCONF{94.77}{0.08}} & \MTCWITHCONF{0.697}{0.004} & \MTCWITHCONF{0.306}{0.001} \\
 &  & ReRep & \MTCWITHCONF{87.62}{0.12} & \MTCWITHCONF{1.599}{0.003} & \MTCWITHCONF{0.819}{0.001} & \MTCWITHCONF{88.15}{0.15} & \MTCWITHCONF{1.616}{0.004} & \MTCWITHCONF{0.845}{0.001} & \MTCWITHCONF{94.48}{0.09} & \MTCWITHCONF{1.302}{0.005} & \MTCWITHCONF{0.547}{0.002} \\
 &  & EASE & \MTCWITHCONF{86.75}{0.13} & \MTCWITHCONF{0.573}{0.004} & \MTCWITHCONF{0.198}{0.001} & \MTCWITHCONF{87.78}{0.15} & \MTCWITHCONF{0.583}{0.004} & \MTCWITHCONF{0.193}{0.001} & \MTCWITHCONF{94.16}{0.09} & \MTCWITHCONF{0.57}{0.004} & \MTCWITHCONF{0.24}{0.001} \\
 &  & TCPR & \BEST{\MTCWITHCONF{87.94}{0.12}} & \MTCWITHCONF{0.718}{0.004} & \MTCWITHCONF{0.271}{0.001} & \MTCWITHCONF{88.15}{0.15} & \MTCWITHCONF{0.816}{0.005} & \MTCWITHCONF{0.294}{0.001} & \MTCWITHCONF{94.47}{0.09} & \MTCWITHCONF{0.635}{0.004} & \MTCWITHCONF{0.265}{0.001} \\
 &  & noHub & \MTCWITHCONF{87.23}{0.13} & \BEST{\MTCWITHCONF{0.297}{0.004}} & \SECONDBEST{\MTCWITHCONF{0.115}{0.001}} & \MTCWITHCONF{87.95}{0.16} & \BEST{\MTCWITHCONF{0.296}{0.004}} & \BEST{\MTCWITHCONF{0.114}{0.001}} & \MTCWITHCONF{94.13}{0.09} & \BEST{\MTCWITHCONF{0.197}{0.004}} & \BEST{\MTCWITHCONF{0.1}{0.001}} \\
 &  & noHub-S & \MTCWITHCONF{87.13}{0.14} & \SECONDBEST{\MTCWITHCONF{0.495}{0.006}} & \BEST{\MTCWITHCONF{0.103}{0.001}} & \MTCWITHCONF{87.84}{0.16} & \SECONDBEST{\MTCWITHCONF{0.399}{0.005}} & \SECONDBEST{\MTCWITHCONF{0.12}{0.001}} & \MTCWITHCONF{94.06}{0.09} & \SECONDBEST{\MTCWITHCONF{0.421}{0.005}} & \SECONDBEST{\MTCWITHCONF{0.126}{0.001}} \\
\cline{2-12}
 & \multirow[c]{9}{*}{\rotatebox{90}{SIAMESE}} & None & \MTCWITHCONF{58.82}{0.31} & \MTCWITHCONF{1.722}{0.01} & \MTCWITHCONF{0.473}{0.001} & \MTCWITHCONF{82.56}{0.22} & \MTCWITHCONF{1.93}{0.01} & \MTCWITHCONF{0.511}{0.001} & \MTCWITHCONF{82.22}{0.37} & \MTCWITHCONF{1.154}{0.008} & \MTCWITHCONF{0.396}{0.001} \\
 &  & L2 & \MTCWITHCONF{87.11}{0.13} & \MTCWITHCONF{0.894}{0.006} & \MTCWITHCONF{0.321}{0.001} & \MTCWITHCONF{87.34}{0.15} & \MTCWITHCONF{0.861}{0.005} & \MTCWITHCONF{0.305}{0.001} & \MTCWITHCONF{94.15}{0.1} & \MTCWITHCONF{0.638}{0.004} & \MTCWITHCONF{0.266}{0.001} \\
 &  & CL2 & \MTCWITHCONF{83.99}{0.16} & \MTCWITHCONF{1.107}{0.009} & \MTCWITHCONF{0.318}{0.001} & \MTCWITHCONF{86.71}{0.16} & \MTCWITHCONF{0.953}{0.006} & \MTCWITHCONF{0.339}{0.001} & \MTCWITHCONF{94.48}{0.09} & \MTCWITHCONF{0.704}{0.004} & \MTCWITHCONF{0.31}{0.001} \\
 &  & ZN & \MTCWITHCONF{20.0}{0.0} & \MTCWITHCONF{0.856}{0.005} & \MTCWITHCONF{0.319}{0.001} & \MTCWITHCONF{20.0}{0.0} & \MTCWITHCONF{0.913}{0.006} & \MTCWITHCONF{0.334}{0.001} & \MTCWITHCONF{20.0}{0.0} & \MTCWITHCONF{0.702}{0.004} & \MTCWITHCONF{0.305}{0.001} \\
 &  & ReRep & \MTCWITHCONF{36.41}{0.3} & \MTCWITHCONF{1.597}{0.003} & \MTCWITHCONF{0.818}{0.001} & \MTCWITHCONF{76.49}{0.24} & \MTCWITHCONF{1.613}{0.004} & \MTCWITHCONF{0.846}{0.001} & \MTCWITHCONF{60.36}{0.6} & \MTCWITHCONF{1.299}{0.005} & \MTCWITHCONF{0.547}{0.002} \\
 &  & EASE & \SECONDBEST{\MTCWITHCONF{87.82}{0.13}} & \MTCWITHCONF{0.579}{0.004} & \MTCWITHCONF{0.199}{0.001} & \SECONDBEST{\MTCWITHCONF{88.06}{0.16}} & \MTCWITHCONF{0.586}{0.004} & \MTCWITHCONF{0.192}{0.001} & \MTCWITHCONF{94.36}{0.09} & \MTCWITHCONF{0.571}{0.004} & \MTCWITHCONF{0.241}{0.001} \\
 &  & TCPR & \MTCWITHCONF{87.8}{0.13} & \MTCWITHCONF{0.717}{0.004} & \MTCWITHCONF{0.27}{0.001} & \MTCWITHCONF{87.95}{0.16} & \MTCWITHCONF{0.822}{0.005} & \MTCWITHCONF{0.295}{0.001} & \MTCWITHCONF{94.25}{0.1} & \MTCWITHCONF{0.637}{0.004} & \MTCWITHCONF{0.266}{0.001} \\
 &  & noHub & \MTCWITHCONF{87.78}{0.14} & \BEST{\MTCWITHCONF{0.29}{0.004}} & \SECONDBEST{\MTCWITHCONF{0.114}{0.001}} & \MTCWITHCONF{87.99}{0.17} & \BEST{\MTCWITHCONF{0.297}{0.004}} & \BEST{\MTCWITHCONF{0.115}{0.001}} & \SECONDBEST{\MTCWITHCONF{94.56}{0.09}} & \BEST{\MTCWITHCONF{0.196}{0.004}} & \BEST{\MTCWITHCONF{0.1}{0.001}} \\
 &  & noHub-S & \BEST{\MTCWITHCONF{88.03}{0.13}} & \SECONDBEST{\MTCWITHCONF{0.492}{0.006}} & \BEST{\MTCWITHCONF{0.103}{0.001}} & \BEST{\MTCWITHCONF{88.31}{0.16}} & \SECONDBEST{\MTCWITHCONF{0.398}{0.005}} & \SECONDBEST{\MTCWITHCONF{0.12}{0.001}} & \BEST{\MTCWITHCONF{94.69}{0.09}} & \SECONDBEST{\MTCWITHCONF{0.416}{0.005}} & \SECONDBEST{\MTCWITHCONF{0.127}{0.001}} \\
\cline{2-12}
 & \multirow[c]{9}{*}{\rotatebox{90}{SimpleShot}} & None & \MTCWITHCONF{78.56}{0.14} & \MTCWITHCONF{1.709}{0.01} & \MTCWITHCONF{0.473}{0.001} & \MTCWITHCONF{80.32}{0.16} & \MTCWITHCONF{1.937}{0.01} & \MTCWITHCONF{0.51}{0.001} & \MTCWITHCONF{89.27}{0.11} & \MTCWITHCONF{1.16}{0.008} & \MTCWITHCONF{0.395}{0.001} \\
 &  & L2 & \MTCWITHCONF{83.81}{0.13} & \MTCWITHCONF{0.887}{0.005} & \MTCWITHCONF{0.322}{0.001} & \MTCWITHCONF{84.82}{0.15} & \MTCWITHCONF{0.86}{0.006} & \MTCWITHCONF{0.305}{0.001} & \MTCWITHCONF{92.06}{0.1} & \MTCWITHCONF{0.632}{0.004} & \MTCWITHCONF{0.266}{0.001} \\
 &  & CL2 & \MTCWITHCONF{81.05}{0.14} & \MTCWITHCONF{1.12}{0.009} & \MTCWITHCONF{0.318}{0.001} & \MTCWITHCONF{83.82}{0.16} & \MTCWITHCONF{0.956}{0.006} & \MTCWITHCONF{0.337}{0.001} & \MTCWITHCONF{92.19}{0.1} & \MTCWITHCONF{0.701}{0.004} & \MTCWITHCONF{0.31}{0.001} \\
 &  & ZN & \MTCWITHCONF{83.92}{0.13} & \MTCWITHCONF{0.858}{0.005} & \MTCWITHCONF{0.32}{0.001} & \MTCWITHCONF{85.1}{0.15} & \MTCWITHCONF{0.912}{0.006} & \MTCWITHCONF{0.335}{0.001} & \MTCWITHCONF{92.17}{0.1} & \MTCWITHCONF{0.699}{0.004} & \MTCWITHCONF{0.305}{0.001} \\
 &  & ReRep & \MTCWITHCONF{79.26}{0.16} & \MTCWITHCONF{1.597}{0.003} & \MTCWITHCONF{0.819}{0.001} & \MTCWITHCONF{82.7}{0.16} & \MTCWITHCONF{1.617}{0.004} & \MTCWITHCONF{0.846}{0.001} & \MTCWITHCONF{91.48}{0.11} & \MTCWITHCONF{1.299}{0.005} & \MTCWITHCONF{0.549}{0.002} \\
 &  & EASE & \MTCWITHCONF{83.65}{0.13} & \MTCWITHCONF{0.579}{0.004} & \MTCWITHCONF{0.199}{0.001} & \MTCWITHCONF{84.47}{0.15} & \MTCWITHCONF{0.585}{0.004} & \MTCWITHCONF{0.193}{0.001} & \MTCWITHCONF{92.01}{0.1} & \MTCWITHCONF{0.572}{0.004} & \MTCWITHCONF{0.241}{0.001} \\
 &  & TCPR & \MTCWITHCONF{83.77}{0.13} & \MTCWITHCONF{0.717}{0.004} & \MTCWITHCONF{0.27}{0.001} & \MTCWITHCONF{84.81}{0.15} & \MTCWITHCONF{0.815}{0.005} & \MTCWITHCONF{0.294}{0.001} & \MTCWITHCONF{91.84}{0.1} & \MTCWITHCONF{0.634}{0.004} & \MTCWITHCONF{0.264}{0.001} \\
 &  & noHub & \BEST{\MTCWITHCONF{85.73}{0.13}} & \BEST{\MTCWITHCONF{0.294}{0.004}} & \SECONDBEST{\MTCWITHCONF{0.115}{0.001}} & \BEST{\MTCWITHCONF{86.58}{0.15}} & \BEST{\MTCWITHCONF{0.298}{0.004}} & \BEST{\MTCWITHCONF{0.115}{0.001}} & \SECONDBEST{\MTCWITHCONF{93.21}{0.09}} & \BEST{\MTCWITHCONF{0.195}{0.004}} & \BEST{\MTCWITHCONF{0.1}{0.001}} \\
 &  & noHub-S & \SECONDBEST{\MTCWITHCONF{84.39}{0.13}} & \SECONDBEST{\MTCWITHCONF{0.494}{0.006}} & \BEST{\MTCWITHCONF{0.103}{0.001}} & \SECONDBEST{\MTCWITHCONF{86.38}{0.15}} & \SECONDBEST{\MTCWITHCONF{0.407}{0.005}} & \SECONDBEST{\MTCWITHCONF{0.12}{0.001}} & \BEST{\MTCWITHCONF{93.39}{0.09}} & \SECONDBEST{\MTCWITHCONF{0.421}{0.005}} & \SECONDBEST{\MTCWITHCONF{0.127}{0.001}} \\
\cline{2-12}
 & \multirow[c]{9}{*}{\rotatebox{90}{\( \alpha \)-TIM}} & None & \MTCWITHCONF{80.61}{0.15} & \MTCWITHCONF{1.711}{0.01} & \MTCWITHCONF{0.473}{0.001} & \MTCWITHCONF{83.05}{0.18} & \MTCWITHCONF{1.928}{0.01} & \MTCWITHCONF{0.51}{0.001} & \MTCWITHCONF{84.89}{0.29} & \MTCWITHCONF{1.153}{0.008} & \MTCWITHCONF{0.396}{0.001} \\
 &  & L2 & \MTCWITHCONF{83.71}{0.16} & \MTCWITHCONF{0.892}{0.005} & \MTCWITHCONF{0.323}{0.001} & \MTCWITHCONF{84.69}{0.18} & \MTCWITHCONF{0.863}{0.005} & \MTCWITHCONF{0.304}{0.001} & \MTCWITHCONF{92.88}{0.1} & \MTCWITHCONF{0.633}{0.004} & \MTCWITHCONF{0.266}{0.001} \\
 &  & CL2 & \MTCWITHCONF{82.35}{0.16} & \MTCWITHCONF{1.111}{0.009} & \MTCWITHCONF{0.318}{0.001} & \MTCWITHCONF{84.06}{0.18} & \MTCWITHCONF{0.949}{0.006} & \MTCWITHCONF{0.339}{0.001} & \MTCWITHCONF{92.81}{0.1} & \MTCWITHCONF{0.7}{0.004} & \MTCWITHCONF{0.31}{0.001} \\
 &  & ZN & \MTCWITHCONF{83.93}{0.13} & \MTCWITHCONF{0.857}{0.005} & \MTCWITHCONF{0.321}{0.001} & \MTCWITHCONF{85.07}{0.15} & \MTCWITHCONF{0.912}{0.006} & \MTCWITHCONF{0.336}{0.001} & \MTCWITHCONF{92.15}{0.1} & \MTCWITHCONF{0.698}{0.004} & \MTCWITHCONF{0.306}{0.001} \\
 &  & ReRep & \MTCWITHCONF{83.4}{0.14} & \MTCWITHCONF{1.596}{0.003} & \MTCWITHCONF{0.82}{0.001} & \MTCWITHCONF{84.4}{0.16} & \MTCWITHCONF{1.615}{0.004} & \MTCWITHCONF{0.845}{0.001} & \SECONDBEST{\MTCWITHCONF{93.19}{0.09}} & \MTCWITHCONF{1.302}{0.005} & \MTCWITHCONF{0.547}{0.002} \\
 &  & EASE & \MTCWITHCONF{82.72}{0.14} & \MTCWITHCONF{0.576}{0.004} & \MTCWITHCONF{0.2}{0.001} & \MTCWITHCONF{83.86}{0.16} & \MTCWITHCONF{0.583}{0.004} & \MTCWITHCONF{0.193}{0.001} & \MTCWITHCONF{92.31}{0.1} & \MTCWITHCONF{0.572}{0.004} & \MTCWITHCONF{0.242}{0.001} \\
 &  & TCPR & \SECONDBEST{\MTCWITHCONF{84.21}{0.15}} & \MTCWITHCONF{0.718}{0.004} & \MTCWITHCONF{0.27}{0.001} & \MTCWITHCONF{84.63}{0.18} & \MTCWITHCONF{0.814}{0.005} & \MTCWITHCONF{0.293}{0.001} & \MTCWITHCONF{92.44}{0.1} & \MTCWITHCONF{0.635}{0.004} & \MTCWITHCONF{0.265}{0.001} \\
 &  & noHub & \BEST{\MTCWITHCONF{85.56}{0.13}} & \BEST{\MTCWITHCONF{0.293}{0.004}} & \SECONDBEST{\MTCWITHCONF{0.115}{0.001}} & \BEST{\MTCWITHCONF{86.37}{0.16}} & \BEST{\MTCWITHCONF{0.3}{0.004}} & \BEST{\MTCWITHCONF{0.115}{0.001}} & \MTCWITHCONF{92.89}{0.1} & \BEST{\MTCWITHCONF{0.193}{0.004}} & \BEST{\MTCWITHCONF{0.099}{0.001}} \\
 &  & noHub-S & \MTCWITHCONF{83.96}{0.15} & \SECONDBEST{\MTCWITHCONF{0.496}{0.006}} & \BEST{\MTCWITHCONF{0.102}{0.001}} & \SECONDBEST{\MTCWITHCONF{86.01}{0.16}} & \SECONDBEST{\MTCWITHCONF{0.395}{0.005}} & \SECONDBEST{\MTCWITHCONF{0.12}{0.001}} & \BEST{\MTCWITHCONF{93.24}{0.1}} & \SECONDBEST{\MTCWITHCONF{0.422}{0.005}} & \SECONDBEST{\MTCWITHCONF{0.126}{0.001}} \\
\cline{1-12} \cline{2-12}
\bottomrule
\end{tabular}
}
        \caption{WideRes28-10: 5-shot}
        \label{tab:main-s2m2-wrn-s2m2-5}
    \end{table*}

    \section{Potential Negative Societal Impacts}
        As is the case with most methodological research in machine learning, the methods developed in this work could be used in downstream applications with potential negative societal impacts.
Real world machine learning-based systems that interact with humans, or the environment in general, should therefore be properly tested and equipped with adequate safety measures.

Since our work relies on a large number of labeled examples from the base classes, un-discovered biases from the base dataset could be transferred to the trained models.
Furthermore, the small number of examples in the inference stage could make the query predictions biased towards the included support examples, and not accurately reflect the diversity of the novel classes.

    \bibliographystyle{unsrtnat}
    \bibliography{references}
\end{document}